\documentclass[11pt]{article}
\usepackage[letterpaper, left=1in, right=1in, top=1in,bottom=1in]{geometry}
\usepackage{amsfonts}       
\usepackage{nicefrac}       
\usepackage{bbm}
\usepackage{natbib}
\usepackage{subcaption}
\usepackage[ruled,noend]{algorithm2e}
\usepackage{bm}
\usepackage{tikz}
\usetikzlibrary{positioning,chains,fit,shapes,calc,arrows}
\usetikzlibrary{arrows.meta}
\usepackage{graphicx}
\usepackage{pgf}
\usepackage{mathtools}
\usepackage{mathrsfs}
\usepackage{amsmath,amsthm,amssymb}
\usepackage{comment}
\usepackage{xfrac}
\usepackage{accents}
\usepackage{cancel}

\usepackage{auxiliary}

\usepackage{float}

\usepackage{url}      

\usepackage[pagebackref=true]{hyperref} \PassOptionsToPackage{pagebackref=true}{hyperref}   

\usepackage{cleveref}

\usepackage[suppress]{color-edits}
\addauthor{tl}{cyan}
\addauthor{ww}{orange}
\addauthor{rev}{blue}

\title{Learning to Defer in Congested Systems: The AI-Human Interplay}

\author{Thodoris Lykouris\thanks{Massachusetts Institute of Technology, \texttt{lykouris@mit.edu}} 
\and Wentao Weng\thanks{Massachusetts Institute of Technology, \texttt{wweng@mit.edu}}}

\date{First version: February 2024\\Current version: August 2025\footnote{The preliminary version of the paper was titled \emph{Learning to Defer in Content Moderation: The Human-AI Interplay}. The current version simplifies the model and abstracts it beyond content moderation. It also adds a numerical study on real content-moderation data to illustrate the practical insights of the work.}}

\begin{document}

\maketitle
\thispagestyle{empty}

\begin{abstract}
High-stakes applications rely on combining Artificial Intelligence (AI) and humans for responsive and reliable decision making. For example, content moderation in social media platforms often employs an AI-human pipeline to promptly remove policy violations without jeopardizing legitimate content. A typical heuristic estimates the risk of incoming content and uses fixed thresholds to decide whether to auto-delete the content (classification) and whether to send it for human review (admission). This approach can be inefficient as it disregards the uncertainty in AI's estimation, the time-varying element of content arrivals and human review capacity, and the selective sampling in the online dataset (humans only review content filtered by the AI). 

In this paper, we introduce a model to capture such an AI-human interplay. In this model, the AI observes contextual information for incoming jobs, makes classification and admission decisions, and schedules admitted jobs for human review. During these reviews, humans observe a job's true cost and may overturn an erroneous AI classification decision. These reviews also serve as new data to train the AI but are delayed due to congestion in the human review system. The objective is to minimize the costs of eventually misclassified jobs. While classical \emph{learning to defer} frameworks assume fixed-cost and immediate human feedback, our model introduces congestion in the human review system. Moreover, unlike work on \emph{learning with delayed feedback} where the delay in the feedback is exogenous to the algorithm's decisions, the delay in our model is endogenous to both the admission and the scheduling decisions. 

We propose a near-optimal learning algorithm that carefully balances the classification loss from a selectively sampled dataset, the idiosyncratic loss of non-reviewed jobs, and the delay loss of having congestion in the human review system. To the best of our knowledge, this is the first result for online learning in contextual queueing systems. Moreover, numerical experiments based on online comment datasets show that our algorithm can substantially reduce the number of misclassifications compared to existing content moderation practice.
\end{abstract}

\newpage
\setcounter{page}{1}

\section{Introduction}\label{sec:intro}
Recent advances in Artificial Intelligence (AI) provide the promise of freeing humans from repetitive tasks by \emph{responsive} automation, thus enabling the humankind to focus on more creative endeavors  \citep{forbesAI}. One example of automating traditionally human-centric tasks is content moderation targeting misinformation and explicitly harmful content in social media such as Facebook \citep{Facebook-standard}, Twitter \citep{X-rule}, and Reddit \citep{reddit-automod}. Historically (e.g., forums in the 2000s), human reviewers would monitor all exchanges to detect any content that violated the community standards \citep{roberts2019behind}. That said, the high volume of content in current platforms coupled with the advances in AI has led platforms to automate their content moderation, harnessing the responsiveness of AI.\footnote{\label{footnote:estimate} 
According to \cite{makhijani2021quest}, Meta receives billions of content pieces per day. Among them, at least 2 million content pieces are reviewed by Meta's 15,000 human reviewers~\citep{Avadhanula2022}. The New York Times reports that a reviewer can check up to $700$ content pieces per shift \citep{Facebook-nyt}. These together imply a review ratio of around $0.2\%$ to $1\%$.} This trend of automating traditionally human-centric tasks applies broadly beyond content moderation; for example, speedy insurance claim process \citep{Pingan} and domain-specific generative AI copilots \citep{Github}.

However, excessive use of automation in such human-centric applications significantly reduces the \emph{system reliability}. AI models are trained based on historical data and therefore their predictions reflect patterns observed in the past that are not always accurate for the current task. 
On the other hand, humans' cognitive abilities and expertise make humans more attune to correct decisions. In content moderation, particular content may have language that is unclear, complex, and too context-dependent, obscuring automated predictions \citep{Facebook-content}. Similarly, AI models may wrongfully reject a valid insurance claim \citep{eubanks2018automating} and Large Language Model (LLM) copilots may hallucinate non-existing legal cases \citep{Lawyer-chatgpt}. These errors can have significant ethical and legal repercussions.

The \emph{learning to defer} paradigm \citep{MadrasCPZ18} is a common way to combine the responsiveness of AI and the reliability of humans. When a new job arrives, the AI model classifies it as \emph{accept} or \emph{reject}, and determines \emph{whether to defer} the job for human review by admitting it to a corresponding queue (in content moderation, incoming jobs correspond to new content and the classification decision pertains to whether the content is kept on or removed from the platform). When a human reviewer becomes available, the AI model determines which job to schedule for human review. As a result, the AI model directly determines which jobs will be reviewed by humans. 

At the same time, humans also affect the AI model as their labels for the reviewed jobs form the dataset based on which the AI is trained, creating an \emph{interplay} between humans and AI. Hence, the AI model's admission decisions are not only useful for correctly classifying the current jobs but are also crucial for its future prediction ability, a phenomenon known as \emph{selective sampling}. 

\vspace{-1em}
\begin{figure}[!h]
  \centering
  \scalebox{0.6}{\includegraphics{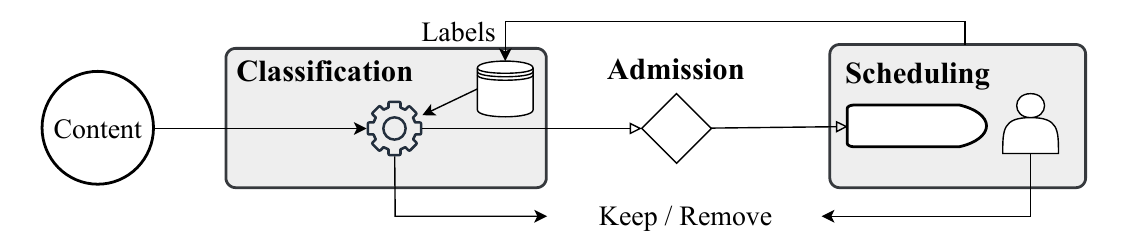}}
  \caption{Pipeline of AI-human interplay in content moderation}
  \label{fig:pipeline} 
\end{figure}

In this paper, we model this \emph{AI-human interplay} and pose the following question:
\vspace{-1em}
\begin{center}
    \emph{How can we scale AI automation with limited human capacity \\
while combining the responsiveness of AI and the reliability of humans?}
\end{center}

\subsection{Modeling and methodological contributions}
To formally tackle the problem, we introduce a model that combines \emph{learning to defer} with soft capacity constraints via queueing delays. To the best of our knowledge, the impact of limited capacity for learning to defer has not been considered in the literature (see Section~\ref{ssec:related_work} for further discussion). In particular, in each period $t = 1,\ldots,T$, job $t$  arrives with a type $\kappa(t)$ drawn from a time-varying distribution over $K$ types. The type of a job can be viewed as \emph{features} about the job; see Section~\ref{sec:contri-content} for how this instantiates in a content moderation pipeline of Meta Platforms. Each type $k$ has an unknown sub-Gaussian cost distribution $\set{F}_k$ with mean $c_k$ and job $t$ has an unknown cost $C_t$ independently sampled from $\set{F}_{\kappa(t)}$. This cost can be negative (the platform should accept this job) or positive (the platform should reject this job).

To capture the difference in reliability between AI and humans, we assume that the AI model only observes the type of the job; the true cost $C_t$ is only observable via human reviews. This distinction captures the imperfect predictive power of features (and thus the AI). On the one hand, having more data can help AI to learn the cost distributions $\set{F}_k$. On the other hand, even if the AI has assess to infinitely many data, it cannot exactly predict the cost of a job due to the randomness in $\set{F}_k$. As a model assumption, humans can observe the true cost of a job because they have more cognitive abilities and are not limited by the features. To capture the impact of limited human capacity, we assume that jobs that are admitted (deferred) by the AI wait in a review queue. At the end of a period, one job in the queue is scheduled for review; the review is completed in this period with a probability $\mu_k N(t)$ which depends on both the type-dependent per-reviewer service rate $\mu_k$\footnote{Service rates can differ across types, e.g., in content moderation, it may be faster to review a text than a video.} and the fluctuating capacity of reviewers $N(t)$.

We measure the \emph{loss} of a policy by the sum of absolute costs of all jobs eventually misclassified (our results extend to other objectives that also capture jobs' holding costs; see Remark~\ref{remark:holding_loss} and Appendix~\ref{app:holding-cost}). To capture the capacity constraint, we consider a fluid benchmark incorporating time-varying capacity constraints and define the \emph{regret} of a policy as its additional loss to this fluid benchmark. The goal is to design policies whose regret is sublinear in the number of jobs $T$.

\paragraph{Balancing idiosyncrasy loss with delays.} As typical in complex learning settings, we start by assuming knowledge of the latent parameters (i.e., the cost distributions $\{\set{F}_k\}_{k \in \set{K}}$). The classification decision is to keep a job $t$ if and only if its expected cost is negative: $c_{\kappa(t)}<0$. Of course, deferring a job to human reviewers can improve upon this ex-ante classification decision as human reviewers can observe the true cost. The admission decision should thus identify the jobs that would benefit the most from this more refined observation. Existing practice (Section~\ref{sec:contri-content}) auto-deletes jobs with predicted costs higher than a threshold and prioritizes the review of remaining jobs based on these predictions. That said, this static-threshold heuristic does not allow the system to adapt to time-varying arrival patterns of jobs and fluctuating human capacities that arise in practice (see \cite[Figure~2]{makhijani2021quest}). In particular, the number of jobs below the threshold may be either too large resulting to significant delays in the review system or too small resulting in misclassifications for some non-admitted jobs that could have been prevented if they were admitted. 

Our approach (Section~\ref{sec:known}) quantifies and balances the two losses hinted above (due to idiosyncrasy and delay) and is henceforth called \textsc{Balanced Admission Control for Idiosyncrasy and Delay} or \textsc{BACID} as a shorthand. We first calculate the expected ex-ante loss of accepting a job on the platform (resp. rejecting it) using the distributional information of the true cost $C_t$; this loss is given by $\ell_k^+ \coloneqq \expect{\max(C_t,0) \mid \kappa(t)=k}$ (resp. $\ell_k^- \coloneqq \expect{-\min(C_t,0) \mid \kappa(t)=k}$). Hence, if a job is not admitted to human review, its \emph{idiosyncrasy} loss is $\ell_k = \min(\ell_k^+,\ell_k^-)$. On the other hand, admitting a job \emph{leads to delay loss}, which is the increase in congestion for future jobs due to the limited human capacity. Motivated by the drift-plus-penalty literature \citep{stolyar2005maximizing,huang2009delay,neely2022stochastic}, we estimate the delay loss of a job by the number of jobs with the same type currently waiting for reviews. We then admit a job if (a weighted version of) its idiosyncrasy loss (which one can avoid by admission) \emph{exceeds} the estimated delay loss (which occurs due to the admission). For scheduling, we choose in each period the type with the most number of waiting jobs (weighted by the difficulty to review this type) to review. We show that \textsc{BACID} achieves a $O(\sqrt{KT})$  regret even with general time-varying arrivals and capacities (Theorem~\ref{thm:bacid}). This bound is optimal on the dependence on $T$ as we also show a lower bound $\Omega(\sqrt{T})$ for any  policy (Theorem~\ref{thm:lower-bound}).

\paragraph{Selective sampling, optimism and forced scheduling.} 
To deal with unknown distributions, a typical approach in the bandit literature is to use an \emph{optimistic} estimator for the unknown quantities. The unknown quantity that affects our admission rule is $\ell_k$, which depends on the unknown loss of accepting (resp. rejecting) a job, i.e., $\ell_k^+$ (resp. $\ell_k^-$). Following the optimistic approach, we first create optimistic estimators $\bar{\ell}_k^+,\bar{\ell}_k^-$ for the unknown expected losses $\ell_k^+, \ell_k^-$ based on reviewed jobs. We then compute an optimistic estimator $\bar{\ell}_k = \min(\bar{\ell}_k^+,\bar{\ell}_k^-)$ on the idiosyncrasy loss. The optimistic admission rule then admits a new job if (a weighted version of) the optimistic idiosyncrasy loss, $\bar{\ell}_k$, exceeds its estimated delay loss. Without considering the endogenous classification decisions, existing content moderation practice \citep{Avadhanula2022} adopts a similar optimism-only heuristic to address the online learning nature of the mean cost $c_k$.

However, optimism-only learning heuristics disregard the selective sampling nature of the feedback. When there is uncertainty on cost distributions, the benefit of admitting a job is not restricted to avoiding the idiosyncrasy loss of the \emph{current} job, but extends to having one more data point that can improve the classification decision of \emph{future} jobs. This benefit does not arise when cost distributions are known because the best classification decision is then clear \textemdash  accepting a job if and only if the expected loss of accepting it is lower, i.e., $\ell_k^+ < \ell_k^-$. With unknown distributions, there is an additional \emph{classification loss} due to incorrect estimation of the sign of~$\ell_k^+-\ell_k^-$. Given that optimism-only heuristics emulate the admission rule with known distributions, they disregard this classification loss and more broadly the positive externality that the label for the current job has on future jobs. We illustrate this inefficiency of optimism-only learning heuristics formally in Proposition~\ref{prop:fail-to-learn} and numerically in Section~\ref{sec:evaluation}.

To circumvent the above issue, we design an online learning version of \textsc{BACID} (Section~\ref{sec:unknown}), which we term $\bacidol$, that augments  optimistic admission with \emph{label-driven admission} and \emph{forced scheduling}. When a job arrives, we estimate its potential classification loss based on its current confidence interval; if this is higher than a threshold, we admit the job to a \emph{label-driven queue} and prioritize jobs in this queue via \emph{forced scheduling} to ensure enough labels for future classification decisions. To avoid exhausting human capacity, we limit the number of jobs in the label-driven queue to one at any time. We show that $\bacidol$ enjoys a regret of $\tilde{O}(K\sqrt{T\ln T})$ (Theorem~\ref{thm:bacidol}), addressing the inefficiency of the optimism-only heuristic. 

A key novelty in our analysis is to address online learning with queueing delayed feedback. This occurs because the algorithm only obtains labels of admitted jobs \emph{after} humans review them. Different from classical bandit and online learning settings, the delays in observations depend on both the algorithm's admission and scheduling decisions and are correlated with each other due to queueing effects. We address this challenge (Section~\ref{sec:bacidol-idio}) by connecting the number of unobserved labels with the queue length of each type, which is always bounded due to our admission algorithm.

\paragraph{Context-defined types with type aggregation and contextual learning.} When the number of types $K$ is large, the above algorithm suffers from two salient challenges: over-exploration and over-admission. First, $\bacidol$ creates a separate confidence interval for each type and refining those intervals necessitates adequate exploration of each type. Second, both \textsc{BACID} and $\bacidol$ estimate the delay loss of admitting a job by the number of same-type jobs in the review system. This naturally creates a dependence on the number of types $K$. 

To address these challenges, we focus on a linear contextual setting where $[\ell_k^+,\ell_k^-] = \bphi_k^{\trans}\bTheta^\star$. Here $\bphi_k$ is a known $d-$dimensional feature vector for each type $k$ and $\bTheta^\star$ is an unknown $d \times 2$ dimensional matrix. 
To tackle over-exploration, our contextual algorithm ($\conbacid$) in Section~\ref{sec:contextual} exploits the contextual information to set appropriate confidence bounds. As mentioned above, compared to canonical works on linear contextual bandits \citep{Abbasi-YadkoriPS11}, our setting has the additional complexity of queueing delayed feedback. Tackling such delayed feedback is more difficult when the number of types is large as we now cannot bound the queue length for each type. To address this challenge, we rely on the contextual structure and properties of our scheduling policies; see Section~\ref{sec:contextual-regret} for detailed discussions. Moreover, to tackle over-admission and avoid maintaining separate queue for every type, $\conbacid$ 
aggregates types into $G$ groups such that any two types $k,k'$ in the same group satisfy $|\mu_k - \mu_{k'}| \leq \Delta$ for some parameter~$\Delta$. This aggregation enables us to estimate the delay loss by the number of same-group waiting jobs and schedule based on groups, thus removing the aforementioned dependence on $K$.

Our instance-dependent regret guarantee $\tilde{O}(d\sqrt{GT}+\Delta T)$~(Theorem~\ref{thm:cbacidol}) scales optimally with the time horizon $T$ and avoids the dependence on $K$. We also offer a worst-case guarantee $\tilde{O}(d^{5/6}T^{2/3})$ with no dependence on $G$ and $\Delta$ at the cost of a worse dependence on $T$ (Corollary~\ref{corr:group-ind}). To the best of our knowledge, this is the first result for online learning in contextual queueing systems and hence our analytical framework may be of independent interest.

\subsection{Contribution to content moderation practice}\label{sec:contri-content}
We now discuss how our insights apply to a content moderation pipeline in Meta Platforms described by \cite{Avadhanula2022}. In this pipeline, there are several machine learning (ML) models predicting the risk scores of a post violating specific platform policies and these ML models are hard to retrain. Due to distribution shifts, \cite{Avadhanula2022} propose and \emph{deploy} an online learning approach to combine these ML models for minimizing the number of misclassified posts (suitably weighted by posts' numbers of views) at the end of a horizon.
\cite{Avadhanula2022} rely on a simple aggregator that adaptively combines the risk scores from these ML models. When a post $t$ arrives to the system (corresponding to a new job in our model), the system queries the ML models and obtains a risk score $x_{t,i}$ for each ML model~$i$. \cite{Avadhanula2022} transform these risk scores into a feature vector~$\bphi_t$. Subsequently they follow three steps: classification, admission, and scheduling. For classification, they auto-delete a post if its maximum risk score is ``unambiguously'' high. For admission, they maintain an aggregator model that predicts a post's actual violation $y_t$ based on its feature vector~$\bphi_t$. To encourage exploration, they admit a post for human review if the upper-confidence-bound estimate from the aggregator, $\bar{y}_t$, is positive. For scheduling, they prioritize posts with large $\bar{y}_t$ (suitably weighted by their numbers of views). Each reviewed post is used to update the aggregator model and \cite{Avadhanula2022} report that the update frequency is around once every five minutes. Sections~\ref{sec:overview-practice} and \ref{sec:model-instantiate} provide further details and how our model captures this pipeline.

Using real comment data from Wikipedia \citep{dataset-wiki} and the Civil Comments platform \citep{dataset-civil}, Section~\ref{sec:set-up} sets up a simulation experiment to numerically compare our algorithm ($\conbacid$) with the above heuristic. Section~\ref{sec:evaluation} highlights the importance of congestion-awareness and using label-driven admission and optimistic admission to address distribution shifts between offline and online datasets. Our experiments show that $\conbacid$ adapts to non-stationarity in the review capacity (Figure~\ref{fig:horizon}), obtains lower misclassifications than the above heuristic (consistently reducing misclassifications by more than $10\%$, see Figure~\ref{fig:80}), and ensures the learning of a good AI classification model (see Figure~\ref{fig:selective} where the above heuristic can fail to classify most policy-violating posts). Moreover, in achieving these benefits, our algorithm requires no more computation than the already implemented upper confidence bound approach in \cite{Avadhanula2022}, making it a viable option to improve the performance of the content moderation system in social media platforms.

\subsection{Related work}\label{ssec:related_work}
\paragraph{Human-AI collaboration.} 
The nature of human-AI collaboration can be broadly classified into two types: \emph{augmentation} and \emph{automation} \citep{brynjolfsson2017can, raisch2021artificial}. In particular, augmentation represents a ``human-in-the-loop'' type workforce where human experts combine machine learning with their own judgement to make better decisions. Such augmentation is often found in high-stake settings such as healthcare \citep{lebovitz2022engage}, child maltreatment hotline screening \citep{De-ArteagaFC20}, refugee resettlement \citep{bansak2020outcome,ahani2023dynamic} and bail decisions \citep{kleinberg2018human}. A general concern for augmentation relates to humans' compliance patterns \citep{bastani2021improving,lebovitz2022engage}, the impact of such patterns to decision accuracy \citep{kleinberg2018human,mclaughlin2022algorithmic} and the impact on fairness \citep{McLaughlinSG22,ge2023rethinking}. Automation, on the other hand, concerns the use of machine learning in place of humans and is widely applied in human and social services where colossal demands overwhelm limited human capacity, such as in data labelling \citep{vishwakarma2023promises}, content moderation \citep{gorwa2020algorithmic,makhijani2021quest} and insurance \citep{eubanks2018automating}. Full automation is clearly undesirable in these applications as machine learning can err. A natural question is how one can better utilize the limited human capacity when machine is uncertain about the prediction. The literature on ``learning to defer'', which studies when machine learning algorithms should defer decisions to downstream experts, tries to answer this question and is where our work fits in. 

Assuming that humans have perfect prediction ability but there is limited capacity, classical learning to defer has two streams of research, \emph{learning with abstention} and \emph{selective sampling}. Although these two streams have been studied separately (with a few exceptions, see discussion below), our work provides an endogenous approach to connect them.  In particular, learning with abstention can be traced back to \cite{Chow57,Chow70} and studies an offline classification problem with the option to not classify a data point for a fixed cost. For example, in content moderation, this corresponds to paying a fixed fee to an exogenous human reviewer to review a post. Since optimizing the original problem is computational infeasible, an extensive line of work investigates suitable surrogate loss functions and optimization methods; see \cite{BartlettW08,El-YanivW10,CortesDM16a} and references therein. For online learning with expert advice, it is shown that even a limited amount of abstention allows better regret bound than without \citep{SayediZB10,LiLWS11,ZhangC16,NeuZ20}; \cite{CortesDGMY18} studies a similar problem but allows experts to also abstain. Selective sampling (or label efficient prediction) considers a different model where the algorithm makes predictions for every arrival; but the ground truth label is unavailable unless the algorithm queries for it; if queried, the label is available immediately \citep{Cesa-BianchiLS05}. The goal is to obtain low regret while using as few queries as possible. Typical solutions query only when a confidence interval exceeds a certain threshold \citep{Cesa-BianchiLS05,cesa2009robust,OrabonaC11,DekelGS12}. There is recent work connecting the two directions by showing that allowing abstention of a small fixed cost can lead to better regret bound for selective sampling \citep{Zhu022a,PuchkinZ22}. However, these works treat the abstention cost and query limits exogenously, neglecting  the impact of limited human capacity to learning to defer \citep{Leitao22}. Our model serve as one step to capture it by endogenously connecting both costly abstention and selective sampling via delays. In particular, our admission component determines both abstention and selective sampling. In addition, the cost of abstention is dynamically affected by the delay in getting human reviews whereas the avoidance of frequent sampling (to get data) is captured by its impact on the delay. 

Although the above line of work as well as ours assumes perfect labels from human predictions, a more recent stream of work on ``learning to defer'' considers imperfect human predictions and is focused on combining prediction ability of experts and learning algorithms; this moves towards the augmentation type of human-machine collaboration. In particular, \cite{MadrasCPZ18, MozannarS20,wilder2020learning,CharusaieMSS22} study settings with an offline dataset and expert labels. The goal is to learn both a classifier, which predicts outcome, and a rejector that predicts when to defer to human experts. The loss is defined by the machine's classification loss over non-deferred data, humans' classification loss over deferred data, and the cost to query experts. \cite{DeKGG20,DeOZR21} study a setting where given expert loss functions, the algorithm picks a size-limited subset of data to outsource to humans and solves a regression or a classification problem on remaining data, with a goal to minimize the total loss. \cite{raghu2019algorithmic} extends the model by allowing the expert classification loss to depend on human effort and further considering an allocation of human effort to different data points. \cite{KeswaniLK21,VermaBN23,mao2023principled} consider learning to defer with multiple experts. 

\paragraph{Bandits with knapsacks or delays.} Restricting our attention to admission decisions, our model bears similar challenges with the literature on bandits with knapsacks or broadly online learning with resource constraints, which finds applications in revenue management \citep{besbes2012blind,wang2014close,ferreira2018online}. In particular, for bandits with knapsacks, there are arrivals with rewards and required resources from an unknown distribution. The algorithm only observes the reward and required resources after admitting an arrival and the goal is to obtain as much reward as possible subject to resource constraints \citep{badanidiyuru2018bandits,agrawal2019bandits}. A typical primal-dual approach learns the optimal dual variables of a fixed fluid model and explores with upper confidence bound \citep{badanidiyuru2014resourceful,LiSY21}; these extend to contextual settings \citep{WuSLJ15,AgrawalD16,AgrawalDL16, SlivkinsSF23}, general linear constraints \citep{PacchianoGBJ21, LiuLSY21} and constrained reinforcement learning \citep{BrantleyDLMSSS20,Cheung19}. These results cannot immediately apply to our setting for two reasons. First, the resource constraint in our model is dynamically captured by time-varying queue lengths instead of a single resource constraint over the entire horizon; thus learning fixed dual variables is insufficient for good performance. Second, our admission decisions must also consider the effect on classification and relying only on optimism-based exploration is inefficient (see Section~\ref{sec:opti-fail}).

The problem of bandits with delays is related to our setting where feedback of an admitted job gets delayed due to congestion. Motivated by conversion in online advertising, bandits with delays consider the problem where the reward of each pulled arm is only revealed after a random delay independently generated from a fixed distribution \citep{DudikHKKLRZ11,Chapelle14}. Assuming independence between rewards and delays as well as bounded delay expectation, \cite{JoulaniGS13,MandelLBP15} propose a general reduction from non-delay settings to their delayed counterpart. Subsequent papers consider censored settings with unobservable rewards \citep{VernadeCP17}, general (heavy-tail) delay distributions \citep{ManegueuVCV20, WuW22} and reward-dependent delay \citep{LancewickiSKM21}. \cite{VernadeCLZEB20,blanchet2023delay} study (generalized) linear contextual bandit with delayed feedback. The key difference between bandits with delays and our model is that delays for jobs in our model are not independent across jobs due to queueing effect; we provide a more elaborate comparison to those works in Appendix~\ref{app:comparison}. 

\paragraph{Learning in queueing systems.} Learning in queueing systems can be classified into two types: 1) learning to schedule with unknown service rates to obtain low delay; 2) learning unknown utility of jobs / servers to obtain high reward in a congested system. For the first line of research, an intuitive approach to measure delay suboptimality is via the queueing regret, defined as the difference in queue lengths compared with a near-optimal algorithm \citep{walton2021learning}. The interest for queueing regret is in its asymptotic scaling in the time horizon, and it is studied for single-queue multi-server systems \citep{Walton14,KrishnasamySJS21,StahlbuhkSM21}, multi-queue single-server systems \citep{krishnasamy2018learning}, load balancing \citep{choudhury2021job}, queues with abandonment \citep{zhong2022learning} and more general markov decision processes with countable infinite state space \citep{adler2023bayesian}. As an asymptotic metric may not capture the learning efficiency of the system, \cite{freund2023quantifying} considers an alternative metric (cost of learning in queueing) that measures transient performance by the maximum increase in time-averaged queue length. This metric is motivated by works that study stabilization of queueing systems without knowledge of parameters. In particular, \cite{NeelyRP12,yang2023learning,nguyen2023learning} combine the celebrated MaxWeight scheduling algorithm \citep{tassiulas1992stability} with either discounted UCB or sliding-window UCB for scheduling with time-varying service rates. \cite{FuHL22,gaitonde2023price} study decentralized learning with strategic queues and \cite{StahlbuhkSM19,sentenac2021decentralized,FreundLW22} consider efficient decentralized learning algorithms for cooperative queues. Although most work for learning in queues focus on an online stochastic setting, \cite{huang2023queue,liang2018minimizing} study online adversarial setting and \cite{singh2022feature} considers an offline feature-based setting.

Our work is closer to the second literature that learns job utility in a queueing system. In particular, \cite{massoulie2018capacity, ShahGMV20} consider Bayesian learning in an expert system where jobs are routed to different experts for labels and the goal is to keep the expert system stable. \cite{johari2021matching,hsu2022integrated,fu2022optimal} study a matching system where incoming jobs have uncertain payoffs when served by different servers and the objective is to maximize the total utility of served jobs within a finite horizon. \cite{jia2022online,Jia0S22,chen2023online} investigate regret-optimal learning algorithms and \cite{li2023experimenting} studies randomized experimentation for online pricing in a queueing system. 

We note that, although most performance guarantees for learning in queueing systems deteriorate as the number of job types $K$ increases, our work allows admission, scheduling and learning in a many-type setting where the performance guarantee is independent of $K$. Although prior work obtains such a guarantee in a Bayesian setting, where the type of a job corresponds to a distribution over a finite set of labels \citep{argon2009priority,massoulie2018capacity,ShahGMV20}, their service rates are only server-dependent (thus finite). In contrast, our work allows for job-dependent service rates. In addition, a Bayesian setting does not immediately capture the contextual information between jobs that may be useful for learning. We note that \cite{singh2022feature} consider a multi-class queueing system where a job has an observed feature vector and an unobserved job type, and the task is to assign jobs to a fixed number of classes with the goal of minimizing mean holding cost, with known holding cost rate for any type. They find that directly optimizing a mapping from features to classes can greatly reduce the holding cost, compared with a predict-then-optimize approach. Different from their setting, we consider an online learning setting where reviewing a job type provides information for other types. This creates an explore-exploit trade-off complicated with the additional challenge that the feedback experiences queueing delay. To the best of our knowledge, our work is the first result for efficient online learning in a queueing system with contextual information.
 
\paragraph{Joint admission and scheduling.} When there is no learning, our problem becomes a joint admission and scheduling problem that is widely studied in wireless networks and the general focus is on a decentralized system \citep{kelly1998rate,lin2006tutorial}. Our method is based on the drift-plus penalty algorithm \citep{neely2022stochastic}, which is a common approach for joint admission and scheduling, first noted as a greedy primal-dual algorithm in \cite{stolyar2005maximizing}. The intuition is to view queue lengths as dual variables to guide admission; \cite{huang2009delay} formalizes this idea and exploits it to obtain better utility-delay tradeoffs. 

\section{Model}\label{sec:model}
We consider a $T$-period discrete-time system to model the AI-human pipeline on a platform. Each job has a type in a set $\set{K}$ with $|\set{K}| = K$. In period $t=1,\ldots,T$, at most one new job arrives with probability $\lambda(t) \leq 1$ and we label it as job $t$. The type of job $t$, $\kappa(t)$, is equal to $k \in \set{K}$ with probability $\lambda_k(t) / \lambda(t)$. We use $\Lambda_k(t) = \indic{\kappa(t) = k}$ to denote whether a type-$k$ job arrives in period~$t$. We call $\lambda_k(t)$ the arrival rate of a type-$k$ job in period $t$. For notation convenience, when there is no job arrival, we denote $\kappa(t) = \perp.$

A job $t$ has an inherent cost $C_t \in \mathbb{R}$; 
this quantity being negative means that the platform should accept this job while a positive cost means that the platform should reject this job. If job~$t$ has type~$k$, its cost $C_t$ is independently sampled from an unknown sub-Gaussian distribution $\set{F}_k$ with mean $c_k$. To ease notation, if $\kappa(t) = \perp$, then $C_t \equiv 0$, i.e., the job has zero cost. To distinguish the ability of AI and humans, although the type $\kappa(t)$ is observable, we assume the cost~$C_t$ is
unknown until humans review the job. We use $\expectsub{k}{\cdot}$ to condition on a job type $k$, i.e., $\expectsub{k}{\cdot} = \expect{\cdot \mid \kappa(t) = k}$.

Given the unknown costs, the platform resorts to a AI-human pipeline (Figure~\ref{fig:pipeline}) by making three decisions in any period $t$, \emph{classification}, \emph{admission}, and \emph{scheduling}:
\begin{itemize}
\item \emph{Classification}. For job $t$, the platform first makes a classification decision $Y(t)$ such that the platform accepts the job if $Y(t) = 1$ or rejects it if $Y(t) = -1$. A human may later reverse this classification decision if it is incorrect.
\item \emph{Admission}. The platform can admit job $t$ into the human review system; $A(t) = 1$ if it is admitted and $0$ if not. If $A(t) = 1$, job $t$ is enqueued into an initially empty review queue $\set{Q}$. We define $\set{A}(t) = t$ if $A(t) = 1$ and $\set{A}(t) = \perp$ otherwise.
\item \emph{Scheduling}. At the end of period $t$, the platform selects a job from the review queue $\set{Q}$ for humans to review; we denote this job by $M(t)$. To capture the service capacity, we assume that we have $N(t)$ reviewers in period $t$. Let $\psi_k(t) = \indic{\kappa(M(t)) = k}$ indicate whether the platform schedules a type-$k$ job for review. If $M(t)$ is of type $k$, humans successfully review the job in period $t$ with probability $N(t)\mu_k$, where $\mu_k$ is a known type-specific quantity. We assume that $\max_{t,k} N(t)\mu_k \leq 1$ so that the product is a valid probability. Let $S_k(t) = 1$ and $\set{S}(t) = M(t)$ if the review is successful and $S_k(t) = 0, \set{S}(t) = \perp$ otherwise. If the review is successful, we assume human reviewers observe the exact cost $C_{M(t)}$ and reverse the previous classification decision if wrong. If the review is not successful (e.g., requiring more time to review), the true cost is not observed and the job is placed back to the queue and requires future review. Let $\set{Q}(t)$ be the set of jobs in the review queue at the beginning of period $t$. Then $\set{Q}(t+1) = \set{Q}(t) \cup \{\set{A}(t)\} \setminus \{\set{S}(t)\}$. In addition, the dataset of reviewed jobs at the beginning of period $t$, $\set{D}(t)$, is given by $\set{D}(t) = \{(\set{S}(\tau), C_{\set{S}(\tau)})\}_{\tau < t}$.
\end{itemize}

We note that our model makes no assumption on arrival probabilities $\{\lambda_k(t)\}_{k \in \set{K}, t \leq T}$ and review capacities $\{N(t)\}_{t \leq T}$ (except that $\max_{t,k} N(t)\mu_k \leq 1$). There is no explicit \emph{stabilizability} condition as in the queueing literature (see e.g. \cite{yang2023learning}) since the admission decisions endogenously control the actual arrival rate of the human review queue. However, we do assume that the arrival probabilities and reviewer capacities are exogenous, i.e., this is an oblivious adversary setting.

As a modeling choice, we assume at most one job arrives and at most one service completes per period. This assumption is benign when each period represents a small amount of time, as the probability that more than two jobs arrive at the same time or more than two reviewers finish services is negligible. To ease analysis, we aggregate the service capacity of reviewers in each period (the probability of a successfully reviewed type-$k$ job is $N(t) \mu_k$ for period $t$). That said, our algorithms do not rely on these modeling assumptions.

\subsection{Objective}
The key tension in our setting is between a) purely relying on AI for the classification of a job and b) increasing the congestion of the human review queue by admitting this job. The simplest objective that captures this tension induces the misclassification loss for a job that is either a) not admitted or b) admitted but not reviewed at the end of the time horizon. To simplify exposition and facilitate the comparison to prior work in our empirical study of Section~\ref{sec:content_moderation}, the main body of the paper will focus on the aforementioned objective. That said, our algorithm and analysis extend to other objectives that more finely capture the aforementioned tension (see Remark~\ref{remark:holding_loss}).

Formally, a policy $\pi$ incurs \emph{loss} for any job $t$ that has a wrong classification \emph{at the end} of the horizon and the loss is equal to the absolute value of its cost $|C_t|$. A job has wrong classification if and only if both of the following events happen: (1) this job has an incorrect initial AI classification and (2) this job is not reviewed by a human at the end of the horizon. The latter happens if this job is either not admitted for human review ($A(t) = 0$) or is in the queue in period $T + 1$. Define  the sign function $\sign(x) = \indic{x > 0} - \indic{x \leq 0}.$ The loss of a policy $\pi$ is thus
\begin{equation}\label{eq:policy-loss}
\set{L}^\pi(T) \coloneqq \sum_{t=1}^T |C_t|\underbrace{\mathbbm{1}\Big(Y(t) \neq \sign(C_t)\Big)}_{\text{incorrect AI classification}}\underbrace{\Big(1 - A(t)\indic{t \not \in \set{Q}(T+1)}\Big)}_{\text{not reviewed by humans}}.
\end{equation}

Our goal is to design a \emph{feasible} policy with small loss. We assume that, for each type $k$, its cost distribution $\set{F}_k$ is \emph{unknown} initially and must be learned via the human-reviewed data set $\set{D}(t)$. The platform has no information of $\{\lambda_k(t),N(t)\}_{t \in [T]}$. We assume $\expectsub{k}{|C_t|}$ is bounded by a constant $c_{\max}$ for any $k \in \set{K}$ and that the cost distribution $\set{F}_k$ is sub-Gaussian with a known variance proxy $\sigma_{\max}^2$ such that $\expectsub{k}{\exp(s(C_j - c_k))} \leq \exp(\sigma_{\max}^2 s^2/2)$ for any $s \in \mathbb{R}$ and $k \in \set{K}.$ A policy is \emph{feasible} if its decisions for any period $t$ are only based on the observed sample path $\{\kappa(t'),\set{A}(t'),\set{S}(t')\}_{t' < t} \cup \{\kappa(t)\}$, the dataset $\set{D}(t)$ and the initial information $\sigma_{\max}, c_{\max}$.

\begin{remark}\label{remark:holding_loss}
Our model assumes that all jobs eventually reviewed by humans incur zero loss. However, in practice, the loss of a job misclassified by AI may also depend on its wait time until a human reviews it. For example, in social media platforms, a policy-violating post may get user views before a reviewer removes this post. This loss can be captured by the notion of a \emph{holding cost} in the queueing literature. Our results extend to a setting where the holding cost of a job in any period is bounded by a quantity that vanishes with $T$ (see Appendix~\ref{app:holding-cost}). Subsequent works by \cite{lee2024design, gocmen2025scheduling} design scheduling algorithms minimizing more complex holding costs in the context of content moderation, allowing for posts to have non-vanishing holding cost in a particular period (albeit both \cite{lee2024design, gocmen2025scheduling} assume exogenous classification and admission decisions).
\end{remark}

\subsection{Benchmark}
Directly minimizing the loss is difficult due to the unknown cost distributions. Therefore, we consider a (fluid) benchmark given in \eqref{eq:fluid} that operates with knowledge of the cost distributions. In this fluid benchmark, for every period $t$, there is a mass of $\lambda_k(t)$ jobs of type $k$. The platform admits a mass of $a_k(t)$ jobs to review and leaves a mass of $\lambda_k(t) - a_k(t)$ jobs not reviewed. Admitted jobs receive human reviews immediately. For a non-admitted job $t$ of type $k$, the expected classification loss of accepting this job is $\ell_k^+ \coloneqq \expectsub{k}{(C_t)^+}$ and the loss of rejecting this job is $\ell_k^- \coloneqq \expectsub{k}{(C_t)^-}$, where $x^+ = \max(x,0)$ and  $x^- = -\min(x,0)$.
The loss of a non-admitted job in the benchmark is thus $\ell_k \coloneqq \min(\ell_k^+, \ell_k^-)$ (assuming a known cost distribution), by rejecting a job if $c_k > 0$ or accepting it if $c_k \leq 0$ (this is because $\ell_k^+ \leq \ell_k^-$ if and only if $\ell_k^+ - \ell_k^- = c_k \leq 0$). The expected loss of the benchmark is then $\sum_{t = 1}^T \sum_{k \in \set{K}} \ell_k(\lambda_k(t) - a_k(t))$. To capture the human capacity constraint, we require the amount of admitted jobs to be bounded by the available capacity \emph{in every period}. Specifically, the platform decides the probability $\nu_k(t)$ of scheduling a type-$k$ job for review in period $t$ (recall that at most one job is scheduled to review per period). Since the review of a  scheduled type-$k$ job is successful with probability $\mu_k N(t)$, the mass (throughput) of reviewed type-$k$ jobs for period $t$ is $\mu_k N(t) \nu_k(t)$. Our capacity constraint requires that $a_k(t) \leq \mu_k N(t)\nu_k(t)$ for every type $k$ and period~$t$.

\begin{equation}\label{eq:fluid}
\tag{Fluid}
\begin{aligned}
\mathcal{L}^\star(T) = &\min_{\{a_k(t),\nu_k(t)\}_{k \in [K], t\in [T]}}\sum_{t=1}^T \sum_{k \in \set{K}} \ell_k(\lambda_k(t) - a_k(t)),\text{s.t.} \\
\quad& \bolds{a_k(t) \leq \mu_k N(t)\nu_k(t),\forall k \in \set{K}, t \in [T]} \\
&a_k(t) \leq \lambda_k(t),~\nu_k(t) \geq 0,~\forall k \in \set{K}, t\in [T] \\
& \sum_{k\in \set{K}} \nu_k(t) \leq 1,~\forall t \in [T].
\end{aligned}
\end{equation}
To measure the performance of a policy, we invoke the notion of \emph{regret} that compares its expected loss $\expect{\set{L}^\pi(T)}$ to the loss of the benchmark $\set{L}^\star(T)$.
\begin{equation}\label{eq:regre}
 \reg^{\pi}(T) = (\expect{\set{L}^{\pi}(T)} - \set{L}^{\star}(T))^+.
 \end{equation}
Our goal is to find an \emph{asymptotically optimal} policy where the regret of a policy is sublinear in $T$, i.e., $\lim_{T \to \infty} \reg^{\pi}(T) / T = 0.$

\begin{remark}\label{remark:w-fluid}
In a setting with stationary arrivals and capacity, \eqref{eq:fluid} is a natural benchmark because it lower bounds the loss of any feasible policy. When arrivals or capacity are non-stationary, \eqref{eq:fluid} is weak as it does not allow a job to be reviewed by later available capacity. For ease of exposition, the main body focuses on deriving performance guarantees based on this fluid benchmark. In Appendix~\ref{app:w-fluid}, we show how our guarantees extend to a series of relaxed benchmarks called $w-$fluid benchmarks motivated by non-stationary queueing systems \citep{borodin2001adversarial, liang2018minimizing, yang2023learning, nguyen2023learning}, where the capacity constraint is satisfied over consecutive windows of periods instead of for each period. 
\end{remark}

\paragraph{Notation.} For ease of exposition when stating our results, we use $\lesssim$ to include dependence only on the number of types $K$, and the time horizon $T$ with other parameters treated as constants. Note that we use $\perp$ to denote an empty element, so a set $\{\perp\}$ should be interpreted as an empty set $\emptyset$. We also follow the convention that $\frac{a}{0} = +\infty$ for any positive $a$. 
For a $d-$dimensional positive semi-definite (PSD) matrix  $\bolds{V}$, we define its corresponding vector norm by $\|\btheta\|_{\bolds{V}} = \sqrt{\btheta^{\trans}\bolds{V}\btheta}$ for any $\btheta \in \mathbb{R}^d$. We denote $\bolds{I}$ as the identity matrix with a suitable dimension, $\mathrm{det}(\bV)$ as the determinant of matrix $\bV$, and $\lambda_{\min}(\bV)$ as the minimum eigenvalue of a PSD matrix $\bV$.

\section{Balancing Idiosyncrasy and Delay with Known Average Cost}\label{sec:known}
Our starting point is the simpler setting where the cost distribution $\set{F}_k$ is known for every type and the number of types $K$ is small. In this setting, classification can be directly optimized by accepting a job if and only if it has weakly negative mean cost, i.e., $Y(t) = -\sign(c_{\kappa(t)})$. 
The two additional decisions (admission and scheduling) are not as straightforward and give rise to an interesting trade-off. To understand this trade-off, consider a new job $t$ of type~$k$. 
\begin{itemize}
    \item If the job is not admitted for review, it incurs an idiosyncrasy loss of $\ell_k$ due to cost uncertainty: the platform either accepts it ($c_k \leq 0$) and incurs an expected loss $\ell_k^+$, or rejects it ($c_k > 0$) and incurs an expected loss $\ell_k^-$. In either case, the loss of $\ell_k = \min(\ell_k^+,\ell_k^-)$ is unavoidable for an un-admitted job because of the \emph{cost idiosyncrasy} of jobs with the same type.
    \item If the job is admitted for review, the platform temporarily avoids its idiosyncrasy loss because humans have a chance to later review this job. However, if the job is not reviewed by the end of the horizon, the platform still suffers from this idiosyncrasy loss. Since this loss arises from the \emph{delay} in human reviews, we refer to it as the delay loss of a policy.
\end{itemize} 
Formally, let $\set{Q}_k(t)$ be the set of type-$k$ jobs in the review queue $\set{Q}(t)$ in period $t$ and let $Q_k(t) = |\set{Q}_k(t)|$. The above discussion thus shows that: 
\begin{align}
\expect{\set{L}^{\pi}(T)} &= \expect{\sum_{t=1}^T (1-A(t))|C_t|\indic{Y(t) \neq \sign(C_t)} + \sum_{t \in \set{Q}(T+1)} |C_t|\indic{Y(t) \neq \sign(C_t)}} \nonumber\\
&= \underbrace{\expect{\sum_{t=1}^T  \ell_{\kappa(t)}(1-A(t))}}_{\text{Idiosyncrasy loss}} + \underbrace{\expect{\sum_{k \in \set{K}} Q_k(T+1)\ell_k}}_{\text{Delay loss}} \label{eq:loss-decompose},
\end{align}
where for the second equation, we recall that a job $t$ with type $\perp$ always has zero cost, so $\ell_{\perp} = 0$.

This decomposition highlights the necessity of employing an AI-human pipeline. On the one hand, relying on AI and not sending any jobs to humans leads to idiosyncrasy loss for all the jobs (due to the misclassification errors of the AI). On the other hand, abstaining from AI use and sending all the jobs to humans yields delay loss from most of the jobs that would end up in the queue (due to the limited human capacity). Admitting more jobs for human review helps reduce the idiosyncracy loss but, at the same time, increases delay loss, showcasing the underlying tension. 

\subsection{Balanced admission control for idiosyncrasy and delay (\textsc{BACID})}
Our algorithm (Algorithm~\ref{algo:bacid}) adopts a simple admission rule to balance the two losses; we henceforth call it \textsc{Balanced Admission for Classification, Idiosyncrasy and Delay}, or $\textsc{BACID}$.

The main idea is to admit a job if and only if the benefit from admission exceeds the externality it causes on later jobs due to increased congestion. Specifically, for a new job $t$, the platform rejects it if $c_{\kappa(t)} > 0$ and accepts it otherwise (Line~\ref{line:bacid-classify}). The platform admits this job to review if its idiosyncrasy loss $\ell_{\kappa(t)}$, scaled by an admission parameter $\beta$ is greater than the congestion externality caused by such admission, estimated by the current number of type-$\kappa(t)$ jobs in the review queue, $Q_{\kappa(t)}(t)$. Formally, the admission decision is $A(t) = \indic{ \beta \ell_{\kappa(t)} \geq Q_{\kappa(t)}(t)}$ (Line~\ref{line:bacid-admit}). For convenience, we take $Q_{\perp}(t) \equiv 1$ so that this rule never admits a job with type $\perp$ as $\ell_{\perp} = 0.$ Implementing this rule requires the value of $\ell_{\kappa(t)}$, which is known in this section because the cost distribution $\set{F}_{\kappa(t)}$ is assumed known. 

For scheduling (Line~\ref{line:bacid-schedule}), we follow the \textsc{MaxWeight} algorithm \citep{tassiulas1992stability} and select the earliest job in $\set{Q}_{k}(t)$ (first-come-first serve) from the type $k$ that maximizes the product of service rate and queue length, i.e., $k \in \arg\max_{k' \in \set{K}} \mu_{k'} Q_{k'}(t)$ (breaking ties arbitrarily). We then update the queues and the dataset by $\set{Q}(t+1) = \set{Q}(t) \cup \{\mathcal{A}(t)\}\setminus \{\mathcal{S}(t)\}$,~$\mathcal{D}(t+1) = \mathcal{D}(t) \cup \{(\mathcal{S}(t), c_{\mathcal{S}(t)})\}$. 

\begin{algorithm}[H]
\LinesNumbered
\DontPrintSemicolon
  \caption{
  \textsc{Balanced Admission for Classification, Idiosyncrasy \& Delay}}\label{algo:bacid}
  \KwData{$T, \{\ell_k, c_k, \mu_k\}_{k \in \set{K}}$ \qquad\qquad \textbf{Admission parameter:} 
  $\beta \gets \sqrt{T / K}$}
  \For{$t = 1$ \KwTo $T$}{
    Observe a new job of type $\kappa(t)$\\
    \lIf{$c_{\kappa(t)} > 0$}{$Y(t) \gets -1$ 
     \textbf{else} {$Y(t) \gets 1$}
    \tcp*[f]{Classification} \label{line:bacid-classify}}
    \lIf{$\beta \cdot \ell_{\kappa(t)}  \geq Q_{\kappa(t)}(t)$}{
      $A(t) = 1$ \textbf{else} {$A(t) = 0$} 
    \tcp*[f]{Admission}
    \label{line:bacid-admit}
    }
    $k \gets \arg\max_{k' \in \mathcal{K}} \mu_{k'} \cdot Q_{k'}(t)$,~$M(t) \gets \text{first job in } \set{Q}_{k}(t) \text{ if any}$
    \tcp*[f]{Scheduling} \label{line:bacid-schedule}
    
    \lIf{\emph{the review is successful (with probability $N(t) \cdot \mu_k$)}}
    {
      $\mathcal{S}(t) = M(t)$ \textbf{else} 
      $\mathcal{S}(t) = \perp$ 
    }
    }    
\end{algorithm}

Our main result is that \textsc{BACID} with $\beta = \sqrt{T/K}$ achieves a regret of $O(\sqrt{KT})$.
\begin{theorem}\label{thm:bacid}
The regret of \textsc{BACID} is upper bounded by $\reg^{\bacid}(T)\lesssim \sqrt{KT}+K.$
\end{theorem}

The next theorem (proof in Appendix~\ref{app:lower-bound}) shows that the dependence on $\sqrt{T}$ is tight. 

\begin{theorem}\label{thm:lower-bound}
There exists a setting where even with the knowledge of distributions $\{\set{F}_k\}_{k \in \set{K}}$, any feasible policy must incur $\Omega(\sqrt{T})$ regret.
\end{theorem}

To prove Theorem~\ref{thm:bacid}, we rely on the loss decomposition in \eqref{eq:loss-decompose}. 
 We first upper bound the idiosyncrasy loss by showing that its difference to the fluid benchmark \eqref{eq:fluid} is bounded by $T/\beta$. 
 As $\beta$ increases (corresponding to more admissions), the idiosyncrasy loss thus decreases. The proof relies on a coupling with the benchmark using Lyapunov analysis and is given in Section~\ref{ssec:bacid-idiosyncrasy}. 
 \begin{lemma}\label{lem:known-idiosyncrasy}
\textsc{BACID}'s idiosyncrasy loss is
$\expect{\sum_{t=1}^T  \ell_{\kappa(t)}(1 - A(t))} \leq \loss^\star(T)+\frac{T}{\beta}$.
\end{lemma}

Moreover, the policy admits a new job when $\beta \ell_{\kappa(t)} \geq Q_{\kappa(t)}(t)$, which upper bounds the queue length by $\beta \ell_{\kappa(t)} + 1$. This implies a delay loss of $Kc_{\max}(\beta c_{\max}+1)$ because $c_{\max}$ upper bounds $\expectsub{k}{|C_j|}$ and thus upper bounds $\ell_k$ for any $k \in \set{K}$. Hence, a larger $\beta$ leads to more delay loss, which matches our intuition on the trade-off between idiosyncrasy and delay loss. Setting $\beta = \sqrt{T/K}$ balances this trade-off.

\begin{lemma}\label{lem:known-delay}
\textsc{BACID}'s delay loss is
$\expect{\sum_{k \in \set{K}} \ell_k Q_k(T+1)} \leq Kc_{\max} (\beta c_{\max}+1)$.
\end{lemma}
\begin{proof}
By induction on $t=1,2,\ldots$, we show that, for any type $k$, $Q_k(t) \leq \beta \ell_k + 1$. The basis of the induction ($Q_k(1) = 0$) holds as the queue is initially empty. Our admission rule implies that $Q_k(t+1) = Q_k(t)$ for $k \neq \kappa(t)$ and that $Q_{k}(t+1)\leq Q_{k}(t) + A(t) \leq Q_{k}(t) + \indic{\beta \ell_{k} \geq Q_k(t)}$ for $k = \kappa(t), k \neq \perp$ (types $\perp$ are never admitted by our rule). Combined with the induction hypothesis, $Q_k(t) \leq \beta \ell_k + 1$, this implies that $Q_k(t+1) \leq \beta \ell_k + 1$, proving the induction step. The lemma then follows as $\expect{\sum_{k \in \set{K}} \ell_k Q_k(T+1)} \leq c_{\max}\sum_{k\in \set{K}} \beta (\ell_k + 1) \leq c_{\max}K(\beta c_{\max}+1)$.
\end{proof}

\begin{proof}[Proof of Theorem~\ref{thm:bacid}]
Applying Lemma~\ref{lem:known-idiosyncrasy} and Lemma~\ref{lem:known-delay} to \eqref{eq:loss-decompose} gives
\begin{align*}
\expect{\loss^{\bacid}(T)} &\leq  \loss^\star(T) + \frac{T}{\beta}+c_{\max}K(\beta c_{\max}+1). \quad\text{Using that $\beta = \sqrt{T / K}$,}\\
\reg^{\bacid}(T)&\leq 2c^2_{\max}\sqrt{KT} + Kc_{\max} \lesssim \sqrt{KT}+K.
\end{align*}
\end{proof}

\subsection{Coupling with the fluid benchmark (Lemma~\ref{lem:known-idiosyncrasy})}\label{ssec:bacid-idiosyncrasy}
Let $\{a^\star_k(t), \nu^\star_k(t)\}_{k \in \set{K},t \in [T]}$ be the optimal solution to the fluid benchmark \eqref{eq:fluid}. The problem \eqref{eq:fluid} is a linear program and multiplying the objective by $\beta$ does not impact its optimal solution; the optimal value is simply multiplied by $\beta$. Taking the Lagrangian of the scaled program on capacity constraints and letting
\[f(\{\bolds{a}(t)\}_{t \in [T]},\{\bolds{\nu}(t)\}_{t \in [T]},\bolds{u}) = \beta\sum_{t=1}^T \sum_{k \in \set{K}} \ell_k(\lambda_k(t) - a_k(t))-\sum_{t=1}^T\sum_{k\in\set{K}}u_{t,k}\left(\mu_kN(t)\nu_k(t) - a_k(t)\right),\]
where $\bolds{u} = (u_{t,k})_{t \in [T], k \in \set{K}} \geq 0$ are dual variables for the capacity constraints, the Lagrangian is
\begin{equation}
\begin{aligned}
f(\bolds{u})\coloneqq&\min_{\{\bolds{a}(t)\}_{t \in [T]},\{\bolds{\nu}(t)\}_{t \in [T]}} f(\{\bolds{a}(t)\}_{t \in [T]},\{\bolds{\nu}(t)\}_{t \in [T]},\bolds{u})  \\
\text{s.t.}\quad& a_k(t) \leq \lambda_k(t),~\nu_k(t) \geq 0,~\sum_{k' \in \set{K}} \nu_{k'}(t) \leq 1,~\forall k \in \set{K}, t\in [T].
\end{aligned}
\end{equation}
If the dual $\bolds{u}$ is fixed, the optimal solution is given by $a_k(t) = \lambda_k(t) \indic{\beta \ell_k \geq u_{t,k}}$ and $\nu_k(t) = \indic{k \in \arg \max_{k' \in \set{K}} \mu_{k'} {u_{t,k'}}}$ for any $t$. Comparing the induced optimal solution with \textsc{BACID}, \textsc{BACID} uses the queue length information $\bolds{Q}(t) = (Q_k(t))_{k \in \set{K}}$ as the dual to make decisions by setting $u_{t,k} = Q_k(t)$. Under this setting of duals, the per-period Lagrangian is 
\begin{equation}\label{eq:per-period-lag}
f_t(\bolds{a}(t), \bolds{\nu}(t), \bolds{Q}(t)) \coloneqq \beta \sum_{k \in \set{K}} \ell_k\lambda_k(t) - \left(\sum_{k \in \set{K}} a_k(t)(\beta \ell_k - Q_k(t)) + \sum_{k \in \set{K}} \nu_k(t) Q_k(t)\mu_kN(t)\right).
\end{equation}
Abusing the notation, we define the admission vector $\bolds{A}(t) = (A_k(t))_{k \in \set{K}}$ where $A_k(t) = 1$ if and only if job $t$ is of type $k$ and is admitted for human review, i.e., $A_k(t) = \Lambda_k(t) \cdot A(t)$. Moreover, recall the service vector $\bolds{\psi}(t) = (\psi_k(t))_{k \in \set{K}}$ with $\psi_k(t) = \indic{\kappa(M(t))= k}$ indicating whether humans review a type-$k$ job at time $t$. The expected Lagrangian of
\textsc{BACID} is then $\sum_{t=1}^T \expect{f_t(\bolds{A}(t), \bolds{\psi}(t), \bolds{Q}(t))}$. Our proof of Lemma~\ref{lem:known-idiosyncrasy} relies on a Lyapunov analysis of the function 
\[L(t) = \beta\sum_{t'=1}^{t-1} \ell_{\kappa(t')}\left(1 - A(t')\right) + \frac{1}{2}\sum_{k \in \set{K}} Q_k^2(t),\]
which connects the idiosyncrasy loss to the Lagrangian by the next lemma (proof in Appendix~\ref{app:lem-connect-idio-lag}).
\begin{lemma}\label{lem:connect-idio-lag}
The expected Lagrangian of \textsc{BACID} upper bounds its idiosyncrasy loss as following:
\[
\beta \expect{\sum_{t=1}^{T} \ell_{\kappa(t)}(1-A(t))} \leq \expect{L(T+1)-L(1)} \leq T + \sum_{t=1}^T \expect{f_t(\bolds{A}(t), \bolds{\psi}(t), \bolds{Q}(t))}.
\]
\end{lemma}
Our second lemma shows that the expected Lagrangian of \textsc{BACID} is close to the optimal fluid, which we prove in Section~\ref{sec:lagrang-bacid}. 
\begin{lemma}\label{lem:lagrang-bacid}
The expected Lagrangian of \textsc{BACID} is upper bounded by:
\[
\sum_{t=1}^T \expect{f_t(\bolds{A}(t), \bolds{\psi}(t), \bolds{Q}(t))}  \leq \beta \loss^\star(T).
\]
\end{lemma}
\begin{remark}
Lemma~\ref{lem:lagrang-bacid} holds even if the queue length sequence $\{\bolds{Q}(t)\}$ is generated by another policy (instead of \textsc{BACID}); $\bolds{A}(t)$ and $\bolds{\psi}(t)$ are still the decisions made by \textsc{BACID} in period $t$ given $\bolds{Q}(t)$. This generalization is useful when we apply the lemma to a learning setting in Section~\ref{sec:unknown}.
\end{remark}
\begin{proof}[Proof of Lemma~\ref{lem:known-idiosyncrasy}]
Combining Lemmas~\ref{lem:connect-idio-lag} and \ref{lem:lagrang-bacid} gives $\beta \expect{\sum_{t=1}^{T} \ell_{\kappa(t)}(1-A(t))} \leq \beta \loss^\star(T) + T$, which finishes the proof by dividing both sides by $\beta$.
\end{proof}

\subsection{Connecting Lagrangian of \textsc{BACID} with Fluid Optimal (Lemma~\ref{lem:lagrang-bacid})}\label{sec:lagrang-bacid} 
We connect the scaled primal objective $\beta \set{L}^\star(T)$ to the Lagrangian $f(\{\bolds{a}^\star(t)\}_{t \in [T]},\{\bolds{\nu}^\star(t)\}_{t \in [T]},\bolds{u})$ by setting the duals equal to the queue length:
\begin{align}
\sum_{t=1}^T \expect{f_t(\bolds{A}(t), \bolds{\psi}(t), \bolds{Q}(t))} - \beta \loss^\star(T) &= \expect{f(\{\bolds{a}^\star(t)\}_{t \in [T]},\{\bolds{\nu}^\star(t)\}_{t \in [T]},\{\bolds{Q}(t)\}_{t \in [T]})} - \beta \loss^\star(T) \label{eq:lagran-relax}\\
&\hspace{-1in}+\sum_{t=1}^T\left(\expect{f_t(\bolds{A}(t), \bolds{\psi}(t), \bolds{Q}(t))} -\expect{f_t(\bolds{a}^\star(t), \bolds{\nu}^\star(t), \bolds{Q}(t))}\right) \label{eq:bacid-step-subopt}.
\end{align}
Hence, \textsc{BACID}'s suboptimality is captured by the sum of two terms:
\eqref{eq:lagran-relax}, the difference between the Lagrangian and the primal and \eqref{eq:bacid-step-subopt}, the difference in Lagrangian compared to the optimal fluid solution when the dual is given by per-period queue length. Our proof bounds these two terms independently. 

The first step is to show that \eqref{eq:lagran-relax} is non-positive because of the definition of Lagrangian (proof in Appendix~\ref{app:lem-bound-lagrang}). 
\begin{lemma}\label{lem:bound-lagrang}
For any dual
$\bolds{u}=(u_{t,k})_{t \in [T], k \in \set{K}} \geq 0$, 
$f(\{\bolds{a}^\star(t)\}_{t\in[T]},\{\bolds{\nu}^\star(t)\}_{t\in[T]},\bolds{u}) - \beta \set{L}^\star(T) \leq 0.$
\end{lemma}
Our next lemma (proof in Appendix~\ref{app:lem-bacid-mw}) shows that  \eqref{eq:bacid-step-subopt} is nonpositive as \textsc{BACID} explicitly optimizes the per-period Lagrangian based on $\bolds{Q}(t)$.
\begin{lemma}\label{lem:bacid-mw}
For every period $t$,
$\expect{f_t(\bolds{A}(t),\bolds{\psi}(t),\bolds{Q}(t)) \mid \bolds{Q}(t)} - f_t(\bolds{a}^\star(t),\bolds{\nu}^\star(t),\bolds{Q}(t)) \leq 0$.  
\end{lemma}

\begin{proof}[Proof of Lemma~\ref{lem:lagrang-bacid}]
The proof follows by applying Lemmas~\ref{lem:bound-lagrang} and \ref{lem:bacid-mw} in \eqref{eq:lagran-relax} and \eqref{eq:bacid-step-subopt}.
\end{proof}

\section{\textsc{BACID} with Learning: Optimism and Selective Sampling}\label{sec:unknown}
In this section, we extend our approach to the setting where the =cost distributions $\{\set{F}_k\}_{k \in \set{K}}$ are initially unknown and the algorithm's classification, admission, and scheduling decisions should account for the need to learn these parameters online.

We first restrict our attention to \textsc{BACID}
admission rule: defer a job of type $k$ to human review if and only if $\beta \ell_k \geq Q_k(t)$ where $\ell_k = \min(\ell_k^+,\ell_k^-)$. When the cost distributions $\{\set{F}_k\}_{k \in \set{K}}$ are unknown, we cannot directly compute $\ell_k$ and we need to instead use some estimate for $\ell_k$. A canonical way to resolve this problem in, e.g., bandits with knapsacks \citep{agrawal2019bandits} is to use an optimistic estimate $\bar{\ell}_k(t)$. In particular, recall that $\set{D}(t)$ is the dataset collected in the first $t-1$ period. For each type~$k$, we can compute the sample-average estimation of $\ell_k^{+}$ and $\ell_k^{-}$ by $\hat{\ell}_k^+(t)=\frac{\sum_{(j,C_j)\in \set{D}(t) \colon \kappa(j) = k} C^+_j}{n_k(t)}$ and $\hat{\ell}_k^-(t)=\frac{\sum_{(j,C_j)\in \set{D}(t) \colon \kappa(j) = k} (C^-_j)}{n_k(t)}$ where  $n_k(t) = \sum_{(j,C_j) \in \set{D}(t)} \indic{\kappa(j)=k}$ is the number of samples from type $k$; if $n_k(t)=0$, we set $\hat{\ell}_k^+(t)=\hat{\ell}_k^-(t) = 0$. The sample-average of the mean cost $c_k$ is given by $\hat{c}_k(t) = \hat{\ell}_k^+(t) - \hat{\ell}_k^-(t)$. We can then compute a confidence bound of $c_k$ and an upper confidence bound of $\ell_k$ by
\begin{equation}\label{eq:conf-h}
\ubar{c}_k(t) = \max\left(-c_{\max},\hat{c}_k(t) - \sigma_{\max}\sqrt{\frac{8\ln t}{n_k(t)}}\right),~\bar{c}_k(t) = \min\left(c_{\max},\hat{c}_k(t) + \sigma_{\max}\sqrt{\frac{8\ln t}{n_k(t)}}\right)
\end{equation}
\begin{equation}\label{eq:conf-r}
\bar{\ell}_k(t) = \min\left(c_{\max},\min\left(\hat{\ell}_k^+(t), \hat{\ell}_k^-(t)\right) + 4\sigma_{\max}\sqrt{\frac{\ln t}{n_k(t)}}\right)
\end{equation}
where we recall that $\sigma_{\max}^2$ is the known variance proxy of cost $C_t$ and $c_{\max} \geq 1$ is a known upper bound on $\expectsub{k}{|C_t|}$ for any $t$, thus an upper bound on $|c_k|,\ell_k^{+}, \ell_k^{-},$ and $\ell_k$. We can then admit a job if and only if $\beta \bar{\ell}_k(t) \geq Q_k(t)$ (for type $\perp$, we set $\bar{\ell}_{\perp}(t) \equiv 0$). 

\subsection{Why Optimism-Only is Insufficient for Classification}\label{sec:opti-fail}
This optimistic admission rule gives a natural adaptation of \textsc{BACID} which we term \textsc{BACID.UCB}: 
\begin{enumerate}
\vspace{-\topsep}
\item reject job $t$ if and only if $\hat{c}_{\kappa(t)}(t) > 0$ (similar to Line~\ref{line:bacid-classify} in Algorithm~\ref{algo:bacid}); 
\vspace{-\topsep}
\item admit job $t$ if $\beta \bar{\ell}_{\kappa(t)}(t) \geq Q_{\kappa(t)}(t)$ (similar to Line~\ref{line:bacid-admit} in Algorithm~\ref{algo:bacid}); 
\vspace{-\topsep}
\item and schedule a type-$k$ job where $k$ maximizes $\mu_{k}Q_{k}(t)$ (same as Line~\ref{line:bacid-schedule} in Algorithm~\ref{algo:bacid}).\vspace{-\topsep}
\end{enumerate}
Such optimism-only heuristics generally work well in a constrained setting, such as bandits with knapsacks. The intuition is that assuming a valid upper confidence bound, we always admit a job that would have been admitted by $\textsc{BACID}$ with known cost distributions $\{\set{F}_k\}_{k \in \set{K}}$. If we admit a job that would not have been admitted by $\textsc{BACID}$, we obtain one more sample; this shrinks the confidence interval which, in turn, limits the number of mistakes and leads to efficient learning. 

Interestingly, this intuition does not carry over to our setting due to the additional error in downstream classification decisions of jobs that are not admitted, for which the sign of $\hat{c}_k(t)$ may differ from that of $c_k$. To illustrate this point, consider the following instance with $K=2$ types of jobs: texts (type-$1$) and videos (type-$2$). A text job has equal probability to have cost $\pm 1$ while a video job has cost equal to $0.01$ with probability $0.95$ and cost equal to $-0.01$ with probability $0.05$. Videos have a smaller absolute cost than texts and appear in the platform only after sufficient text jobs are in the review system; see Figure~\ref{fig:ucb-fail} for the exact instance parameters.

Proposition~\ref{prop:fail-to-learn} shows that \textsc{BACID.UCB} with initial knowledge of $\set{F}_1$ and knowledge that $|C_t| = 0.01$ conditioning on $\kappa(t) = 2$ does not review video jobs (for which $\set{F}_2$ is unknown) with high probability. Therefore, there is no video job in the dataset for the entire horizon, although there are $\Omega(T)$ arrivals of them. The proof, given in Appendix~\ref{app:prop-fail-to-learn}, relies on the fact that the queue of text jobs is always longer than that of video jobs and thus no video job will be scheduled for human review. As a result, when classifying video jobs, since there is no data, the algorithm estimates $\hat{c}_2(t) = 0$, and accepts all videos even if they are likely harmful, incurring $\Omega(T)$ regret.\footnote{We assume accepting a job when $\hat{c}_k(t) = 0$. If the algorithm rejects the job, the same issue exists by setting the cost distribution such that a video job has negative cost with probability $0.95$ and positive cost with probability $0.05$.}  

\begin{proposition}\label{prop:fail-to-learn}
There exists a setting such that for any $\beta > 100$ and $T \in (\beta / 2,\exp(\beta / 1152)]$, with probability at least $1 - 2/T$, there is no video job in the dataset $\set{D}(T+1)$ under $\textsc{BACID.UCB}$.
\end{proposition} 
\begin{figure}[!h]
  \centering
  \scalebox{0.75}{\includegraphics{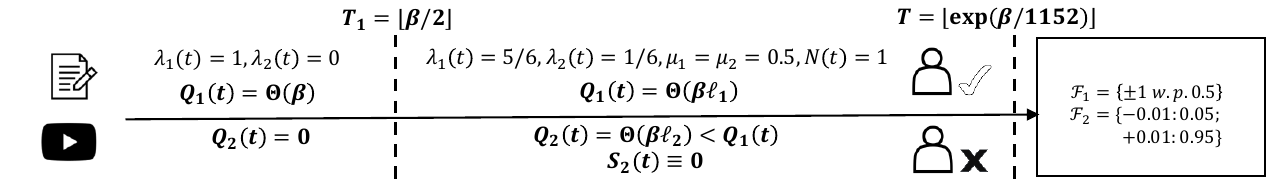}}
  \caption{An example where \textsc{BACID.UCB} fails to correctly classify video (type-$2$) jobs. Humans never review a video because the corresponding queue length is much smaller than that of texts.}
  \label{fig:ucb-fail} 
\end{figure}
\textsc{BACID.UCB} incurs linear regret in the above example as its decisions are inherently myopic to current rounds and disregard the importance of labels towards classification decisions in future rounds. Broadly speaking, optimism-only heuristics focus on the most optimistic estimate on the contribution of each action (in our setting, the admission decision) subject to a confidence interval and then select the action that maximizes this optimistic contribution in the current round. In our example, if we can only admit one type, the text jobs have larger idiosyncrasy loss due to their costs and thus the benefit of admitting a text job outweighs even the most optimistic estimate on the idiosyncrasy loss of video jobs in the current round. This mimics the admission rule of \textsc{BACID} which operates with known parameters and would never review a video job. Although always prioritizing text jobs is myopically beneficial in the current round, this means that we collect no new video data, thus harming our classification performance in the long run.

Though our proof of Proposition~\ref{prop:fail-to-learn} concerns the specific form of $\textsc{BACID.UCB}$, in Appendix~\ref{app:sim-fail-to-learn} we simulate several variants to $\textsc{BACID.UCB}$ in a setting similar to Figure~\ref{fig:ucb-fail}. The considered variants  include one with optimism in scheduling, one with discounted samples, and one with an initial exploration phase. We observe that the first two variants fail to collect enough samples of one type of jobs to ensure correct AI classification (similar to what happen for the video jobs in Figure~\ref{fig:ucb-fail}). Although initial exploration eases this issue, the third variant  leads to much higher loss and is not amenable to a setting with non-stationary arrivals. Our results thus highlight the robustness of the observation in Proposition~\ref{prop:fail-to-learn} that an optimism-only approach is insufficient for classification.

\subsection{Label-Driven Admission and Forced Scheduling for Classification}
The inefficiency of \textsc{BACID.UCB} suggests the need to complement the myopic nature of optimism-only approaches by a forward-looking exploration that enhances classification decisions. Our algorithm, \textsc{Optimistic and Label-driven Admission for Balanced Classification, Idiosyncrasy and Delay} or $\bacidol$ in short (Algorithm~\ref{algo:bacidol}) incorporates this forward-looking exploration and evades the shortcomings of optimism-only approaches. When a new job $t$ arrives, the platform rejects it if and only if the empirical average cost $\hat{c}_{\kappa(t)}(t)$ is positive. Unlike $\textsc{BACID}$ which assumes knowledge of $c_{\kappa(t)}$, when $\ubar{c}_{\kappa(t)}(t) < 0 < \bar{c}_{\kappa(t)}(t)$, we cannot confidently infer the sign of $c_{\kappa(t)}$ from the sign of $\hat{c}_{\kappa(t)}(t)$. 

To enhance future classification decisions on those jobs, we complement the optimism-based admission (Line~\ref{line:bacidol-admit}) with \emph{label-driven admission} (Line~\ref{line:bacidol-lda}). Specifically, we maintain a new \emph{label-driven queue} $\set{Q}^{\textsc{ld}}(t)$ and add the job to that queue ($E(t) = 1$) if (i) the queue is empty and (ii) there is high uncertainty on the sign of $\hat{c}_{\kappa(t)}(t)$, i.e., $\ubar{c}_{\kappa(t)}(t) \leq -\gamma$ and $\gamma \leq \bar{c}_{\kappa(t)}(t)$. For (i), we require the label-driven queue to be empty to avoid over-admission and to control its length over the horizon. For (ii), the parameter $\gamma$ avoids wasting reviewing capacity on \emph{indifferent} jobs,  for whom the difference in misclassification loss between acceptance and rejection is small (at most $\gamma$). If the job is admitted by the optimism-based admission ($E(t) = 0$, $\beta\bar{\ell}_{\kappa(t)}(t) \geq Q_{\kappa(t)}(t)$), we denote $A(t) = 1$. We stress that $Q_{\kappa(t)}(t)$ includes only jobs in the review queue $\set{Q}(t)$ and not in $\set{Q}^{\ld}(t)$. 

We use \emph{forced scheduling} to prioritize reviews for the label-driven queue. If $\set{Q}^{\ld}(t)$ is not empty, we review a job from $\set{Q}^{\ld}(t)$. Otherwise, we follow the \textsc{MaxWeight} scheduling (as in \textsc{BACID}): select a type $k$ that maximizes $\mu_{k}Q_{k}(t)$ and review the earliest waiting job in the review queue $\set{Q}(t)$ that has type $k$. We let $\psi_{k}(t) = 1$ if a type-$k$ job in $\set{Q}(t)$ is scheduled to review in period~$t$. 

\begin{algorithm}[H]
\LinesNumbered
\DontPrintSemicolon
  \caption{
  \textsc{Optimistic and Label-driven $\bacid$ ($\bacidol$)}}\label{algo:bacidol}
  \KwData{$T, \{\mu_k\}_{k \in \set{K}}, \sigma_{\max}, c_{\max}$} 
  \textbf{Admission parameters:} 
  $\beta = \sqrt{T / K}, \gamma = (T/(K\ln T))^{-1/3}$
  
  \For{$t = 1$ \KwTo $T$}{
    Observe a new job of type $\kappa(t)$ \\
    \lIf{$\hat{c}_{\kappa(t)}(t) > 0$}{$Y(t) \gets -1$ 
     \textbf{else} {$Y(t) \gets 1$}
    \tcp*[f]{Empirical Classification} \label{line:bacidol-classify}}
    \tcc{Label-Driven and Optimistic Admission}
Calculate $\ubar{c}_{\kappa(t)}(t),\bar{c}_{\kappa(t)}(t),\bar{\ell}_{\kappa(t)}(t)$ by \eqref{eq:conf-h} and \eqref{eq:conf-r} \label{line:bacidol-esti}\;
  \lIf{$\ubar{c}_{\kappa(t)}(t) < -\gamma < \gamma < \bar{c}_{\kappa(t)}(t)$ \emph{\textbf{and}} $|\set{Q}^{\ld}(t)|=0$}{
  $E(t) \gets 1$~\textbf{else} $E(t)=0$\label{line:bacidol-lda}}
    
    \lIf{$E(t) = 0$ \emph{\textbf{and}} $\beta \cdot \bar{\ell}_{\kappa(t)}(t) \geq Q_{\kappa(t)}(t)$}{
      $A(t) = 1$ \textbf{else} {$A(t) = 0$} 
    \label{line:bacidol-admit}
    }
\tcc{Forced Scheduling and \textsc{MaxWeight} Scheduling}
    
    \lIf{$\set{Q}^{\ld}(t) \neq \emptyset$}{$M(t)\gets$ the job in $\set{Q}^{\ld}(t)$} \label{line:bacidol-prio}
    \lElse{
    $k \gets \arg\max_{k' \in \mathcal{K}} \mu_{k'} \cdot Q_{k'}(t)$,~$M(t) \gets \text{first job in } \set{Q}_{k}(t) \text{ if any}$
    }
    \lIf{\emph{the review finishes (with probability $N(t) \cdot \mu_{\kappa(M(t))}$)}}
    {
      $\mathcal{S}(t) = M(t)$ \textbf{else} 
      $\mathcal{S}(t) = \perp$ 
    }
  }
\end{algorithm}
Our main result (Theorem~\ref{thm:bacidol}) is that, setting $\beta = \sqrt{T / K}, \gamma = (T/(K\ln T))^{-1/3}$, $\bacidol$ achieves a regret of $\tilde{O}(\sqrt{T})$ when there is a large margin $\eta\coloneqq \min(1,\min_{k \in \set{K}} |c_k|)$ and the number of types $K$ is small. In particular, the guarantee matches the lower bound $\Omega(\sqrt{T})$ in its dependence on $T$.  Even if there is no margin (so classification is difficult), $\bacidol$ still obtains a regret of $\tilde{O}(T^{2/3})$; this worse dependence is due to the classification difficulty and is unavoidable (Theorem~\ref{thm:bacidol-lowerbound}). 

\begin{theorem}\label{thm:bacidol}
The regret of $\bacidol$ is upper bounded by 
\[\reg^{\bacidol}(T) \lesssim 
K\sqrt{T\ln T} + \min\left(K\ln T / \eta^2, T^{2/3}(K\ln T)^{1/3}\right).\]
\end{theorem}
The below theorem shows that the worst-case $\tilde{O}(T^{2/3})$ regret is tight. The proof of this theorem is similar to standard arguments from multi-armed bandits and is deferred to Appendix~\ref{app:thm-bacidol-lowerbound}. 
\begin{theorem}\label{thm:bacidol-lowerbound}
When cost distributions are unknown, for any $T \geq 8$, there exists a setting such that $\reg^{\pi}(T)  \geq \frac{T^{2/3}}{36}$ for any feasible policy $\pi$.
\end{theorem}

Our proof of Theorem~\ref{thm:bacidol} relies on a loss decomposition of $\loss^{\bacidol}(T)$ that is similar to \eqref{eq:loss-decompose} 
but also captures the possible incorrect classification decisions. The incurred loss per period is $\left(Y(t)^+ \ell_k^{+}\right) + \left(Y(t)^-  \ell_k^{-}\right)$ instead of $\ell_k = \min(\ell_k^{+},\ell_k^{-})$ (recall that $Y(t) = 1$ means acceptance of the job and $Y(t) = -1$ means rejection of the job). 

Given that $\ell_k^{+}, \ell_k^{-} \leq c_{\max}$, we can upper bound $\expect{\loss^{\bacidol}(T)}$ by 
\begin{align}
&\hspace{0.1in}\underbrace{\expect{\sum_{t=1}^T \left(Y(t)^+(\ell_{\kappa(t)}^{+} - \ell_{\kappa(t)}) + Y(t)^-(\ell_{\kappa(t)}^{-}-\ell_{\kappa(t)})\right)(1 - A(t) - E(t))}}_{\text{Classification Loss}} \nonumber\\
&+  \underbrace{\expect{\sum_{t=1}^T  \ell_{\kappa(t)}(1 - A(t) - E(t))}}_{\text{Idiosyncrasy Loss}} +\underbrace{c_{\max}\expect{\sum_{k \in \set{K}} Q_k(T+1) + \left|\set{Q}^{\ld}(T+1)\right|}}_{\text{Relaxed Delay Loss}}.\label{eq:learning-loss-decompose}
\end{align}
Comparing \eqref{eq:learning-loss-decompose} with \eqref{eq:loss-decompose}, there is a new term on classification loss, which captures the loss when the classification $Y(t)$ is incorrect. We adjust the other terms to capture label-driven admission. The last term is called \emph{relaxed} delay loss because we upper bound the $\ell_k$ term in \eqref{eq:loss-decompose} by $c_{\max}.$

To bound the losses, we define  the minimum per-period review rate as $\hat{\mu}_{\min} = \min_{t \in [T], k \in \set{K}} N(t)\mu_k$. We focus on $T \geq 3$ and $K \leq T$ (the bound in Theorem~\ref{thm:bacidol} becomes trivial if $K > T$). Our bounds (Lemmas~\ref{lem:bacidol-delay},~\ref{lem:bacidol-class} and \ref{lem:bacidol-idio}) hold for general $\beta, \gamma$, and are proven in \ref{sec:bacidol-class} and \ref{sec:bacidol-idio} respectively. The lemmas are defined for $T \geq 3$ and $K \leq T$. The proof of Theorem~\ref{thm:bacidol} (provided in Appendix \ref{app:thm-bacidol}) directly combines the lemmas.
\begin{lemma}\label{lem:bacidol-delay}
For any $\beta \geq 1 / c_{\max}$, the Relaxed Delay Loss of $\bacidol$ is at most $3c^2_{\max} K\beta .$
\end{lemma}
\begin{proof}
For the review queue $\set{Q}$, we admit a type-$k$ job into $\set{Q}$ only if $\beta \bar{\ell}_k(t)\leq Q_k(t)$ and $\bar{\ell}_k(t) \leq c_{\max}$. As in the proof of Lemma~\ref{lem:known-delay}, the queue length is bounded by $Q_k(t) \leq \beta c_{\max} + 1 \leq 2\beta c_{\max}$. Moreover, the label-driven queue has at most one job by our admission rule on Line~\ref{line:bacidol-lda} of Algorithm~\ref{algo:bacidol}. The relaxed delay loss is thus upper bounded by $c_{\max}K(2\beta c_{\max}) + 1 \leq 3c^2_{\max}K\beta.$
\end{proof}

\begin{lemma}\label{lem:bacidol-class}
For any $0 < \gamma \leq 1$, the Classification Loss of $\bacidol$ is at most 
\[
\gamma \indic{\gamma \geq \eta} T + \frac{38c_{\max}K\sigma^2_{\max}\ln T}{\max(\eta,\gamma)^2\hat{\mu}_{\min}} + 5Kc_{\max}.
\]
\end{lemma}
\begin{lemma}\label{lem:bacidol-idio}
For any $\beta \geq 1 / c_{\max}, \gamma \in (0,1]$, the Idiosyncrasy Loss of $\bacidol$ is at most
\[
\loss^\star(T) + \frac{T}{\beta}+\frac{76c_{\max}K\sigma^2_{\max}\ln T}{\max(\eta,\gamma)^2\hat{\mu}_{\min}}+20Kc_{\max}\left(\sqrt{T\ln T} + \beta c_{\max}\right).
\]
\end{lemma}

\subsection{Bounding the Classification Loss (Lemma~\ref{lem:bacidol-class})}\label{sec:bacidol-class}
A key ingredient in the proofs of Lemmas  \ref{lem:bacidol-class} and \ref{lem:bacidol-idio} is to provide an upper bound on the number of periods that the label-driven queue is non-empty. Given that $\set{Q}^{\ld}(t)$ has at most one job at any period $t$, letting $Q^{\ld}(t) = |\set{Q}^{\ld}(t)|$, this is equal to $\expect{\sum_{t=1}^T Q^{\ld}(t)}$. This quantity allows to bound 1) the number of jobs which we do not admit into the label-driven queue despite not being able to confidently estimate the sign of their expected cost and 2) the number of periods that we do not follow \textsc{MaxWeight} scheduling.

Our first lemma connects $\expect{\sum_{t=1}^T Q^{\ld}(t)}$ to the number of jobs admitted to the label-driven queue $\set{Q}^{\textsc{ld}}(t)$. The proof (Appendix~\ref{app:lem-connect-et-qe}) relies on two facts: (1) the queue $\set{Q}^{\textsc{ld}}(t)$ has length at most one; and (2) each job stays in the queue for at most $1 / \hat{\mu}_{\min}$ periods in expectation. 
\begin{lemma}\label{lem:connect-et-qe}
The total length of the label-driven queue is at most $\expect{\sum_{t=1}^T Q^{\ld}(t)} \leq \frac{\expect{\sum_{t=1}^TE(t)}}{\hat{\mu}_{\min}}$.
\end{lemma}
Our second lemma applies concentration bounds (Appendix~\ref{app:lem-prob-good-event}) to show that the event $\set{E}_{k,t} = \{c_k \in [\ubar{c}_k(t),\bar{c}_k(t)],\ell_k \leq \bar{\ell}_k(t)\}$ (confidence bounds are valid) holds with high probability. A challenge in establishing this lemma is to show that $C_t^+$ and $C_t^-$ are also sub-Gaussian given that $C_t$ is sub-Gaussian, which we establish with techniques from \cite{kontorovich2014concentration}.
\begin{lemma}\label{lem:prob-good-event}
For any $k,t$, the confidence bounds \eqref{eq:conf-h},\eqref{eq:conf-r} are valid with probability $\Pr\{\set{E}_{k,t}\} \geq 1 - 4t^{-3}$.
\end{lemma}

Our next lemma (proof in Appendix~\ref{app:lem-bound-sum-ek}) bounds $\expect{\sum_{t=1}^T E(t)}$ via considering the type-$k$ confidence interval when the last type-$k$ job is admitted to the label-driven queue $\set{Q}^{\textsc{ld}}$.
\begin{lemma}\label{lem:bound-sum-ek}
The label-driven queue admits at most
$\expect{\sum_{t=1}^T E(t)} \leq \frac{38K\sigma^2_{\max}\ln T}{\max(\eta,\gamma)^2}$ jobs.
\end{lemma}

To show Lemma~\ref{lem:bacidol-class} we also need the below lemma (proof in Appendix~\ref{app:lem-bacidol-class-err}), which establishes a bound on the type-$k$ period-$t$ classification loss \[Z_k(t) = \Lambda_k(t)\left(Y(t)^+(\ell_k^{+} - \ell_k) + Y(t)^-(\ell_k^{-}-\ell_k)\right)(1 - A(t) - E(t)).\]

\begin{lemma}\label{lem:bacidol-class-err}
For any type $k$ and period $t$,  $Z_k(t)\indic{\set{E}_{k,t}} \leq \left(\gamma \indic{\gamma \geq \eta} + c_{\max}Q^{\ld}(t)\right)\Lambda_k(t).$
\end{lemma}
\begin{proof}[Proof of Lemma~\ref{lem:bacidol-class}]
By definition, the total classification loss is $\expect{\sum_{t=1}^T \sum_{k \in \set{K}} Z_k(t)}$. As a result,
\begin{align*}
\expect{\sum_{t=1}^T \sum_{k \in \set{K}} Z_k(t)} &\leq \expect{\sum_{t=1}^T \sum_{k \in \set{K}} Z_k(t)\indic{\set{E}_{k,t}}} + c_{\max}\sum_{t=1}^T \sum_{k \in \set{K}} \Pr\{\set{E}_{k,t}^c\} \tag{$Z_k(t)\leq c_{\max}$}\\
&\leq \expect{\sum_{t=1}^T \sum_{k \in \set{K}} Z_k(t)\indic{\set{E}_{k,t}}} + Kc_{\max}\sum_{t=1}^T \frac{4}{t^3} \tag{By Lemma~\ref{lem:prob-good-event}}\\
&\leq \expect{\sum_{t=1}^T \left(\gamma \indic{\gamma \geq \eta} + c_{\max}Q^{\ld}(t)\right)} + 5Kc_{\max}\tag{By Lemma~\ref{lem:bacidol-class-err}} \\
&\hspace{-0.5in}\leq \gamma \indic{\gamma \geq \eta} T + \frac{38c_{\max}K\sigma^2_{\max}\ln T}{\max(\eta,\gamma)^2\hat{\mu}_{\min}} + 5Kc_{\max}.\tag{By Lemmas~\ref{lem:connect-et-qe},~\ref{lem:bound-sum-ek}}
\end{align*}
\end{proof}

\subsection{Bounding 
the Idiosyncrasy Loss (Lemma~\ref{lem:bacidol-idio})}\label{sec:bacidol-idio}
We follow a similar strategy as in the proof of Lemma~\ref{lem:known-idiosyncrasy}, but we encounter three new challenges due to the algorithmic differences between \textsc{BACID} and $\bacidol$. 

Our first challenge arises because of the additional label-driven admission. Without the label-driven admission, we would admit a type-$k$ job to the review queue $\set{Q}(t)$ in period $t$ if its (optimistic) idiosyncrasy loss outweighs the delay loss, i.e., $\bar{A}_k(t) = \Lambda_k(t)\indic{\beta \bar{\ell}_k(t) \geq Q_k(t)}$ is equal to one. However, we now only admit such a job to $\set{Q}(t)$ if it is not admitted to the label-driven queue $\set{Q}^{\ld}(t)$, i.e., the real admission decision is $A_k(t) = \bar{A}_k(t)(1-E(t))$. The following Lyapunov function defined based on only the length of $\set{Q}(t)$ accounts for this difference and offers an analogue of Lemma \ref{lem:connect-idio-lag}, which we prove in Appendix~\ref{app:lem-bacidol-idio-lag}. 
\[
L(t) = \beta\sum_{t'=1}^{t-1}\ell_{\kappa(t')}(1-A(t') - E(t'))+\frac{1}{2}\sum_{k \in \set{K}}Q_k^2(t).
\]
\begin{lemma}\label{lem:bacidol-idio-lag}
The expected Lagrangian of $\bacidol$ upper bounds the idiosyncrasy loss by:
\[
\beta\expect{\sum_{t=1}^T \ell_{\kappa(t)}(1 - A(t) - E(t))} \leq \expect{L(T+1) - L(1)} \leq T + \sum_{t=1}^T \expect{f_t(\bolds{\bar{A}}(t),\bolds{\psi}(t),\bolds{Q}(t))}.
\]
\end{lemma}
We next bound the right hand side of Lemma~\ref{lem:bacidol-idio-lag}. By Lemma~\ref{lem:lagrang-bacid}, we know how to bound this quantity when admission and scheduling decisions are made according to \textsc{BACID}, i.e., $A_k^{\bacid}(t) = \indic{\beta \ell_k \geq Q_k(t)}\Lambda_k(t)$ and $\psi_k^{\bacid}(t) = \indic{k = \arg\max_{k' \in \set{K}} \mu_{k'} Q_{k'}(t)}$. To bound the Lagrangian under $\bacidol$, we connect it to its analogue under $\bacid$ via the \emph{regret in Lagrangian}:
\begin{align*}
\textsc{RegL}(T) &=\expect{\sum_{t=1}^T f_t(\bolds{\bar{A}}(t), \bolds{\psi}(t), \bolds{Q}(t)) - f_t(\bolds{A}^{\bacid}(t), \bolds{\psi}^{\bacid}(t), \bolds{Q}(t))} \\
&\hspace{-0.5in}= \underbrace{\expect{\sum_{t=1}^T \sum_{k\in\set{K}} (A_k^{\bacid}(t) - \bar{A}_k(t))(\beta \ell_k-Q_k(t))}}_{\textsc{RegA}(T)} + \underbrace{\expect{\sum_{t=1}^T \sum_{k \in \set{K}} (\psi_k^{\bacid}(t) - \psi_k(t))Q_k(t)\mu_k N(t)}}_{\textsc{RegS}(T)}.
\end{align*}
Our second challenge is to upper bound the regret in admission $\textsc{RegA}(T)$. This regret arises because, unlike $\bacid$ which admits based on the ground-truth per-period idiosyncrasy loss $\ell_k$, $\bacidol$ uses the optimistic estimation $\bar{\ell}_k(t)$. Different from works in bandits with knapsacks, to upper bound this regret, the following lemma also needs to account for the endogenous queueing delay in label acquisition. The proof in Appendix~\ref{app:lem-bacidol-rega} addresses this challenge by connecting the number of unobserved labels with the queue length of the system and using the fact that the queue length is always bounded under our admission decisions. 

\begin{lemma}\label{lem:bacidol-rega}
The regret in admission is $\textsc{RegA}(T) \leq 20K\beta^2c^2_{\max}+16K\beta c_{\max}\sqrt{T\ln T}. $
\end{lemma}
Our third challenge is to bound the regret in scheduling $\textsc{RegS}(T)$. Unlike $\bacid$ which schedules based on \textsc{MaxWeight}, $\bacidol$ prioritizes jobs in the $\set{Q}^{\ld}(t)$. The effect of those deviations is bounded in the next lemma (proof in Appendix~\ref{app:lem-bacidol-regs}), using Lemmas \ref{lem:connect-et-qe} and \ref{lem:bound-sum-ek}.
\begin{lemma}\label{lem:bacidol-regs}
The regret in scheduling is  $\textsc{RegS}(T) \leq \frac{76\beta c_{\max}K\sigma^2_{\max}\ln T}{\max(\eta,\gamma)^2\hat{\mu}_{\min}}$.
\end{lemma}
\begin{proof}[Proof of Lemma~\ref{lem:bacidol-idio}]
By Lemma~\ref{lem:bacidol-idio-lag}, the idiosyncrasy loss is upper bounded by
\begin{align*}
\beta\expect{\sum_{t=1}^T \ell_{\kappa(t)}(1 - A(t) - E(t))} &\leq T + \sum_{t=1}^T \expect{f_t(\bolds{\bar{A}}(t), \bolds{\psi}(t), \bolds{Q}(t))} \\
&= T + \sum_{t=1}^T \expect{f_t(\bolds{A}^{\bacid}(t), \bolds{\psi}^{\bacid}(t),\bolds{Q}(t))} + \textsc{RegL}(T)\\
&\leq T + \beta \loss^\star(T) + \textsc{RegA}(T) + \textsc{RegS}(T) \tag{By Lemma~\ref{lem:lagrang-bacid}} \\
&\hspace{-1in}\leq T + \beta \loss^\star(T)+ 20K\beta^2c^2_{\max} + 16K\beta c_{\max}\sqrt{T\ln T} + \frac{76\beta c_{\max}K\sigma^2_{\max}\ln T}{\max(\eta,\gamma)^2\hat{\mu}_{\min}} \tag{By Lemmas~\ref{lem:bacidol-rega} and \ref{lem:bacidol-regs}} \\
&\hspace{-1in}\leq T + \beta \loss^\star(T)+\frac{76\beta c_{\max}K\sigma^2_{\max}\ln T}{\max(\eta,\gamma)^2\hat{\mu}_{\min}}+20K\beta c_{\max}\left(\sqrt{T\ln T} + \beta c_{\max}\right).
\end{align*}
Dividing both sides of the inequality by $\beta$ gives the desired result.
\end{proof}

\section{\textsc{OLBACID} with Type Aggregation and Contextual Learning}\label{sec:contextual}
In this section, we design an algorithm whose performance does not deteriorate with the number of types $K$. We adopt a linear contextual structure assumption that is common in the bandit literature \citep{LiCLS10,ChuLRS11,Abbasi-YadkoriPS11}, and in online content moderation practice \citep{Avadhanula2022}. In particular, we assume that each job comes with a $d$-dimensional feature vector $\bphi_k \in \mathbb{R}^d$ associated with its type $k$. The expected loss of accepting a job, $\ell_k^+ = \expectsub{k}{C_t^+}$, and the expected loss of rejecting a job, $\ell_k^- = \expectsub{k}{-C_t^-}$, both satisfy a linear model such that $\ell_k^+ = \bphi_k^{\trans}\btheta^{\star,+}$ and $\ell_k^- = \bphi_k^{\trans}\btheta^{\star,-}$ with two fixed unknown vectors $\btheta^{\star,+},\btheta^{\star,-} \in \mathbb{R}^d$. This model also implies $c_k = \bphi_k^{\trans}\bolds{c}^{\star}$ with $\bolds{c}^{\star} = \btheta^{\star,+} - \btheta^{\star,-}$. Letting $\bolds{L}_k$ be the vector $[\ell_k^+, \ell_k^-]$ and $\bTheta^\star = [\btheta^{\star,+}, \btheta^{\star,-}]$, the linear model is equivalent to $\bolds{L}_k = \bphi_k^{\trans}\bTheta^\star.$ We assume that the Euclidean norms of $\btheta^{\star,+},\btheta^{\star,-}$ and feature vectors $\bphi_k$ are upper bounded by a known constant~$U\geq 1$.

Two challenges arise when the number of types is large: \emph{over-exploration} and \emph{over-admission}. First, it is no longer amenable to maintain separate confidence intervals for each type because there are few arrivals of each type and the algorithm may end up exploring all arrivals to learn. Second, using queue lengths of each type as estimates of the delay loss for admission leads to over-admission of arrivals as the delay loss is underestimated when there are many types. We next elaborate how to address these two challenges.

\subsection{Contextual Learning for Efficient Exploration}
We incorporate the contextual information of types to design an efficient exploration mechanism. Our approach is motivated by contextual bandit techniques \citep{Abbasi-YadkoriPS11}, but it is tailored to handle the impact of delayed feedback; see a detailed comparison in Appendix~\ref{app:comparison}.

The algorithm constructs confidence intervals around $c_k$ as follows. First, we maintain confidence sets for $\btheta^{\star,+}, \btheta^{\star,-}$ based on the collected dataset $\set{D}(t)$. Recalling that $\set{S}(\tau)$ is the reviewed job for period $\tau$ (or $\perp$ when no job is reviewed), the dataset in period $t$ is
$\{(\bolds{X}_{\tau},\bolds{Z}_{\tau})\}_{\tau \leq t}$ with feature $\bolds{X}_{\tau} = \bolds{\phi}_{\kappa}(\set{S}(\tau))$ and observation $\bolds{O}_{\tau} = \begin{pmatrix}C^+_{\set{S}(\tau)} \\ C^-_{\set{S}(\tau)} \end{pmatrix}$. We use the ridge estimator with regularization parameter~$\xi$, 
\begin{equation}\label{eq:regression}
\begin{aligned}
\hat{\bTheta}(t) = [\hat{\btheta}^+(t),\hat{\btheta}^-(t)] &=  \bar{\bolds{V}}_t^{-1}[\bolds{X}_1,\ldots,\bolds{X}_{t}][\bolds{O}_1,\ldots,\bolds{O}_{t}]^{\trans} \\
&\text{where}~ \bar{\bolds{V}}_{t} = \xi \bolds{I} + \sum_{\tau \leq t} \bolds{X}_{\tau}\bolds{X}_{\tau}^{\trans},~ \hat{\bTheta}(0) = [\bolds{0},\bolds{0}].
\end{aligned}
\end{equation}
We also construct an estimator for $\bolds{c}^{\star}$ by $\hat{\bolds{c}}(t) = \hat{\btheta}^+(t) - \hat{\btheta}^-(t).$

Recalling that $\sigma_{\max}$ is the variance proxy of cost distributions and  $U$ upper bounds the Euclidean norm of $\btheta^{\star,+},\btheta^{\star,-}$ and any feature $\bphi_k$, we define the confidence sets with a confidence level $\delta$ by 
\begin{equation}\label{eq:contextual-def-conf}
\begin{aligned}
\set{C}^+_t \coloneqq \left\{\btheta^{+} \in \mathbb{R}^{d} \colon \|\hat{\btheta}^{+}(t) - \btheta^{+}\|_{\bar{\bolds{V}}_t} \leq B_{\delta}(t)\right\}, &\quad \set{C}^-_t \coloneqq \left\{\btheta^{-}\in \mathbb{R}^{d} \colon \|\hat{\btheta}^{-}(t) - \btheta^{-}\|_{\bar{\bolds{V}}_t} \leq B_{\delta}(t)\right\} \\ 
\text{where }
B_{\delta}(t) &\coloneqq \sigma_{\max}\sqrt{2d\ln\left(\frac{1+tU^2/\xi}{\delta}\right)} + \sqrt{\xi}U.
\end{aligned}
\end{equation}
For ease of notation, we also define 
the confidence set of $\bolds{c}^{\star}$: $\set{C}_t = \{\btheta^{+} - \btheta^{-} \colon [\btheta^{+},\btheta^{-}] \in \set{C}_t^+ \times \set{C}_t^-\}$. Using confidence sets from period $t-1$, we can modify our confidence bounds from \eqref{eq:conf-h} and \eqref{eq:conf-r} for period $t$: 
\begin{equation}\label{eq:conf-h-feature}
\ubar{c}_k(t) = \max\left(-c_{\max},\min_{\bolds{c} \in \set{C}_{t-1}} \bphi_k^{\trans}\bolds{c}\right),~\bar{c}_k(t) = \min\left(c_{\max},\max_{\bolds{c} \in \set{C}_{t-1}}\bphi_k^{\trans}\bolds{c}\right).
\end{equation}
\begin{equation}\label{eq:conf-r-feature}
\bar{\ell}_k(t) = \max_{[\btheta^{+},\btheta^{-}] \in \set{C}^+_{t-1} \times \set{C}^-_{t-1}} \min\left\{\bphi_k^{\trans}\btheta^{+}, \bphi_k^{\trans}\btheta^{-}\right\}.
\end{equation}

\subsection{Type Aggregation for Better Admission Decisions}\label{sec:aggregate}
As discussed above, when the number of types $K$ is large, it is difficult to estimate delay loss, a crucial component in the design of \textsc{BACID} which admits a job if and only if its idiosyncrasy loss is above an estimated delay loss. $\bacid$ estimates the delay loss of a job based on the number of same-type waiting jobs. With a large $K$, this approach underestimates the real delay loss, leading to overly admitting jobs into the review system. An alternative delay estimator uses the total number of waiting jobs $|\set{Q}(t)|$. This ignores the heterogeneous delay loss of admitting different types. In particular, our scheduling algorithm prioritizes jobs with less review workload (higher $\mu_k$) to effectively manage the limited capacity. A job with a higher service rate thus has smaller delay loss and neglecting this heterogeneity results in overestimating its delay loss.

To address this challenge, we create an estimator for the delay loss that lies in the middle ground of the aforementioned estimators. In particular, we map each type $k$ to a group $g(k)$ based on its  service rate; we denote this partition by $\set{K}_{\set{G}} = \{\set{K}_g\}_{g \in \set{G}}$ where $\set{G}$ is the set of groups, $G=|\set{G}|$ is its cardinality, and $\set{K}_g$ is the set of types in group $g \in \set{G}$.
For a group $g$, we define its proxy service rate $\tilde{\mu}_g = \min_{k \in \set{K}_g} \mu_k$ as the minimum service rate across types in this group. For a new job of type~$k$, we estimate its delay loss by the number of jobs of types in $\set{K}_{g(k)}$ waiting in the review queue. This estimator is efficient if the number of groups is small, and service rates of types in a group are close to each other. Specifically, letting $N_{\max} = \max_t N(t)$ be the maximum number of reviewers, we define the \emph{aggregation gap} $\Delta(\set{K}_{\set{G}})= \max_{g \in \set{G}} \max_{k \in \set{K}_g} N_{\max}(\mu_{k} - \tilde{\mu}_g)$ of a group partition $\set{K}_{\set{G}}$ as the maximum within-group service rate difference (scaled by reviewing capacity). 

\subsection{Algorithm, Theorem and Proof Sketch}
Our algorithm, \textsc{Contextual} $\bacidol$ (Algorithm~\ref{algo:cbacidol}), or $\conbacid$ in short, works as follows. In
period $t$, we first compute a ridge estimator $\hat{\bTheta}(t-1)$ of $\bTheta^\star$ by \eqref{eq:regression}. The algorithm then rejects a new job ($Y(t) = -1$) if its empirical average cost, $\hat{c}_{\kappa(t)}(t)=\bphi_{\kappa(t)}^{\trans}\hat{\bolds{c}}(t-1)$, is positive. We follow the same label-driven admission with $\bacidol$ in Line~\ref{line:conbacid-lda} but we set the confidence interval on $c_{k}$ by \eqref{eq:conf-h-feature} using the confidence sets $\set{C}_t^+, \set{C}_t^-$  from \eqref{eq:contextual-def-conf}. The optimistic admission rule (Line~\ref{line:conbacid-admit}) similarly finds an optimistic per-period idiosyncrasy loss $\bar{\ell}_k(t)$ but estimates the delay loss by type-aggregated queue lengths. Specifically, letting $\tilde{\set{Q}}_g(t) = \{\tau \in \set{Q}(t)\colon g(\kappa(\tau)) = g\}$ be the set of waiting jobs whose types belong to group $g$ and $\tilde{Q}_g(t) = |\tilde{\set{Q}}_g(t)|$, we admit a type-$k$ job if and only if its (optimistic) idiosyncrasy loss is higher than the estimated delay loss, i.e., $\beta \bar{\ell}_{k}(t) \geq \tilde{Q}_{g(k)}(t)$. We still prioritize the label-driven queue for scheduling. If there is no job in the label-driven queue, we use a type-aggregated \textsc{MaxWeight} scheduling: we first pick a group $g$ maximizing $\tilde{\mu}_g \tilde{Q}_g(t)$, and then pick the earliest admitted job in the review queue $\set{Q}(t)$ whose type is of group $g$ to review. 

\begin{algorithm}[H]
\LinesNumbered
\DontPrintSemicolon
  \caption{  \textsc{Contextual} $\bacidol$ ($\conbacid$)}\label{algo:cbacidol}
   \KwData{$\{\mu_k\}_{k \in \set{K}}$, upper bounds $U,c_{\max},\sigma_{\max}$, and group partition $\set{K}_{\set{G}}$ \\ \textbf{Admission parameters:} $\beta = \sqrt{T / (Gd^{1.5})}, \gamma = (T / (d^{2.5}\ln^{2}(T)))^{-1/3},~\delta \gets \min(\gamma,0.5/T),~\xi \gets \max(1,U^2)$}
  \For{$t = 1$ \KwTo $T$}{
    Observe a new job of type $\kappa(t)$ \\
    Compute ridge estimator $\hat{\bTheta}(t-1)$ by \eqref{eq:regression}\;
    \lIf{$\bphi_{\kappa(t)}^{\trans}\hat{\bolds{c}}(t-1)> 0$}{$Y(t) \gets -1$
     \textbf{else} {$Y(t) \gets 1$}
    \tcp*[f]{Empirical Classification} \label{line:conbacid-classify}}
    \tcc{Label-Driven and Optimistic Admission}
Calculate $\ubar{c}_{\kappa(t)}(t),\bar{c}_{\kappa(t)}(t),\bar{\ell}_{\kappa(t)}(t)$ by \eqref{eq:conf-h-feature} and \eqref{eq:conf-r-feature} \label{line:conbacid-esti}\;
  \lIf{$\ubar{c}_{\kappa(t)}(t) < -\gamma < \gamma < \bar{c}_{\kappa(t)}(t)$ \emph{\textbf{and}} $|\set{Q}^{\ld}(t)|=0$}{
  $E(t) \gets 1$~\textbf{else} $E(t)=0$\label{line:conbacid-lda}}
    \lIf{$E(t) = 0$ \emph{\textbf{and}} $\beta \cdot \bar{\ell}_{\kappa(t)}(t)  \geq \tilde{Q}_{g(\kappa(t))}(t)$}{
      $A(t) = 1$ \textbf{else} {$A(t) = 0$} 
    \label{line:conbacid-admit}
    }
    \tcc{Forced Scheduling and Type-Aggregated \textsc{MaxWeight} Scheduling}
    \lIf{$\set{Q}^{\ld}(t) \neq \emptyset$}{$M(t)\gets$ the job in $\set{Q}^{\ld}(t)$} \label{line:conbacid-prio}
    \lElse{
    $g \gets \arg\max_{g' \in \mathcal{G}} \tilde{\mu}_{g'} \cdot \tilde{Q}_{g'}(t)$,~$M(t) \gets \text{first job in } \tilde{\set{Q}}_g(t)   \text{ if any}$  \label{line:conbacid-schedule}
    } 
    \lIf{\emph{the review finishes (with probability $N(t) \cdot \mu_{\kappa(M(t))}$)}}
    {
      $\mathcal{S}(t) = M(t)$ \textbf{else} 
      $\mathcal{S}(t) = \perp$ 
    }
  }
\end{algorithm}

For ease of exposition, we follow the $\lesssim$ notation to include only dependence on the number of groups $G$, feature dimension $d$, the margin $\eta$ (recall that $\eta = \min(1,\min_{k \in \set{K}} |c_k|)$), the time horizon $T$, and the aggregation gap $\Delta(\set{K}_{\set{G}})$. By setting $\beta = \sqrt{T / (Gd^{1.5})}, \gamma = (T / (d^{2.5}\ln^{2}(T)))^{-1/3}$, we obtain the following guarantee of $\conbacid$. 
\begin{theorem}\label{thm:cbacidol}
The regret of $\conbacid$ is upper bounded by
\[
\reg^{\conbacid}(T) \lesssim \Delta(\set{K}_{\set{G}})T+\min\left(\frac{d^{2.5}\ln^{2} T}{\eta^2}, d^{5/6}(T\ln(T))^{2/3}\right)+ d\sqrt{GT \ln T}.
\]
\end{theorem}
\noindent Note that, if, for all groups $g \in \set{G}$, all types $k \in \set{K}_g$ have the same service rate, then $\Delta(\set{K}_{\set{G}}) = 0$, but the regret still depends on $G$ which is large if there are many types with unequal service rates.

To provide a worst-case guarantee, we select a fixed aggregation gap $0<\zeta\leq 1$ and create a partition $\set{K}_{\set{G}}^{\zeta}$ that segments types based on the their maximum scaled service rate $N_{\max}\mu_k$ into intervals $(0,\zeta],(\zeta,2\zeta],\ldots,(\lfloor \frac{1}{\zeta}\rfloor \zeta,1]$. The number of groups is at most $G \leq \frac{1}{\zeta} + 1 \leq \frac{2}{\zeta}$ and the aggregation gap $\Delta(\set{K}_{\set{G}}^{\zeta})$ is at most $\zeta$. Optimizing the bound in Theorem~\ref{thm:cbacidol} for the term $d\sqrt{GT} + \Delta(\set{K}_{\set{G}})T$, and setting $\zeta^\star = d^{2/3}T^{-1/3}$, we obtain that the regret of $\conbacid$ scales as $\tilde{O}(T^{2/3})$.
\begin{corollary}\label{corr:group-ind}
$\conbacid$ with group partition $\set{K}_{\set{G}}^{\zeta^\star}$ has
$
\reg^{\conbacid}(T) \lesssim d^{5/6}(T\ln(T))^{2/3}.
$
\end{corollary}
\noindent The proof of Theorem~\ref{thm:cbacidol} (Appendix~\ref{app:thm-cbacidol}) bounds the losses in the same decomposition as in \eqref{eq:learning-loss-decompose},\footnote{Note that $\sum_{g \in \set{G}} \tilde{Q}_g(T+1) = \sum_{k \in \set{K}} Q_k(T+1)$ but summing across groups facilitates our per-group analysis.}
\begin{align}
&\hspace{0.1in}\underbrace{\expect{\sum_{t=1}^T \left(Y(t)^+(\ell_{\kappa(t)}^{+} - \ell_{\kappa(t)}) + Y(t)^-(\ell_{\kappa(t)}^{-}-\ell_{\kappa(t)})\right)(1 - A(t) - E(t))}}_{\text{Classification Loss}} \nonumber\\
&+  \underbrace{\expect{\sum_{t=1}^T  \ell_{\kappa(t)}(1 - A(t) - E(t))}}_{\text{Idiosyncrasy Loss}} +\underbrace{c_{\max}\expect{\sum_{g \in \set{G}} \tilde{Q}_g(T+1) + \left|\set{Q}^{\ld}(T+1)\right|}}_{\text{Relaxed Delay Loss}}.\label{eq:contextual-loss-decompose}
\end{align}
Similar to the proof of Theorem~\ref{thm:bacidol}, we set $\hat{\mu}_{\min} = \min_{t \leq T, k \in \set{K}} N(t)\mu_k$ and focus on $T \geq 3$. The following lemmas also assume $G \leq T$; Theorem~\ref{thm:cbacidol} holds directly otherwise.

\begin{lemma}\label{lem:contextual-delay}
For any $\beta \geq 1 / c_{\max}$, the Relaxed Delay Loss of $\conbacid$ is at most $3c^2_{\max}G\beta$. 
\end{lemma}
The proof of Lemma~\ref{lem:contextual-delay} is identical to that of Lemma~\ref{lem:bacidol-delay} but bounds the group-level queue length $\{\tilde{Q}_g(t)\}_{g \in \set{G}}$ by $2\beta c_{\max}$ instead of bounding the type-level queue length $\{Q_k(t)\}_{k \in \set{K}}$.
\begin{lemma}\label{lem:contextual-class}
For any $0 < \gamma \leq 1$, $\xi \geq U^2$ and $\delta \in (0,0.5/T]$,
\[
\text{Classification Loss} \leq c_{\max} + \gamma \indic{\gamma \geq \eta} T + \frac{34c_{\max}B_{\delta}^2(T)d\ln(1+T/d)}{\max(\eta,\gamma)^2\hat{\mu}_{\min}}.
\]
\end{lemma}
\begin{lemma}\label{lem:contextual-idio}
For any $\beta \geq 1 / c_{\max}, \gamma \in (0,1]$, $\xi \geq U^2$ and $\delta \in (0,0.5/T]$,  \[\text{Idiosyncrasy Loss} \lesssim
\set{L}^\star(T) + \frac{T}{\beta} +\left(\Delta(\set{K}_{\set{G}})T+\beta Gd^{1.5}\sqrt{\ln T}+d\sqrt{GT \ln T}+\frac{d^{2.5}\ln^{2} T}{\max(\eta,\gamma)^2}\right).\]
\end{lemma}

\subsection{Bounding Classification Loss (Lemma~\ref{lem:contextual-class})}
Similar to the proof of Lemma~\ref{lem:bacidol-class}, we first bound the sum of the label-driven queue length $\expect{\sum_{t=1}^T Q^{\ld}(t)}$. By Lemma~\ref{lem:connect-et-qe}, bounding the first term requires bounding the expected number of jobs that we admit into the label-driven queue, i.e., $\expect{\sum_{t=1}^T E(t)}.$ We first define the ``good'' event $\set{E}$, where the confidence sets are valid for any period $t$: $\set{E} \coloneqq \{\forall t, \bTheta^\star \in \set{C}_t^+ \times \set{C}_t^-\}$. Our first lemma shows that this event happens with probability at least $1-2\delta$ using \cite[Theorem~2]{Abbasi-YadkoriPS11} (proof in Appendix~\ref{app:lem-contextual-prob-good-event}).
\begin{lemma}\label{lem:contextual-prob-good-event}
For any $\delta \in (0,1)$, it holds that with probability at least $1-2\delta$, for any $t \geq 0$, the confidence sets are valid, i.e., $\bTheta^\star \in \set{C}_t^+ \times \set{C}_t^-$.
\end{lemma}
Our second lemma bounds the expected number of jobs admitted into the label-driven queue. The proof is based on classical linear contextual bandit analysis and is provided in Appendix~\ref{app:lem-bound-norm-selective}.
\begin{lemma}\label{lem:bound-norm-selective}
If $\xi \geq U^2, \gamma \leq 1$, $\delta \leq 0.5/T$, the number of jobs admitted to the label-driven queue is  \[\expect{\sum_{t=1}^T E(t)} \leq \frac{34B_{\delta}^2(T)d\ln(1+T/d)}{\max(\eta,\gamma)^2}.\]
\end{lemma}
We next bound the per-period classification loss (proof in Appendix~\ref{app:lem-con-period-class}), i.e.,
\[
Z_k(t) = \Lambda_k(t)(Y(t)^+(\ell_k^+ - \ell_k) + Y(t)^-(\ell_k^- -\ell_k))(1-A(t)-E(t)).\]
\begin{lemma}\label{lem:con-period-class}
For any type $k$ and period $t$, $Z_k(t)\indic{\set{E}} \leq \left(\gamma\indic{\gamma \geq \eta} + c_{\max}Q^{\ld}(t)\right)\Lambda_k(t).$
\end{lemma}
\begin{proof}[Proof of Lemma~\ref{lem:contextual-class}]
The classification loss is given by $\expect{\sum_{t=1}^T\sum_{k \in \set{K}} Z_k(t)}$, which we bound by
\begin{align*}
\expect{\sum_{t=1}^T\sum_{k \in \set{K}} Z_k(t)} &\leq  c_{\max}\expect{\sum_{t=1}^T \indic{\set{E}^c}} + \expect{\sum_{t=1}^T \sum_{k \in \set{K}} Z_k(t)\indic{\set{E}}}  \tag{$\sum_{k \in \set{K}} Z_k(t) \leq c_{\max}$}\\
&\leq 2c_{\max} \delta T+ \expect{\sum_{t=1}^T \sum_{k \in \set{K}} Z_k(t)\mid \set{E}} \tag{By Lemma~\ref{lem:contextual-prob-good-event}}\\
&\hspace{-0.7in}\leq c_{\max} + \expect{\sum_{t=1}^T \left(\gamma \indic{\gamma \geq \eta} + c_{\max}Q^{\ld}(t)\right)} \tag{By Lemma~\ref{lem:con-period-class} and $\delta \leq 0.5 / T$} \\
&\hspace{-0.7in}\leq c_{\max} + \gamma \indic{\gamma \geq \eta} T + \frac{34c_{\max}B_{\delta}^2(T)d\ln(1+T/d)}{\max(\eta,\gamma)^2\hat{\mu}_{\min}} \tag{By Lemmas~\ref{lem:connect-et-qe},~\ref{lem:bound-norm-selective}}.
\end{align*}
\end{proof}

\subsection{Bounding the Idiosyncrasy Loss (Lemma~\ref{lem:contextual-idio})}
To bound the idiosyncrasy loss, we connect it to the fluid benchmark via Lagrangians. In the corresponding lemmas of the previous sections (Lemmas~\ref{lem:known-idiosyncrasy} and \ref{lem:bacidol-idio}), we evaluate the Lagrangians by the queue length vector across types $\bolds{Q}(t)$ because we estimate the delay loss of admitting a new type-$k$ job by the number of type-$k$ jobs in the review queue. To avoid greatly underestimating the delay loss (due to the larger number of types), a key innovation of $\conbacid$ is to use the number of waiting jobs in the same group $\tilde{Q}_{g(k)}(t)$ as an estimator. This new estimator motivates us to use $\tilde{Q}_{g(k)}(t)$ as the dual for the Lagrangian analysis and to define a new Lyapunov function
\[
\tilde{L}(t) = \beta\sum_{t'=1}^{t-1}\ell_{\kappa(t')}(1- A(t') - E(t')) + \frac{1}{2}\sum_{g \in \set{G}} \tilde{Q}_g^2(t).
\]
We also define the \emph{type-aggregated} queue length vector, $\bolds{Q}^{\ta}(t) = (Q^{\ta}_k(t))_{k \in \set{K}}$, such that $Q^{\ta}_k(t) = \tilde{Q}_{g(k)}(t)$ for any $k \in \set{K}$. We denote $\bar{A}_k(t) = \Lambda_k(t)\indic{\beta \bar{\ell}_k(t) \geq \tilde{Q}_{g(k)}(t)}$, which captures whether the job would have been admitted in the absence of the label-driven admission. For scheduling, we select a group $g$ that maximizes $\tilde{\mu}_g\tilde{Q}_g(t)$ and choose the first waiting job of that group to review if the label-driven queue is empty.  We denote $\psi_{k}(t) = 1$ where $k$ is the type of that reviewed job and let $\bolds{\psi}(t) = (\psi_k(t))_{k \in \set{K}}$. The following lemma (proved in Appendix~\ref{app:lem-cbacidol-idio-lag}) connects the idiosyncrasy loss and the per-period Lagrangian \eqref{eq:per-period-lag} with the type-aggregated queue lengths $\bolds{Q}^{\ta}(t)$ as the dual.
\begin{lemma}\label{lem:cbacidol-idio-lag}
The expected Lagrangian of $\conbacid$ upper bounds the idiosyncrasy loss by:
\[\expect{\beta\sum_{t=1}^{T}\ell_{\kappa(t)}(1-A(t)-E(t))} \leq \expect{\tilde{L}(T+1) - \tilde{L}(1)} \leq T + \sum_{t=1}^T \expect{f_t(\bolds{\bar{A}}(t), \bolds{\psi}(t), \bolds{Q}^{\ta}(t))}.\]
\end{lemma}

Our next step is to bound the Lagrangian with $\bolds{Q}^{\ta}(t)$ as the dual. In Lemma~\ref{lem:lagrang-bacid} (Section~\ref{sec:lagrang-bacid}), the Lagrangian of \textsc{BACID} for dual $\bolds{Q}(t)$ is connected to the fluid benchmark via a Lagrangian optimality result of \textsc{BACID} (Lemma~\ref{lem:bacid-mw}). We also use this result to bound the Lagrangian of $\bacidol$ (Section~\ref{sec:bacidol-idio}). However, with the type-aggregated queue length $\bolds{Q}^{\ta}(t)$ as the dual, the Lagrangian optimality of $\bacid$ no longer holds. To deal with this challenge, we thus introduce another benchmark policy which we call \textsc{Type Aggregated BACID} or $\tabacid$. 

Letting $\{\bolds{A}^{\tabacid}(t)\}_{t \in [T]}, \{\bolds{\psi}^{\tabacid}(t)\}_{t \in [T]}$ be the admission and scheduling decisions, $\tabacid$ admits a type-$k$ job if $A^{\tabacid}_k(t) = \Lambda_k(t)\indic{\beta \ell_k \geq \tilde{Q}_{g(k)}(t)} = 1$. For scheduling, we first pick a group $g$ maximizing $\tilde{\mu}_g\tilde{Q}_g(t)$ and then select the earliest job $j$ in the review queue that belongs to this group (unless there is no waiting job). We set $\psi^{\tabacid}_{\kappa(j)}(t) = 1$ for the corresponding type. Note that the admission decisions of $\tabacid$ are the same as $\conbacid$ except that $\tabacid$ uses the ground-truth $\ell_k$ for admission (not the optimistic estimation) and does not consider the impact of label-driven admission. The scheduling differs from $\textsc{MaxWeight}$ and is suboptimal due to grouping types with different service rates which is captured by $\Delta(\set{K}_{\set{G}})\max_g \tilde{Q}_g(t)$. We now upper bound the expected Lagrangian of $\tabacid$ (proved in Appendix~\ref{app:lem-context-benchmark-lagran}). 
\begin{lemma}\label{lem:context-benchmark-lagran}
The expected Lagrangian of $\tabacid$ is upper bounded by:
\[
\sum_{t=1}^T \expect{f_t(\bolds{A}^{\tabacid}(t), \bolds{\psi}^{\tabacid}(t), \bolds{Q}^{\ta}(t))} \leq \beta \set{L}^\star(T) + 2\beta c_{\max}\Delta(\set{K}_{\set{G}})T.
\]
\end{lemma}
As in the proof of Lemma~\ref{lem:bacidol-idio} (Section~\ref{sec:bacidol-idio}), we relate the right hand side in Lemma~\ref{lem:cbacidol-idio-lag} (Lagrangian of $\conbacid$) to the left hand side in Lemma~\ref{lem:context-benchmark-lagran} (Lagrangian of $\tabacid$) by 
\begin{align}
\textsc{RegL}(T) &=\expect{\sum_{t=1}^T f_t(\bolds{\bar{A}}(t), \bolds{\psi}(t), \bolds{Q}^{\ta}(t)) - f_t(\bolds{A}^{\tabacid}(t), \bolds{\psi}^{\tabacid}(t), \bolds{Q}^{\ta}(t))} \nonumber\\
&= \underbrace{\expect{\sum_{t=1}^T \sum_{k\in\set{K}} (A_k^{\tabacid}(t) - \bar{A}_k(t))(\beta \ell_k-\tilde{Q}_{g(k)}(t))}}_{\textsc{RegA}(T)} \nonumber \\
&\hspace{0.2in}+ \underbrace{\expect{\sum_{t=1}^T \sum_{k \in \set{K}} (\psi_k^{\tabacid}(t) - \psi_k(t))\tilde{Q}_{g(k)}(t)\mu_k N(t)}}_{\textsc{RegS}(T)}\label{eq:cbacid-decomp-RegL}.
\end{align}
Letting $Q_{\max} = 2\beta c_{\max}$, we bound $\textsc{RegS}(T)$ as for Lemma~\ref{lem:bacidol-regs} (proof in Appendix~\ref{app:lem-cbacid-regS}).
\begin{lemma}\label{lem:cbacid-regS}
If $\xi \geq U^2, \gamma \leq 1,\delta \leq 0.5/T$, the regret in scheduling is $\textsc{RegS}(T) \leq \frac{68Q_{\max}B^2_{\delta}(T)d\ln(1+T/d)}{\max(\eta,\gamma)^2\hat{\mu}_{\min}}.$
\end{lemma}
Our novel contribution is the following result bounding the regret in admission $\textsc{RegA}(T)$. The proof handles contextual learning with queueing delayed feedback, which we discuss in Section~\ref{sec:contextual-regret}.
\begin{lemma}\label{lem:conbacid-rega}
If $\xi \geq U^2, \gamma \leq 1,\delta \leq 0.5/T$, the regret in admission is 
\[
\textsc{RegA}(T) \leq 3\beta c_{\max}B_{\delta}(T)\left(d\ln(1+T/d)\left(4GQ_{\max} + \frac{34B_{\delta}^2(T)}{\max(\eta,\gamma)^2}\right) + \sqrt{2GTd\ln(1+T/d)}\right).
\]
\end{lemma}
\begin{proof}[Proof of Lemma~\ref{lem:contextual-idio}]
By Lemma~\ref{lem:cbacidol-idio-lag}, we have 
\begin{align*}
\expect{\beta\sum_{t=1}^{T}\ell_{\kappa(t)}(1-A(t)-E(t))} &\leq T + \sum_{t=1}^T \expect{f_t(\bolds{\bar{A}}(t), \bolds{\psi}(t), \bolds{Q}^{\ta}(t))} \\
&\hspace{-2in}= T + \sum_{t=1}^T \expect{f_t(\bolds{A}^{\tabacid}(t), \bolds{\psi}^{\tabacid}(t), \bolds{Q}^{\ta}(t))} + \textsc{RegS}(T) + \textsc{RegA}(T) \\
&\hspace{-2in}\leq T+\beta \set{L}^\star(T) + 2\beta c_{\max}\Delta(\set{K}_{\set{G}})T + \textsc{RegS}(T) + \textsc{RegA}(T)\tag{Lemma~\ref{lem:context-benchmark-lagran}}.
\end{align*}
The result follows by bounding $\textsc{RegS}(T)$ and $\textsc{RegA}(T)$ respectively by Lemmas~\ref{lem:cbacid-regS} and \ref{lem:conbacid-rega}, and by dividing both sides by $\beta$ and noting that $B_{\delta}(T) \lesssim \sqrt{d\ln T}, Q_{\max} \lesssim \beta c_{\max}$.
\end{proof}

\subsection{Contextual Learning with Queueing-Delayed Feedback (Lemma~\ref{lem:conbacid-rega})}\label{sec:contextual-regret}

To bound the regret in admission in a way that avoids dependence on $K$ (that Lemma~\ref{lem:bacidol-rega} exhibits), we rely on the contextual structure to more effectively bound the total estimation error of admitted jobs. This has the additional complexity that feedback is observed after a queueing delay (that is endogenous to the algorithmic decisions) and is handled via Lemma~\ref{lem:conbacid-error} below.
The proof of Lemma~\ref{lem:conbacid-rega} (Appendix~\ref{app:lem-conbacid-rega}) then follows from classical linear contextual bandit analysis. 

The estimation error for one job with feature $\bphi_k$ given observed data points that form matrix $\bar{V}_{t-1}$ (defined in \eqref{eq:regression}) corresponds to $\|\bphi_k\|_{\bar{V}_{t-1}^{-1}}$. To see the correspondence, in a non-contextual setting, $\bphi_k$ is a unit vector for the $k-$th dimension, and $\bar{\bV}_{t-1}$ is a diagonal matrix where the $k-$th element is the number of type-$k$ reviewed jobs $n_k(t-1)$. As a result, $\|\bphi_k\|_{\bar{\bV}_{t-1}^{-1}} = 1 /\sqrt{n_k(t-1)}$, which is the estimation error we expect from a concentration inequality. 

Our first result (proof in Appendix~\ref{app:lem-bound-delayed-error}) bounds the estimation error when feedback of all admitted jobs is delayed by a fixed duration, which we utilize to accommodate random delays.
\begin{lemma}\label{lem:bound-delayed-error}
Given a sequence of $M$ vectors $\hat{\bphi}_1,\ldots,\hat{\bphi}_M$ in $\mathbb{R}^d$, let $\hat{\bV}_j = \xi \bI + \sum_{j'=1}^j \hat{\bphi}_{j'}\hat{\bphi}_{j'}^{\trans}$. If $\|\hat{\bphi}_i\|_2 \leq U$ for any $i \leq M$ and $\xi \geq U^2$, the estimation error for a fixed delay $q \geq 1$ is \[\sum_{i=q}^M \|\hat{\bphi}_i\|_{\hat{\bV}_{i-q}^{-1}} \leq \sqrt{2Md\ln(1+M/d)}+2qd\ln(1+M/d).\]
\end{lemma}
The challenge in our setting is that the feedback delay is not fixed, but is indeed affected by both the admission and scheduling decisions. The following lemma upper bounds the estimation error under this queueing delayed feedback, enabling our contextual online learning result.

\begin{lemma}\label{lem:conbacid-error}
If $\xi \geq U^2$ and $T \geq 3$, the estimation error is bounded by 
\[
\expect{\sum_{t=1}^T \sum_{k \in \set{K}} \|\bphi_k\|_{\bar{\bV}_{t-1}^{-1}}\bar{A}_k(t)} \leq \sqrt{2GTd\ln(1+T/d)} + d\ln(1+T/d)\left(3GQ_{\max}+\frac{34B_{\delta}^2(T)}{\max(\eta,\gamma)^2}\right).
\]
\end{lemma}

\begin{proof}[Proof sketch] 
The proof contains three steps. The first step is to connect the error of a job in our setting to a fixed-delay setting by the first-come-first-serve (FCFS) property of our scheduling algorithm (Lemma~\ref{lem:connect-fix-delay}). In particular, consider the sequence of admitted group-$g$ jobs. For a job $j$ on this sequence, the set of jobs before $j$ whose feedback is still not available can include at most the $Q_{\max}$ jobs right before $j$ on the sequence (where $Q_{\max}$ is controlled by our admission rule). Therefore, the error of group-$g$ jobs accumulates as in a setting with a fixed delay $Q_{\max}$. 

Based on this result, the second step (Lemma~\ref{lem:error-total}) bounds $\expect{\sum_{t=1}^T \sum_k \|\bphi_k\|_{\bar{\bV}_{t-1}^{-1}}A_k(t)}$, which is the estimation error for admitted jobs. Enabled by the connection to the fixed-delay setting, we bound the estimation error for each group separately by Lemma~\ref{lem:bound-delayed-error} and aggregate them to get the total error, i.e, the first two terms in Lemma~\ref{lem:conbacid-error}. 

For the third step, corresponding to the last term of our bound, we bound the difference between the error of all admitted jobs (which we upper bounded in the second step) and the error of jobs admitted by the optimistic admission. We show that this difference is at most the number of label-driven admissions $\expect{\sum_{t=1}^T E(t)}$ (bounded by Lemma~\ref{lem:bound-norm-selective}) because $\bar{A}_k(t) - A_k(t) \leq E(t)$ and $\|\bphi_k\|_{\bar{\bV}_{t-1}^{-1}} \leq 1$. The full proof is provided in Appendix~\ref{app:lem-conbacid-error}.
\end{proof}

\section{Application to Content Moderation}
\label{sec:content_moderation}
In this section, we show how our algorithms and insights apply to content moderation for social media platforms. We first overview the content moderation practice in Meta Platforms based on the descriptions in \cite{Avadhanula2022} in Section~\ref{sec:overview-practice}. We then show how our model instantiates to this setting in Section~\ref{sec:model-instantiate}. Section~\ref{sec:set-up} describes the set-up of our numerical experiments based on real content moderation data. With this set-up, in Section~\ref{sec:evaluation}, we (i) numerically evaluate the performance of our algorithm (specifically $\conbacid$) and (ii) test to what extent optimism-only approaches (e.g. the one applied by \cite{Avadhanula2022}) fail to classify policy-violating content.

\subsection{Overview of social media content moderation pipeline}\label{sec:overview-practice}
Our model is based on a content-moderation pipeline deployed at Meta Platforms as described in \cite{Avadhanula2022}. Their pipeline contains several pretrained ML models that predict whether each content violates a specific platform policy (e.g., with respect to hate speech) or not. The central task is to combine these predictions for classification, admission, and scheduling decisions that minimize the prevalence of policy-violating content. Specifically, their pipeline contains the following components.

\begin{itemize}
\item \emph{Arrivals}: At time $t$, content  $t$ arrives into the content-moderation system. The content has severity $y_t \in \{0,1\}$ such that $y_t = 1$ if and only if it violates one of the platform's policies. The severity is unobservable until the content is reviewed by a human reviewer. Each content has an ``\emph{integrity value}'' (IV). Specifically, content $t$ has a predicted number of future views $\pview_t$ upon its arrival and its IV is defined by $\textsc{IV}_t = \pview_t \cdot y_t$ (there can be a constant adjustment on $\pview_t$). 
\item \emph{Context generation}: The platform has $m$ pretrained ML models, each of which focuses on a specific platform policy. These models are difficult to retrain and are thus only retrained periodically. For content $t$, ML model $i$ receives as input the content's features and outputs a score $x_{t,i} \in \mathbb{R}$ denoting the content's risk of violating a corresponding platform policy. As a result, the output of those ML models is  a vector of risk scores $\bolds{x}_t \in \mathbb{R}^m$. Moreover, \cite{Avadhanula2022} discretize the space of risk scores into $b$ bins $\set{B}_1,\ldots,\set{B}_b$. The eventual context of content $t$ is a $d-$dimensional vector $\bolds{\phi}_t$ with $d = m \cdot b$, such that the $[(i-1)\cdot b + j]-$th element of this vector is given by $x_{t,i}\indic{x_{t,i} \in \set{B}_j}$ for any $i \leq m$ and $j \leq b$.
\item \emph{Classification}: Content with ``unambiguously'' high risk scores is automatically deleted from the platform (``filtering'' and ``auto-delete'' in figure 1 of \cite{Avadhanula2022}).
\item \emph{Admission}: Content with ``ambiguous'' or ``uncertain'' risk scores is admitted for further human review. 
For this purpose, section 3 of \cite{Avadhanula2022} develops a small risk prediction model that aggregates the scores from the $m$ ML models.  For the $i-$th ML model and the $j-$th bin, they maintain a parameter $\bar{\beta}_{i,j}$ such that the score $x_{t,i}$ is scaled by $\bar{\beta}_{i,j}$ when $x_{t,i} \in \set{B}_j$. Content $t$ is admitted for human review if and only if its predicted ``severity'' $\bar{y}_t \coloneqq \max_{i \leq m} \sum_{j=1}^b \indic{x_{t,i} \in \set{B}_j}x_{t,i} \bar{\beta}_{i,j} $ is positive (the admission decision $A(t) = 1$ when $\bar{y}_t > 0$ in algorithm~1 of \cite{Avadhanula2022}). 

The key innovation of \cite{Avadhanula2022} is to update the parameters $\{\bar{\beta}_{i,j}\}_{i \leq m, j \leq b}$ in an online manner with labels provided by human reviewers. Specifically, \cite{Avadhanula2022} maintain a dataset $\{\bolds{x}_{\tau}, y_{\tau}\}_{\tau \in \set{D}(t)}$ where $\set{D}(t)$ is the set of reviewed content by time $t$ and $y_{\tau}$ is the severity of content $\tau$. They estimate $\bar{\beta}_{i,j}$ separately. Specifically,  for every pair of $i \leq m,j \leq b$, they fit a one-dimensional linear regression such that  $y_{\tau} \approx (\indic{x_{\tau,i} \in \set{B}_j}x_{\tau,i}) \cdot \beta_{i,j}$ for $\tau \in \set{D}(t).$ The parameters $\{\bar{\beta}_{i,j}\}_{i \leq m, j \leq b}$ are upper-confidence-bound estimates for the linear regression parameters $\{\beta_{i,j}\}_{i \leq m, j \leq b}$.  These parameters are updated ``\emph{once every 5 minutes}'' to adapt to changing violation trends (page 5 in \cite{Avadhanula2022}).
\item \emph{Scheduling}: Admitted pieces of content are prioritized by their predicted severity for human reviews, weighted suitably by their view trajectory information (see section 4 of \cite{Avadhanula2022}).
\item \emph{Objective}: The objective of the system is to maximize the sum of IVs of content reviewed by humans by the end of the horizon (section~2 of \cite{Avadhanula2022}):
\begin{equation}\label{eq:obj-pipeline}
\sum_{t \leq T} A(t) \cdot \text{IV}_t - \sum_{t \leq T} A(t) \cdot \text{IV}_t \cdot \indic{t \in \set{Q}(T+1)}.
\end{equation}
We note that \cite{Avadhanula2022} optimize the admission and scheduling decisions while assuming that the classification decisions are made exogenously.
\end{itemize}

\begin{remark}The use of online learning to refine content-moderation decisions is not unique to \cite{Avadhanula2022}. For example, \cite{lu2025} describe the use of LLMs for content moderation deployed in Kuaishou, a chinese short video platform. \cite{lu2025} highlight the need of using ``online feedback'' to continuously refine their model because ``the scope of violation contents ... evolve with users and social trends'' (page 4687 of \cite{lu2025}). Moreover, although \cite{Avadhanula2022} use the outputs from several ML models as simple features for  uncertainty quantification, \cite{felicioniimportance} propose approaches to estimate the \emph{epistemic} uncertainty (akin to our confidence sets in Section~\ref{sec:contextual}) of LLMs. \cite{felicioniimportance} then show the benefit of online exploration (Thompson Sampling) to LLMs for content moderation.
\end{remark}

\subsection{Instantiating our model to the above pipeline}\label{sec:model-instantiate}
We now describe how our model captures the described content moderation pipeline.
\begin{itemize}
\item \emph{Arrivals}: A job $t$ in our model corresponds to content $t$ in the above pipeline. If the content is policy-violating (severity $y_t = 1$), the cost of this job is $C_t = \pview_t$, i.e., the cost is equal to the content's IV value. On the other hand, if the content is not policy-violating, we take the job's cost to be $C_t = -v \cdot \pview_t < 0$, where the quantity $v > 0$ captures the value of a non-policy-violating content to the platform (see Remark~\ref{remark:value} for further discussion). We instantiate the job's type by its index, i.e., $\kappa(t) = t$.
\item \emph{Context generation}: Content (type) $t$ has a $d$-dimensional feature vector $\bolds{\phi}_t$ with $d = m \cdot b$, where $m$ is the number of pretrained models, $b$ is the number of bins, and feature $\phi_{t,(i-1)b + j}$ is equal to $\indic{x_{t,i} \in \set{B}_j}x_{t,i}$ for $i \leq m, j \leq b$. 
\item \emph{Classification}: Observing the feature vector $\bolds{\phi}_t,$ the system decides to accept or reject job $t$, corresponding to keeping or auto-deleting content $t$. In the pipeline described in the prior section, the platform keeps all content except those with ``unambiguously'' high risk scores. An auto-deleted content is removed from the platform. However, if it was misclassified and a human later corrects this misclassification, the content may reappear on the platform.
\item \emph{Admission}: The system also decides whether to admit the job for further human review based on $\bolds{\phi}_t$. Section~\ref{sec:overview-practice} describes a specific admission rule that admits job (content) $t$ if and only if the predicted severity $\bar{y}_t > 0$.
\item \emph{Scheduling}: The system schedules a job for review in each period. Section~\ref{sec:overview-practice} gives a scheduling rule that prioritizes content with high predicted severity. Since \cite{Avadhanula2022} does not distinguish the review speed for different content, in this section we assume that all content share the same service rate $\mu$.
\item \emph{Objective}: Our objective in \eqref{eq:policy-loss} is to minimize content's classification loss by the end of the horizon. This objective is \emph{identical} to the objective of the above pipeline when the classification decisions are exogenous and never wrongly reject jobs with negative costs (this approximates the setting in Section~\ref{sec:overview-practice} which only auto-deletes unambiguously policy-violating content.) This is because in this case, minimizing the loss in \eqref{eq:policy-loss} is equivalent to maximizing 
\[
\sum_{t \leq T} C^+_t (A(t) - \indic{t \in \set{Q}(T+1)}) = \sum_{t \leq T} A(t) \cdot \text{IV}_t  - \sum_{t \leq T} A(t)\text{IV}_t \indic{t \in \set{Q}(T+1)} = \eqref{eq:obj-pipeline}.
\]
Compared to the objective of the pipeline \eqref{eq:obj-pipeline}, our objective in \eqref{eq:policy-loss} also captures the classification loss when classification decisions also depend on admitted content and are thus not exogenous as in \cite{Avadhanula2022}.
\end{itemize}
\begin{remark}\label{remark:value}
To contextualize the value $v > 0$ of a non-policy-violating content, we can view it as the dual of a constrained optimization problem where the platform wants to remove as many policy-violating content pieces while maintaining a high precision \citep{linkedin}. Tuning $v$ helps the platform navigate the trade-off between false positives (incorrect removal of non-policy-violating content) and false negatives (wrongly keeping policy-violating content): a larger $v$ prioritizes the reduction of false positive while a smaller $v$ prioritizes the reduction of false negative. 
\end{remark}

\subsection{Experiment set-up based on real data} \label{sec:set-up}
Our experiments are based on both real data on content moderation and simulations of the content-moderation pipeline. We simulate both the $\conbacid$ algorithm developed in Section~\ref{sec:contextual} and an algorithm based on static threshold and UCB exploration, mimicking the algorithm in \cite{Avadhanula2022}. Below we provide details of the experiment, including the datasets, the pretrained ML model, the simulation of the pipeline, and how we implement different algorithms.

\noindent{\textbf{Datasets}.}We use two datasets on toxic comment classification. The first dataset is the test set of Kaggle competition ``Toxic Comment Classification Challenge'', which contains Wikipedia comments collected by Jigsaw and Google \citep{dataset-wiki}. We denote this dataset by $\set{D}_{\text{offline}}$ as it serves the purpose of offline available dataset in our experiment. The second dataset is the test set of Kaggle competition ``Jigsaw Unintended Bias in Toxicity Classification'', which contains comments from the Civil Comments platform \citep{dataset-civil}. We denote this dataset by $\set{D}_{\text{online}}$ as it will represent the online content the platform needs to classify. Both datasets contain the original text of each piece of content and whether this text is marked by a human reviewer to violate a specific policy (called toxicity type in the original competitions). We remove all comments that do not have any human labels (their labels were all set to -1). After cleaning, the offline dataset contains $63978$ and the online dataset contains $97320$ samples. For the purpose of our experiment, a piece of content $t$ is policy-violating (i.e., its severity $y_t = 1$) if it violates any policy in the dataset. The distribution shift between the offline and online datasets captures the non-stationarity in content's violation trends common in social media platforms \citep{Avadhanula2022}.

\noindent{\textbf{Pretrained ML model}.} To mimic the pretrained ML models in the pipeline, we use a pretrained toxicity prediction model from \cite{detoxify}, which outputs risk scores of a content piece for various violation types. The ML model we use is the ``original'' model in \cite{detoxify}. This model is trained with the \emph{training} set of the aforementioned Wikipedia dataset ($\set{D}_{\text{offline}}$ is the test set). For each input text, it outputs the probabilities of violating each of six policies. This model finetunes the pretrained BERT transformers \cite{devlin2019bert} for the training set and achieves an area under curve, averaged over the six policies, of 98.64 for the test set \citep{detoxify}. 

We follow Section~\ref{sec:overview-practice} to generate content contexts for our numerical experiments. For content $t \in \set{D}_{\text{offline}} \cup \set{D}_{\text{online}}$, the ML model outputs $x_{t,i}$ as its probability of violating policy $i \leq 6$. We divide the interval $[0,1]$ into five equal-size bins $\{\set{B}_j\}_{1 \leq j \leq 5}$. As discussed in Section~\ref{sec:overview-practice}, the feature vector of comment $t$, $\bolds{\phi}_t$, is a $30-$dimensional vector given by $\phi_{t,5(i-1)+j} = x_{t,i}\indic{x_{t,i} \in \set{B}_j}$ for $1 \leq i \leq 6, 1 \leq j \leq 5$.

\noindent{\textbf{Simulation of the pipeline}.} The simulation has $T = |\set{D}_{\text{online}}|$ periods. In period $t$, the $t-$th content from the dataset $\set{D}_{\text{online}}$ arrives to the system. An algorithm makes the classification, admission, and scheduling decisions as in Section~\ref{sec:model-instantiate}. The information an algorithm can rely on is the offline dataset $\set{D}_{\text{offline}}$, the contexts of all arrived content, and the actual severity (cost) of human reviewed content. In the simulation, we fix the service rate $\mu = 0.005$ but vary the number of reviewers $N(t).$ We call the product $N(t)\mu$ \emph{review ratio} as it captures the fraction of content humans can review. With the exception of the last setting of Section~\ref{sec:evaluation}, we set the predicted number of views of all content and the value of a non-policy-violating content to be one. We evaluate an algorithm based on the number of misclassified content at the end of the horizon, which is equivalent to the algorithm's loss \eqref{eq:policy-loss} when the cost of a post is either $1$ (policy-violating) or $-1$ (non-policy-violating). Since there is randomness in the number of reviewed content per period (which is a Bernoulli random variable with mean $\mu N(t)$ for period $t$), all results we present are averaged over $50$ independent runs.

\noindent{\textbf{Algorithm implementation}.} We consider three algorithms for the simulation.

The first algorithm is $\staticTh$, which mimics the algorithm from \cite{Avadhanula2022} described in Section~\ref{sec:overview-practice}. Since \cite{Avadhanula2022} did not reveal specific implementation of their classification and scheduling decisions, we consider heuristics motivated by \cite{makhijani2021quest}. For period $t$,
\begin{itemize}
\item Classification: we auto-delete the new content $t$ if its largest risk scores from the ML models, $\max_{i \leq d} \phi_{t,i}$, is larger than a threshold $\bar{x}$. We set the threshold $\bar{x}$ to be the $80-$th percentile of the maximum risk scores among all policy-violating content in the offline dataset $\set{D}_{\text{offline}}$; in Appendix~\ref{app:numerics}, we show that our experiment results remain robust to different thresholds.
\item Admission: we implement the UCB approach described in Section~\ref{sec:overview-practice}.
\item Scheduling: the system schedules for review the content $\tau \in \set{Q}(t)$ with the largest product of predicted severity (with upper confidence bound) and predicted number of views, $\bar{y}_{\tau} \cdot \text{pView}_{\tau}.$ This is also akin to the $\textsc{pIV}$ scheduling algorithm described in \cite{gocmen2025scheduling}.
\end{itemize}

The second algorithm we implement is the $\conbacid$ algorithm that we introduce in Section~\ref{sec:contextual}. The algorithm has parameters $\beta = \sqrt{T}$ and $\gamma = (T/\ln(T))^{-1/3}$. For a period $t$, we use the same estimation of $\bar{y}_{t}$ described in Section~\ref{sec:overview-practice} to obtain an upper confidence bound estimate of the loss of accepting content $t$,  $\bar{\ell}^+_t(t)$, and its lower confidence bound $\ubar{\ell}^+_t(t).$ Since the cost of content $t$ is $\pm 1$, we obtain a confidence bound of the loss of rejecting this content by setting $\ubar{\ell}^-_t(t) = 1 - \bar{\ell}^+_t(t), \bar{\ell}^-_t(t) = 1 - \ubar{\ell}^+_t(t).$ These give an alternative computation to the confidence bounds in \eqref{eq:conf-h-feature} and \eqref{eq:conf-r-feature} by $\ubar{c}_t(t) = \ubar{\ell}^+_t(t) - \bar{\ell}^-_t(t), \bar{c}_t(t) = \bar{\ell}^+_t(t) - \ubar{\ell}^-_t(t), $ and $\bar{\ell}_t(t) = \min(\bar{\ell}^+_t(t), \bar{\ell}^-_t(t)).$ We then multiple these values by $\text{pView}_t$ to account for content's different number of views. The classification step (Line~\ref{line:conbacid-classify}) in $\conbacid$ may overly reject non-policy-violating content initially. To address this issue, our implementation initially accepts content $t$ if $\ubar{c}_t(t) \leq -\gamma$, rejects it if $\bar{c}_t(t) \geq \gamma$, and otherwise follows the threshold-based classification in $\staticTh$. Moreover, there is only one group of content because in our simulation all content pieces have the same service rate. As a result, admitted content is scheduled in a first-come-first-serve manner.

The third algorithm, which we call $\bacid$ with offline ML, is identical to the last one but only uses offline data without online learning. This algorithm  trains the linear models of $\ell^+$ by reviewing all content in the offline dataset $\set{D}_{\text{offline}}$. It freezes these linear models for the online simulation without updating the estimation, even though there are new data from human reviews.

\subsection{Benefit of congestion awareness, online learning, and label-driven admission}\label{sec:evaluation}
We illustrate the numerical benefit for two main algorithmic contributions of our algorithms: (1) to adjust the admission threshold dynamically with the queue length and (2) to handle the issue of unknown (shifting) distribution with label-driven and optimistic admission.

\noindent \textbf{Congestion awareness self-modulates against nonstationarity}. The first setting is a non-stationary setting where $N(t) \equiv 20$ for $t \leq T / 2$ and $N(t) \equiv 4$ for $t > T / 2$. Recalling that $\mu = 0.005$, in the first interval reviewers can (in expectation) check $10\%$ of all content while in the second interval they can only check $2\%$. Figure~\ref{fig:horizon} compares the number of misclassifications of the three algorithms in this setting. A few observations follow. First, when the reviewer capacity is high ($t \leq T / 2$), all three algorithms perform similarly. However, when the capacity decreases, the system must be admitting content more carefully. Both $\conbacid$ and $\bacid$ with offline ML are able to adjust their admission threshold accordingly and thus perform better than the static threshold approach described in  \cite{Avadhanula2022}. However, note that $\conbacid$ has better performance than $\bacid$ with offline ML (though slightly) due to online learning. The next setting further supports the benefit of online learning.
\begin{figure}[!hbp]
\centering
\includegraphics[width=4in]{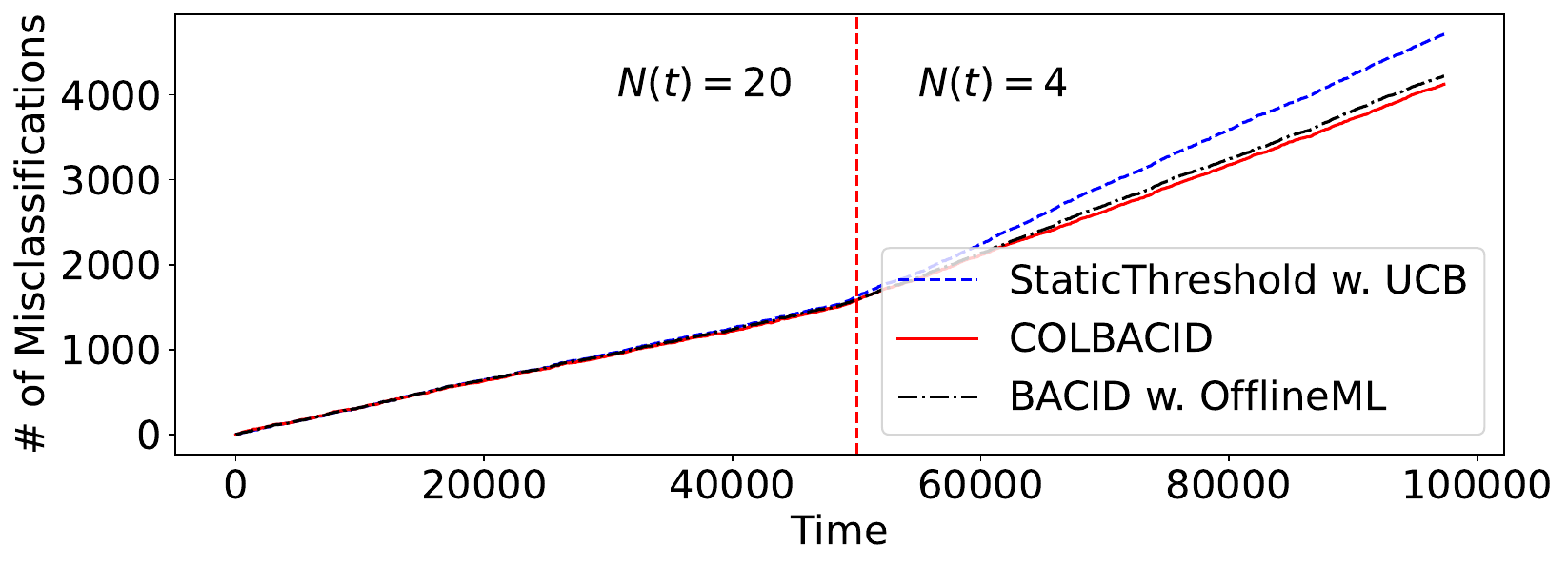}
\caption{Number of misclassifications as a function of time with nonstationary reviewer capacity}
\label{fig:horizon}
\end{figure}

\begin{figure}[htbp]
  \centering
  \begin{minipage}{0.48\textwidth}
    \centering
    
\includegraphics[width=3in]{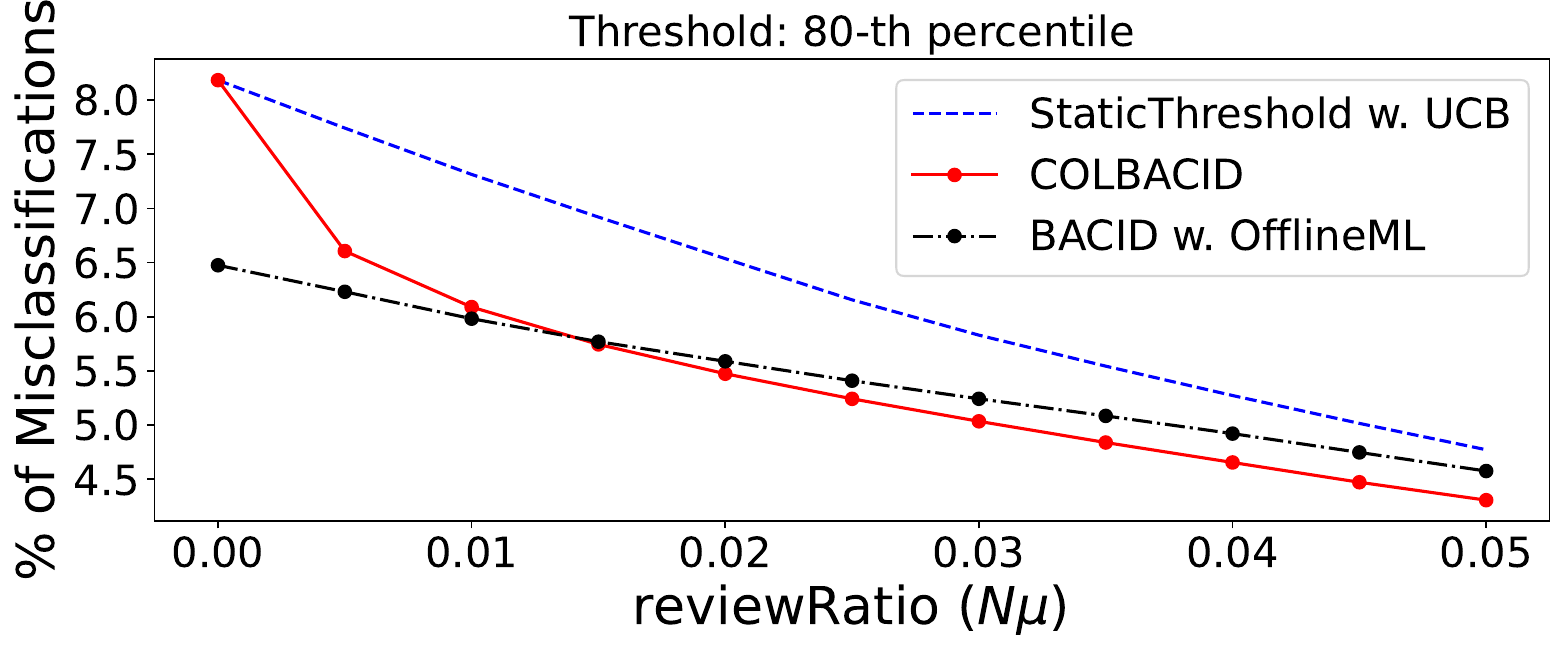}
  \end{minipage}%
  \\
  \begin{minipage}{1\textwidth}
    \centering
    \begin{tabular}{|l|c|c|c|c|c|}
    \hline
    \textbf{Review Ratio ($N\mu$)} & \textbf{0.01} & \textbf{0.02} & \textbf{0.03} & \textbf{0.04} & \textbf{0.05} \\
    \hline
    $\staticTh$  & 7.3\% & 6.5\% & 5.8\% & 5.3\% & 4.8\% \\
    \hline
    \textbf{$\conbacid$}                  & \textbf{6.1\%} & \textbf{5.5\%}  & \textbf{5.0\%} & \textbf{4.7\%} & \textbf{4.3\%} \\
    \hline
    BACID w.\ OfflineML       & 6.0\% & 5.6\% & 5.2\% & 4.9\% & 4.6\% \\
    \hline
  \end{tabular}
  \end{minipage}
\caption{Percentage of Misclassifications as a function of the review ratio}
\label{fig:80}
\end{figure}

\noindent
\textbf{Benefit of joint congestion awareness and online learning}. The second setting is a stationary setting with $N(t) \equiv N$ and we consider the percentage of misclassifications as a function of $N$ (or equivalently the review ratio $N\mu$) for the three algorithms. We vary $N \in \{0,1,\ldots,10\}$, so the range of considered review ratio is in $[0,5\%]$. Note that although Footnote~\ref{footnote:estimate} estimates a review ratio of $0.2\%$ to $1\%$ in Meta, we focus on a review ratio above $1\%$ because otherwise the number of content humans can review is at most $1000$ in our online dataset; at the scale of Meta this is not an issue as $0.2\%$ of a billion content pieces yields two millions of online samples per day.

Figure~\ref{fig:80} shows the percentage of misclassifications at the end of horizon for each algorithm. As discussed above, we fix the auto-deletion threshold to $80\%-$percentile maximum score in the offline dataset; Appendix~\ref{app:numerics} shows the results for other thresholds, which are similar to Figure~\ref{fig:80}. Consistently, $\conbacid$ outperforms $\staticTh$ by reducing more than $10\%$ misclassifications ($1 - 4.3\% / 4.8\%$ in the table of Figure~\ref{fig:80}). Moreover, when the review ratio is more than $1.5\%$, $\conbacid$ outperforms $\bacid$ w. OfflineML. This demonstrates the effectiveness of online learning when there is distribution shift between the offline dataset and the online dataset. Note that since $\conbacid$ only minimally uses the offline dataset for the auto-deletion threshold, it cannot have the same performance as $\bacid$ w. OfflineML when the online dataset is extremely small (review ratio less than $1\%$ in our simulation). As discussed above, this is not an issue for large social media platforms where millions of online samples are reviewed every day. Lastly, it is possible to further improve the numerical efficiency of $\conbacid$: in Appendix~\ref{app:numerics}, we simulate a variant of $\conbacid$ where instead of a FCFS scheduling, admitted content is scheduled based on the optimistically estimated loss $\bar{\ell}$. This algorithm shows better performance than $\conbacid$ when the review ratio is small.

\noindent \textbf{Importance of label-driven admission.} Motivated by our example in Section~\ref{sec:opti-fail}, the last setting shows that an optimism-only learning approach can fail to learn correct classification. Different from the above two settings, which set $\text{pView}_t \equiv 1$ for any content $t$, this setting varies this predicted number of views across content and studies how it impacts algorithms' performance. Specifically, for any non-policy-violating content $t \in \set{D}_{\text{online}}$, we set $\text{pView}_t \equiv P$ and vary $P \in \{1,\ldots,100\}.$ For policy-violating content $t$ we still set $\text{pView}_t \equiv 1.$ This setting captures practical scenarios where non-policy-violating content tends to have more views than policy-violating content. Fixing the review ratio to $0.05$, Figure~\ref{fig:misclass-selective} shows the percentage of misclassifications as a function of $P$ for both $\staticTh$, an optimism-only approach, and $\conbacid$, which contains a label-driven admission component. As Figure~\ref{fig:misclass-selective} shows, both $\staticTh$ and $\conbacid$ may increase their number of misclassifications as $P$ increases. However, $\conbacid$ is more robust and retains strong performance, while eventually $\staticTh$ fails to classify any policy-violating content (the percentage of policy-violating content in the online dataset is $8.2\%$). Figure~\ref{fig:nreview} explains why this happens. As $P$ increases, $\staticTh$ prioritizes more and more non-policy-violating content for reviews and thus eventually it will not review any policy-violating content. This mimics the phenomenon our theory predicts in Proposition~\ref{prop:fail-to-learn}. In contrast, thanks to the label-driven component, $\conbacid$ maintains a minimum number of human reviews for policy-violating content, although it has lower value ($\text{pView}_t$) to review.
\begin{figure}
\centering
  \begin{subfigure}[b]{0.5\textwidth}
  \centering
    \includegraphics[width=2in]{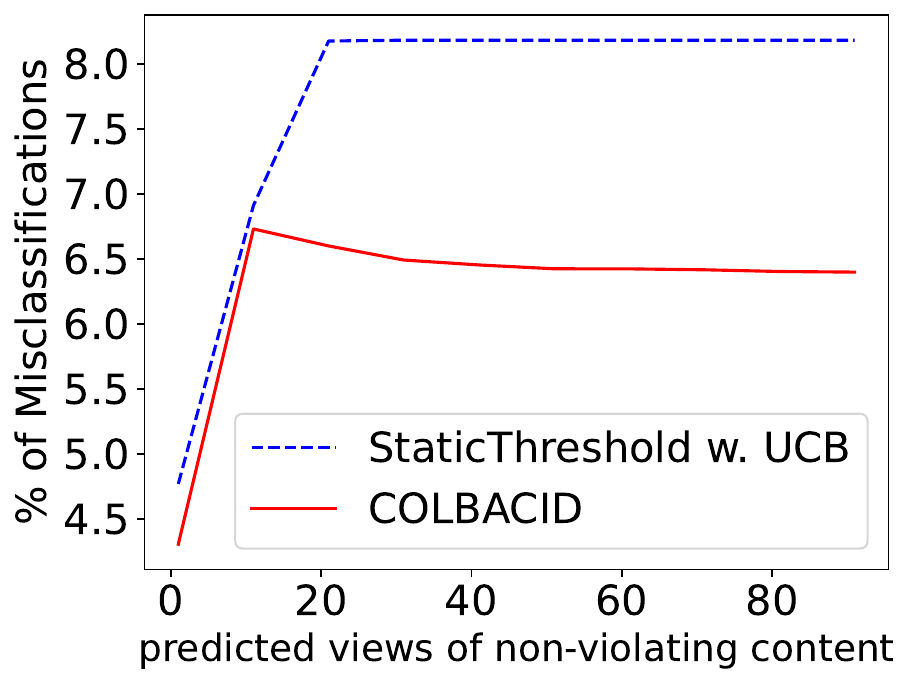}
    \caption{Percentage of misclassifications}
    \label{fig:misclass-selective}
  \end{subfigure}\hfill
  \begin{subfigure}[b]{0.5\textwidth}
  \centering
    \includegraphics[width=2in]{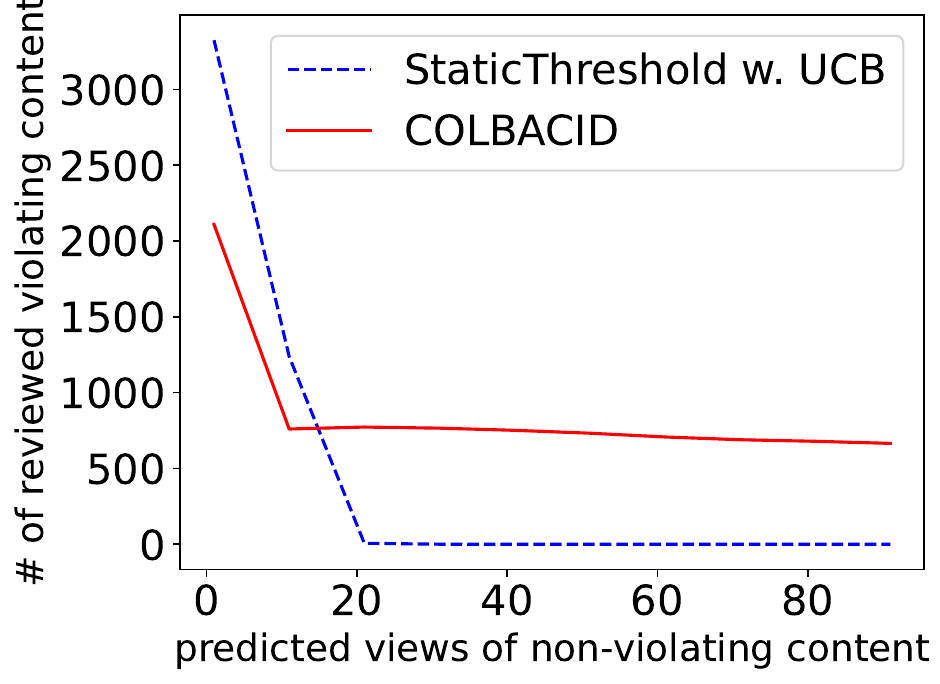}
    \caption{Number of reviewed violating content}
    \label{fig:nreview}
  \end{subfigure}
\caption{Algorithmic performance when non-violating content has more numbers of views}
\label{fig:selective}
\end{figure}

\section{Conclusion}\label{sec:conclusion}
Motivated by the prevalence of the AI-human interplay in high-stake operations, we propose a \emph{learning to defer} model with limited and time-varying human capacity. When the cost distribution of jobs is known, we propose $\textsc{BACID}$ which balances the idiosyncrasy loss avoided by admitting a job and the delay loss the admission could incur to other jobs. We show that $\textsc{BACID}$ achieves a near-optimal $O(\sqrt{KT})$ regret, with $K$ being the number of job types and $T$ being the number of jobs. When the cost distribution of jobs is unknown, we show that an optimism-only extension of $\textsc{BACID}$ fails to learn because of the \emph{selective sampling} nature of the system. That is, humans only see jobs admitted by the AI, while labels from humans affect AI's classification accuracy. To address this issue, we introduce label-driven admissions. Finally, we extend our algorithm to a contextual setting and derive a regret guarantee of $\tilde{O}(T^{2/3})$ without dependence on the number of types. Our type aggregation technique for a many-class queueing system and analysis for queueing-delayed feedback enable us to provide (to the best of our knowledge) the first online learning result in contextual queueing systems. Moreover, we test our algorithm on real data and find that it can substantially reduce the number of misclassifications of existing content moderation practice.

Our work opens up several interesting questions:
\begin{itemize}
\item First, we assume that reviewers produce perfect labels with uniform speeds. In practice, reviewers have different expertise and can have errors or heterogeneous review speeds. How can we design the scheduling algorithms when reviewers have skill-dependent non-perfect review qualities and speeds? How can we learn these quantities for new reviewers?
\item Second, although we allow for time-varying arrival patterns and human reviewing capacity, we assume stationary cost distributions of jobs for the learning problem. In practice the cost distributions can change over time due to, e.g., the update of the offline ML models in content moderation. It is thus useful to extend our work to a non-stationary learning setting. However, as noted in Appendix~\ref{app:sim-fail-to-learn}, typical heuristics like discounted UCB may worsen the selective sampling issue in Section~\ref{sec:opti-fail}. Therefore, new algorithmic ideas are needed to address selective sampling in non-stationary settings.
\item Third, the main focus of this paper is to minimize the cost of misclassified jobs. Content moderation, however, has multiple objectives: for example, the platform may want to maintain high precision in the decisions. In this case, the problem becomes a constrained optimization with a constraint on the number of erroneously removed legitimate content. Although our model can incorporate this consideration by setting a value for legitimate content (Remark~\ref{remark:value}), how to find the most suitable value of legitimate content in a learning, congested, and non-stationary setting is however unclear.
\item Fourth, this paper assumes \emph{exogenous} arrivals of jobs. In content moderation, the jobs are content created by users who may change their behavior according to the platform's decision. For example, malicious users may create more policy-violating content when the human review system is congested. Tackling such \emph{endogenous} job arrival patterns is an open direction that can further improve the efficiency of content moderation practice.
\item Finally, while we, and a long line of queueing literature, focus on the fluid benchmark (and the $w-$fluid benchmark in Appendix~\ref{app:w-fluid}) to handle non-stationary arrivals and human capacities, they are not necessarily the tightest benchmarks. Identifying a suitable benchmark for a non-stationary queueing system remains an open question.
\end{itemize}

\bibliographystyle{alpha}
\bibliography{references}

\newcommand{\etalchar}[1]{$^{#1}$}
\begin{thebibliography}{MVCV20}

\bibitem[ABB{\etalchar{+}}22]{Avadhanula2022}
Vashist Avadhanula, Omar~Abdul Baki, Hamsa Bastani, Osbert Bastani, Caner Gocmen, Daniel Haimovich, Darren Hwang, Dima Karamshuk, Thomas~J. Leeper, Jiayuan Ma, Gregory Macnamara, Jake Mullett, Christopher Palow, Sung Park, Varun~S. Rajagopal, Kevin Schaeffer, Parikshit Shah, Deeksha Sinha, Nicol{\'{a}}s~Stier Moses, and Peng Xu.
\newblock Bandits for online calibration: An application to content moderation on social media platforms.
\newblock 2022.

\bibitem[AD16]{AgrawalD16}
Shipra Agrawal and Nikhil~R. Devanur.
\newblock Linear contextual bandits with knapsacks.
\newblock In {\em Annual Conference on Neural Information Processing Systems 2016}, pages 3450--3458, 2016.

\bibitem[AD19]{agrawal2019bandits}
Shipra Agrawal and Nikhil~R Devanur.
\newblock Bandits with global convex constraints and objective.
\newblock {\em Operations Research}, 67(5):1486--1502, 2019.

\bibitem[ADL16]{AgrawalDL16}
Shipra Agrawal, Nikhil~R. Devanur, and Lihong Li.
\newblock An efficient algorithm for contextual bandits with knapsacks, and an extension to concave objectives.
\newblock In {\em the 29th Conference on Learning Theory, {COLT}}, 2016.

\bibitem[AGP{\etalchar{+}}23]{ahani2023dynamic}
Narges Ahani, Paul G{\"o}lz, Ariel~D Procaccia, Alexander Teytelboym, and Andrew~C Trapp.
\newblock Dynamic placement in refugee resettlement.
\newblock {\em Operations Research}, 2023.

\bibitem[APS11]{Abbasi-YadkoriPS11}
Yasin Abbasi{-}Yadkori, D{\'{a}}vid P{\'{a}}l, and Csaba Szepesv{\'{a}}ri.
\newblock Improved algorithms for linear stochastic bandits.
\newblock In {\em 25th Annual Conference on Neural Information Processing Systems 2011}, pages 2312--2320, 2011.

\bibitem[AS23]{adler2023bayesian}
Saghar Adler and Vijay Subramanian.
\newblock Bayesian learning of optimal policies in markov decision processes with countably infinite state-space.
\newblock {\em arXiv preprint arXiv:2306.02574}, 2023.

\bibitem[AZ09]{argon2009priority}
Nilay~Tan{\i}k Argon and Serhan Ziya.
\newblock Priority assignment under imperfect information on customer type identities.
\newblock {\em Manufacturing \& Service Operations Management}, 11(4):674--693, 2009.

\bibitem[BBS21]{bastani2021improving}
Hamsa Bastani, Osbert Bastani, and Wichinpong~Park Sinchaisri.
\newblock Improving human decision-making with machine learning.
\newblock {\em arXiv preprint arXiv:2108.08454}, 2021.

\bibitem[BDL{\etalchar{+}}20]{BrantleyDLMSSS20}
Kiant{\'{e}} Brantley, Miroslav Dud{\'{\i}}k, Thodoris Lykouris, Sobhan Miryoosefi, Max Simchowitz, Aleksandrs Slivkins, and Wen Sun.
\newblock Constrained episodic reinforcement learning in concave-convex and knapsack settings.
\newblock In {\em Annual Conference on Neural Information Processing Systems, NeurIPS 2020}, 2020.

\bibitem[BKR{\etalchar{+}}01]{borodin2001adversarial}
Allan Borodin, Jon Kleinberg, Prabhakar Raghavan, Madhu Sudan, and David~P Williamson.
\newblock Adversarial queuing theory.
\newblock {\em Journal of the ACM (JACM)}, 48(1):13--38, 2001.

\bibitem[BKS18]{badanidiyuru2018bandits}
Ashwinkumar Badanidiyuru, Robert Kleinberg, and Aleksandrs Slivkins.
\newblock Bandits with knapsacks.
\newblock {\em Journal of the ACM (JACM)}, 65(3):1--55, 2018.

\bibitem[BLM13]{boucheron2013concentration}
St{\'e}phane Boucheron, G{\'a}bor Lugosi, and Pascal Massart.
\newblock {\em Concentration inequalities: A nonasymptotic theory of independence}.
\newblock Oxford university press, 2013.

\bibitem[BLS14]{badanidiyuru2014resourceful}
Ashwinkumar Badanidiyuru, John Langford, and Aleksandrs Slivkins.
\newblock Resourceful contextual bandits.
\newblock In {\em Conference on Learning Theory}, pages 1109--1134. PMLR, 2014.

\bibitem[BM17]{brynjolfsson2017can}
Erik Brynjolfsson and Tom Mitchell.
\newblock What can machine learning do? workforce implications.
\newblock {\em Science}, 358(6370):1530--1534, 2017.

\bibitem[BP22]{bansak2020outcome}
Kirk Bansak and Elisabeth Paulson.
\newblock Outcome-driven dynamic refugee assignment with allocation balancing.
\newblock In {\em The 23rd {ACM} Conference on Economics and Computation}, pages 1182--1183. {ACM}, 2022.

\bibitem[BW08]{BartlettW08}
Peter~L. Bartlett and Marten~H. Wegkamp.
\newblock Classification with a reject option using a hinge loss.
\newblock {\em J. Mach. Learn. Res.}, 9:1823--1840, 2008.

\bibitem[BXZ23]{blanchet2023delay}
Jose Blanchet, Renyuan Xu, and Zhengyuan Zhou.
\newblock Delay-adaptive learning in generalized linear contextual bandits.
\newblock {\em Mathematics of Operations Research}, 2023.

\bibitem[BZ12]{besbes2012blind}
Omar Besbes and Assaf Zeevi.
\newblock Blind network revenue management.
\newblock {\em Operations research}, 60(6):1537--1550, 2012.

\bibitem[CBGO09]{cesa2009robust}
Nicolo Cesa-Bianchi, Claudio Gentile, and Francesco Orabona.
\newblock Robust bounds for classification via selective sampling.
\newblock In {\em Proceedings of the 26th annual international conference on machine learning}, pages 121--128, 2009.

\bibitem[cBi{\etalchar{+}}19]{dataset-civil}
cjadams, Daniel Borkan, inversion, Jeffrey Sorensen, Lucas Dixon, Lucy Vasserman, and nithum.
\newblock Jigsaw unintended bias in toxicity classification, 2019.
\newblock Retrieved on June 21, 2025 from \url{https://www.kaggle.com/competitions/jigsaw-unintended-bias-in-toxicity-classification}.

\bibitem[CDG{\etalchar{+}}18]{CortesDGMY18}
Corinna Cortes, Giulia DeSalvo, Claudio Gentile, Mehryar Mohri, and Scott Yang.
\newblock Online learning with abstention.
\newblock In {\em Proceedings of the 35th International Conference on Machine Learning, {ICML}}, volume~80 of {\em Proceedings of Machine Learning Research}, pages 1067--1075. {PMLR}, 2018.

\bibitem[CDM16]{CortesDM16a}
Corinna Cortes, Giulia DeSalvo, and Mehryar Mohri.
\newblock Learning with rejection.
\newblock In {\em Algorithmic Learning Theory - 27th International Conference, {ALT}}, volume 9925 of {\em Lecture Notes in Computer Science}, pages 67--82, 2016.

\bibitem[Cha14]{Chapelle14}
Olivier Chapelle.
\newblock Modeling delayed feedback in display advertising.
\newblock In {\em The 20th {ACM} {SIGKDD} International Conference on Knowledge Discovery and Data Mining, {KDD} '14}, pages 1097--1105. {ACM}, 2014.

\bibitem[Che19]{Cheung19}
Wang~Chi Cheung.
\newblock Regret minimization for reinforcement learning with vectorial feedback and complex objectives.
\newblock In {\em Annual Conference on Neural Information Processing Systems, NeurIPS 2019}, pages 724--734, 2019.

\bibitem[Cho57]{Chow57}
C.~K. Chow.
\newblock An optimum character recognition system using decision functions.
\newblock {\em {IRE} Trans. Electron. Comput.}, 6(4):247--254, 1957.

\bibitem[Cho70]{Chow70}
C.~K. Chow.
\newblock On optimum recognition error and reject tradeoff.
\newblock {\em {IEEE} Trans. Inf. Theory}, 16(1):41--46, 1970.

\bibitem[CJWS21]{choudhury2021job}
Tuhinangshu Choudhury, Gauri Joshi, Weina Wang, and Sanjay Shakkottai.
\newblock Job dispatching policies for queueing systems with unknown service rates.
\newblock In {\em Proceedings of the Twenty-second International Symposium on Theory, Algorithmic Foundations, and Protocol Design for Mobile Networks and Mobile Computing}, pages 181--190, 2021.

\bibitem[CLH23]{chen2023online}
Xinyun Chen, Yunan Liu, and Guiyu Hong.
\newblock Online learning and optimization for queues with unknown demand curve and service distribution.
\newblock {\em arXiv preprint arXiv:2303.03399}, 2023.

\bibitem[CLRS11]{ChuLRS11}
Wei Chu, Lihong Li, Lev Reyzin, and Robert~E. Schapire.
\newblock Contextual bandits with linear payoff functions.
\newblock In {\em Proceedings of the Fourteenth International Conference on Artificial Intelligence and Statistics, {AISTATS} 2011}, 2011.

\bibitem[CLS05]{Cesa-BianchiLS05}
Nicol{\`{o}} Cesa{-}Bianchi, G{\'{a}}bor Lugosi, and Gilles Stoltz.
\newblock Minimizing regret with label efficient prediction.
\newblock {\em {IEEE} Trans. Inf. Theory}, 51(6):2152--2162, 2005.

\bibitem[CMSS22]{CharusaieMSS22}
Mohammad{-}Amin Charusaie, Hussein Mozannar, David~A. Sontag, and Samira Samadi.
\newblock Sample efficient learning of predictors that complement humans.
\newblock In {\em International Conference on Machine Learning, {ICML} 2022}, 2022.

\bibitem[cSE{\etalchar{+}}17]{dataset-wiki}
cjadams, Jeffrey Sorensen, Julia Elliott, Lucas Dixon, Mark McDonald, nithum, and Will Cukierski.
\newblock Toxic comment classification challenge, 2017.
\newblock Retrieved on May 9, 2025 from \url{https://www.kaggle.com/c/jigsaw-toxic-comment-classification-challenge}.

\bibitem[DCLT19]{devlin2019bert}
Jacob Devlin, Ming-Wei Chang, Kenton Lee, and Kristina Toutanova.
\newblock Bert: Pre-training of deep bidirectional transformers for language understanding.
\newblock In {\em Proceedings of the 2019 conference of the North American chapter of the association for computational linguistics: human language technologies, volume 1 (long and short papers)}, pages 4171--4186, 2019.

\bibitem[DFC20]{De-ArteagaFC20}
Maria De{-}Arteaga, Riccardo Fogliato, and Alexandra Chouldechova.
\newblock A case for humans-in-the-loop: Decisions in the presence of erroneous algorithmic scores.
\newblock In {\em {CHI} Conference on Human Factors in Computing Systems}, pages 1--12. {ACM}, 2020.

\bibitem[DGS12]{DekelGS12}
Ofer Dekel, Claudio Gentile, and Karthik Sridharan.
\newblock Selective sampling and active learning from single and multiple teachers.
\newblock {\em J. Mach. Learn. Res.}, 13:2655--2697, 2012.

\bibitem[DHK{\etalchar{+}}11]{DudikHKKLRZ11}
Miroslav Dud{\'{\i}}k, Daniel~J. Hsu, Satyen Kale, Nikos Karampatziakis, John Langford, Lev Reyzin, and Tong Zhang.
\newblock Efficient optimal learning for contextual bandits.
\newblock In {\em Proceedings of the Twenty-Seventh Conference on Uncertainty in Artificial Intelligence}, pages 169--178. {AUAI} Press, 2011.

\bibitem[DKGG20]{DeKGG20}
Abir De, Paramita Koley, Niloy Ganguly, and Manuel Gomez{-}Rodriguez.
\newblock Regression under human assistance.
\newblock In {\em The Thirty-Fourth Conference on Artificial Intelligence, {AAAI} 2020}, pages 2611--2620. {AAAI} Press, 2020.

\bibitem[DOZR21]{DeOZR21}
Abir De, Nastaran Okati, Ali Zarezade, and Manuel~Gomez Rodriguez.
\newblock Classification under human assistance.
\newblock In {\em Thirty-Fifth Conference on Artificial Intelligence, {AAAI} 2021}, pages 5905--5913. {AAAI} Press, 2021.

\bibitem[Dur19]{durrett2019probability}
Rick Durrett.
\newblock {\em Probability: theory and examples}, volume~49.
\newblock Cambridge university press, 2019.

\bibitem[Eub18]{eubanks2018automating}
Virginia Eubanks.
\newblock {\em Automating inequality: How high-tech tools profile, police, and punish the poor}.
\newblock St. Martin's Press, 2018.

\bibitem[EW10]{El-YanivW10}
Ran El{-}Yaniv and Yair Wiener.
\newblock On the foundations of noise-free selective classification.
\newblock {\em J. Mach. Learn. Res.}, 2010.

\bibitem[FHL22]{FuHL22}
Hu~Fu, Qun Hu, and Jia'nan Lin.
\newblock Stability of decentralized queueing networks beyond complete bipartite cases.
\newblock In {\em Web and Internet Economics - 18th International Conference, {WINE}}, volume 13778 of {\em Lecture Notes in Computer Science}, pages 96--114. Springer, 2022.

\bibitem[FLW23a]{FreundLW22}
Daniel Freund, Thodoris Lykouris, and Wentao Weng.
\newblock Efficient decentralized multi-agent learning in asymmetric bipartite queueing systems.
\newblock {\em Operations Research}, 2023.

\bibitem[FLW23b]{freund2023quantifying}
Daniel Freund, Thodoris Lykouris, and Wentao Weng.
\newblock The transient cost of learning in queueing systems.
\newblock {\em arXiv preprint arXiv:2308.07817v3}, 2023.

\bibitem[FM22]{fu2022optimal}
Xinzhe Fu and Eytan Modiano.
\newblock Optimal routing to parallel servers with unknown utilities—multi-armed bandit with queues.
\newblock {\em IEEE/ACM Transactions on Networking}, 2022.

\bibitem[FMGC24]{felicioniimportance}
Nicol{\`o} Felicioni, Lucas Maystre, Sina Ghiassian, and Kamil Ciosek.
\newblock On the importance of uncertainty in decision-making with large language models.
\newblock {\em Transactions on Machine Learning Research}, 2024.

\bibitem[FSW18]{ferreira2018online}
Kris~Johnson Ferreira, David {Simchi-Levi}, and He~Wang.
\newblock Online network revenue management using thompson sampling.
\newblock {\em Operations research}, 66(6):1586--1602, 2018.

\bibitem[GBB23]{ge2023rethinking}
Haosen Ge, Hamsa Bastani, and Osbert Bastani.
\newblock Rethinking fairness for human-ai collaboration.
\newblock {\em arXiv:2310.03647}, 2023.

\bibitem[GBK20]{gorwa2020algorithmic}
Robert Gorwa, Reuben Binns, and Christian Katzenbach.
\newblock Algorithmic content moderation: Technical and political challenges in the automation of platform governance.
\newblock {\em Big Data \& Society}, 7(1):2053951719897945, 2020.

\bibitem[Git23]{Github}
Github.
\newblock Github copilot, 2023.
\newblock \url{https://github.com/features/copilot}. Accessed on December 19, 2023.

\bibitem[GLSW25]{gocmen2025scheduling}
Caner Gocmen, Thodoris Lykouris, Deeksha Sinha, and Wentao Weng.
\newblock Scheduling with uncertain holding costs and its application to content moderation.
\newblock {\em arXiv preprint arXiv:2505.21331}, 2025.

\bibitem[GT23]{gaitonde2023price}
Jason Gaitonde and {\'E}va Tardos.
\newblock The price of anarchy of strategic queuing systems.
\newblock {\em Journal of the ACM}, 2023.

\bibitem[HGH23]{huang2023queue}
Jiatai Huang, Leana Golubchik, and Longbo Huang.
\newblock Queue scheduling with adversarial bandit learning.
\newblock {\em arXiv preprint arXiv:2303.01745}, 2023.

\bibitem[HN11]{huang2009delay}
Longbo Huang and Michael~J. Neely.
\newblock Delay reduction via lagrange multipliers in stochastic network optimization.
\newblock {\em {IEEE} Trans. Autom. Control.}, 56(4):842--857, 2011.

\bibitem[HU20]{detoxify}
Laura Hanu and {Unitary team}.
\newblock Detoxify.
\newblock Github. https://github.com/unitaryai/detoxify, 2020.

\bibitem[HXLB22]{hsu2022integrated}
Wei-Kang Hsu, Jiaming Xu, Xiaojun Lin, and Mark~R Bell.
\newblock Integrated online learning and adaptive control in queueing systems with uncertain payoffs.
\newblock {\em Operations Research}, 70(2):1166--1181, 2022.

\bibitem[JGS13]{JoulaniGS13}
Pooria Joulani, Andr{\'{a}}s Gy{\"{o}}rgy, and Csaba Szepesv{\'{a}}ri.
\newblock Online learning under delayed feedback.
\newblock In {\em Proceedings of the 30th International Conference on Machine Learning, {ICML}}, volume~28, pages 1453--1461. JMLR.org, 2013.

\bibitem[JKK21]{johari2021matching}
Ramesh Johari, Vijay Kamble, and Yash Kanoria.
\newblock Matching while learning.
\newblock {\em Operations Research}, 69(2):655--681, 2021.

\bibitem[JSS22a]{Jia0S22}
Huiwen Jia, Cong Shi, and Siqian Shen.
\newblock Online learning and pricing for network revenue management with reusable resources.
\newblock In {\em NeurIPS}, 2022.

\bibitem[JSS22b]{jia2022online}
Huiwen Jia, Cong Shi, and Siqian Shen.
\newblock Online learning and pricing for service systems with reusable resources.
\newblock {\em Operations Research}, 2022.

\bibitem[KAJS18]{krishnasamy2018learning}
Subhashini Krishnasamy, Ari Arapostathis, Ramesh Johari, and Sanjay Shakkottai.
\newblock On learning the c$\mu$ rule in single and parallel server networks.
\newblock {\em arXiv preprint arXiv:1802.06723}, 2018.

\bibitem[KLK21]{KeswaniLK21}
Vijay Keswani, Matthew Lease, and Krishnaram Kenthapadi.
\newblock Towards unbiased and accurate deferral to multiple experts.
\newblock In {\em {AIES} '21: {AAAI/ACM} Conference on AI, Ethics, and Society}, pages 154--165. {ACM}, 2021.

\bibitem[KLL{\etalchar{+}}18]{kleinberg2018human}
Jon Kleinberg, Himabindu Lakkaraju, Jure Leskovec, Jens Ludwig, and Sendhil Mullainathan.
\newblock Human decisions and machine predictions.
\newblock {\em The quarterly journal of economics}, 133(1):237--293, 2018.

\bibitem[KMT98]{kelly1998rate}
Frank~P Kelly, Aman~K Maulloo, and David Kim~Hong Tan.
\newblock Rate control for communication networks: shadow prices, proportional fairness and stability.
\newblock {\em Journal of the Operational Research society}, 49:237--252, 1998.

\bibitem[Kon14]{kontorovich2014concentration}
Aryeh Kontorovich.
\newblock Concentration in unbounded metric spaces and algorithmic stability.
\newblock In {\em International conference on machine learning}, pages 28--36. PMLR, 2014.

\bibitem[KSJS21]{KrishnasamySJS21}
Subhashini Krishnasamy, Rajat Sen, Ramesh Johari, and Sanjay Shakkottai.
\newblock Learning unknown service rates in queues: {A} multiarmed bandit approach.
\newblock {\em Oper. Res.}, 69(1):315--330, 2021.

\bibitem[LCLS10]{LiCLS10}
Lihong Li, Wei Chu, John Langford, and Robert~E. Schapire.
\newblock A contextual-bandit approach to personalized news article recommendation.
\newblock In {\em Proceedings of the 19th International Conference on World Wide Web, {WWW} 2010}, pages 661--670. {ACM}, 2010.

\bibitem[LJWX23]{li2023experimenting}
Shuangning Li, Ramesh Johari, Stefan Wager, and Kuang Xu.
\newblock Experimenting under stochastic congestion.
\newblock {\em arXiv preprint arXiv:2302.12093}, 2023.

\bibitem[LLL22]{lebovitz2022engage}
Sarah Lebovitz, Hila {Lifshitz-Assaf}, and Natalia Levina.
\newblock To engage or not to engage with ai for critical judgments: How professionals deal with opacity when using ai for medical diagnosis.
\newblock {\em Organization Science}, 33(1):126--148, 2022.

\bibitem[LLSY21]{LiuLSY21}
Xin Liu, Bin Li, Pengyi Shi, and Lei Ying.
\newblock An efficient pessimistic-optimistic algorithm for stochastic linear bandits with general constraints.
\newblock In {\em Annual Conference on Neural Information Processing Systems}, pages 24075--24086, 2021.

\bibitem[LLWS11]{LiLWS11}
Lihong Li, Michael~L. Littman, Thomas~J. Walsh, and Alexander~L. Strehl.
\newblock Knows what it knows: a framework for self-aware learning.
\newblock {\em Mach. Learn.}, 82(3):399--443, 2011.

\bibitem[LM18]{liang2018minimizing}
Qingkai Liang and Eytan Modiano.
\newblock Minimizing queue length regret under adversarial network models.
\newblock {\em Proceedings of the ACM on Measurement and Analysis of Computing Systems}, 2(1):1--32, 2018.

\bibitem[LNZ24]{lee2024design}
Jiung Lee, Hongseok Namkoong, and Yibo Zeng.
\newblock Design and scheduling of an ai-based queueing system.
\newblock {\em arXiv preprint arXiv:2406.06855}, 2024.

\bibitem[LSFB22]{Leitao22}
Diogo Leit{\~{a}}o, Pedro Saleiro, M{\'{a}}rio A.~T. Figueiredo, and Pedro Bizarro.
\newblock Human-ai collaboration in decision-making: Beyond learning to defer.
\newblock {\em CoRR}, abs/2206.13202, 2022.

\bibitem[LSKM21]{LancewickiSKM21}
Tal Lancewicki, Shahar Segal, Tomer Koren, and Yishay Mansour.
\newblock Stochastic multi-armed bandits with unrestricted delay distributions.
\newblock In {\em Proceedings of the 38th International Conference on Machine Learning, {ICML}}, volume 139 of {\em Proceedings of Machine Learning Research}, pages 5969--5978. {PMLR}, 2021.

\bibitem[LSS06]{lin2006tutorial}
Xiaojun Lin, Ness~B Shroff, and Rayadurgam Srikant.
\newblock A tutorial on cross-layer optimization in wireless networks.
\newblock {\em IEEE Journal on Selected areas in Communications}, 24(8):1452--1463, 2006.

\bibitem[LSY21]{LiSY21}
Xiaocheng Li, Chunlin Sun, and Yinyu Ye.
\newblock The symmetry between arms and knapsacks: {A} primal-dual approach for bandits with knapsacks.
\newblock In {\em the 38th International Conference on Machine Learning, {ICML}}, volume 139 of {\em Proceedings of Machine Learning Research}, pages 6483--6492. {PMLR}, 2021.

\bibitem[LZM{\etalchar{+}}25]{lu2025}
Xingyu Lu, Tianke Zhang, Chang Meng, Xiaobei Wang, Jinpeng Wang, Yi-Fan Zhang, Shisong Tang, Changyi Liu, Haojie Ding, Kaiyu Jiang, Kaiyu Tang, Bin Wen, Hai-Tao Zheng, Fan Yang, Tingting Gao, Di~Zhang, and Kun Gai.
\newblock Vlm as policy: Common-law content moderation framework for short video platform.
\newblock In {\em Proceedings of the 31st ACM SIGKDD Conference on Knowledge Discovery and Data Mining V.2}, KDD '25, page 4682–4693, New York, NY, USA, 2025. Association for Computing Machinery.

\bibitem[{Met}22]{Facebook-content}
{Meta Platforms, Inc.}
\newblock How meta prioritizes content for review, 2022.
\newblock \url{https://transparency.fb.com/policies/improving/prioritizing-content-review/}. Accessed on May 25, 2024 [https://perma.cc/6ARK-SWB8?type=image].

\bibitem[{Met}23]{Facebook-standard}
{Meta Platforms, Inc.}
\newblock Facebook community standards, 2023.
\newblock \url{https://transparency.fb.com/policies/community-standards/}. Accessed on December 19, 2023.

\bibitem[MLBP15]{MandelLBP15}
Travis Mandel, Yun{-}En Liu, Emma Brunskill, and Zoran Popovic.
\newblock The queue method: Handling delay, heuristics, prior data, and evaluation in bandits.
\newblock In {\em Proceedings of the Twenty-Ninth {AAAI} Conference on Artificial Intelligence}, pages 2849--2856, 2015.

\bibitem[MMZ23]{mao2023principled}
Anqi Mao, Mehryar Mohri, and Yutao Zhong.
\newblock Principled approaches for learning to defer with multiple experts.
\newblock {\em arXiv preprint arXiv:2310.14774}, 2023.

\bibitem[Mod22]{linkedin}
Sanket Modi.
\newblock How our content abuse defense systems work to keep members safe, 2022.
\newblock \url{https://www.linkedin.com/blog/engineering/trust-and-safety/how-our-content-abuse-defense-systems-work-to-keep-members-safe}.

\bibitem[MPZ18]{MadrasCPZ18}
David Madras, Toni Pitassi, and Richard Zemel.
\newblock Predict responsibly: improving fairness and accuracy by learning to defer.
\newblock volume~31, 2018.

\bibitem[MS20]{MozannarS20}
Hussein Mozannar and David~A. Sontag.
\newblock Consistent estimators for learning to defer to an expert.
\newblock In {\em Proceedings of the 37th International Conference on Machine Learning, {ICML} 2020}, volume 119 of {\em Proceedings of Machine Learning Research}, pages 7076--7087. {PMLR}, 2020.

\bibitem[MS22]{mclaughlin2022algorithmic}
Bryce McLaughlin and Jann Spiess.
\newblock Algorithmic assistance with recommendation-dependent preferences.
\newblock {\em arXiv preprint arXiv:2208.07626}, 2022.

\bibitem[MSA{\etalchar{+}}21]{makhijani2021quest}
Rahul Makhijani, Parikshit Shah, Vashist Avadhanula, Caner Gocmen, Nicol{\'a}s~E Stier-Moses, and Juli{\'a}n Mestre.
\newblock Quest: Queue simulation for content moderation at scale.
\newblock {\em arXiv preprint arXiv:2103.16816}, 2021.

\bibitem[MSG22]{McLaughlinSG22}
Bryce McLaughlin, Jann Spiess, and Talia Gillis.
\newblock On the fairness of machine-assisted human decisions.
\newblock In {\em {ACM} Conference on Fairness, Accountability, and Transparency}, page 890. {ACM}, 2022.

\bibitem[MVCV20]{ManegueuVCV20}
Anne~Gael Manegueu, Claire Vernade, Alexandra Carpentier, and Michal Valko.
\newblock Stochastic bandits with arm-dependent delays.
\newblock In {\em the 37th International Conference on Machine Learning, {ICML}}, volume 119 of {\em Proceedings of Machine Learning Research}, pages 3348--3356. {PMLR}, 2020.

\bibitem[MX18]{massoulie2018capacity}
Laurent Massouli{\'e} and Kuang Xu.
\newblock On the capacity of information processing systems.
\newblock {\em Operations Research}, 66(2):568--586, 2018.

\bibitem[Nee22]{neely2022stochastic}
Michael Neely.
\newblock {\em Stochastic network optimization with application to communication and queueing systems}.
\newblock 2022.

\bibitem[NM23]{nguyen2023learning}
Quang~Minh Nguyen and Eytan Modiano.
\newblock Learning to schedule in non-stationary wireless networks with unknown statistics.
\newblock {\em arXiv preprint arXiv:2308.02734}, 2023.

\bibitem[Nov23]{Lawyer-chatgpt}
Matt Novak.
\newblock Lawyer uses chatgpt in federal court and it goes horribly wrong, 2023.
\newblock \url{https://www.forbes.com/sites/mattnovak/2023/05/27/lawyer-uses-chatgpt-in-federal-court-and-it-goes-horribly-wrong/?sh=735755923494}. Accessed on December 19, 2023.

\bibitem[NRP12]{NeelyRP12}
Michael~J. Neely, Scott Rager, and Thomas F.~La Porta.
\newblock Max weight learning algorithms for scheduling in unknown environments.
\newblock {\em {IEEE} Trans. Autom. Control.}, 57(5):1179--1191, 2012.

\bibitem[NZ20]{NeuZ20}
Gergely Neu and Nikita Zhivotovskiy.
\newblock Fast rates for online prediction with abstention.
\newblock In {\em Conference on Learning Theory, {COLT} 2020}, volume 125 of {\em Proceedings of Machine Learning Research}, pages 3030--3048. {PMLR}, 2020.

\bibitem[OC11]{OrabonaC11}
Francesco Orabona and Nicol{\`{o}} Cesa{-}Bianchi.
\newblock Better algorithms for selective sampling.
\newblock In {\em Proceedings of the 28th International Conference on Machine Learning, {ICML}}, pages 433--440, 2011.

\bibitem[PGBJ21]{PacchianoGBJ21}
Aldo Pacchiano, Mohammad Ghavamzadeh, Peter~L. Bartlett, and Heinrich Jiang.
\newblock Stochastic bandits with linear constraints.
\newblock In {\em The 24th International Conference on Artificial Intelligence and Statistics, {AISTATS}}, volume 130 of {\em Proceedings of Machine Learning Research}, pages 2827--2835. {PMLR}, 2021.

\bibitem[{Pin}23]{Pingan}
{Ping An Insurance}.
\newblock Insurance business, 2023.
\newblock \url{https://group.pingan.com/about_us/our_businesses/insurance.html}. Accessed on December 19, 2023.

\bibitem[PZ22]{PuchkinZ22}
Nikita Puchkin and Nikita Zhivotovskiy.
\newblock Exponential savings in agnostic active learning through abstention.
\newblock {\em {IEEE} Trans. Inf. Theory}, 68(7):4651--4665, 2022.

\bibitem[RBC{\etalchar{+}}19]{raghu2019algorithmic}
Maithra Raghu, Katy Blumer, Greg Corrado, Jon Kleinberg, Ziad Obermeyer, and Sendhil Mullainathan.
\newblock The algorithmic automation problem: Prediction, triage, and human effort.
\newblock {\em arXiv preprint arXiv:1903.12220}, 2019.

\bibitem[{Red}23]{reddit-automod}
{Reddit Inc.}
\newblock Automoderator, 2023.
\newblock \url{https://www.reddit.com/wiki/automoderator/}. Accessed on December 19, 2023.

\bibitem[RK21]{raisch2021artificial}
Sebastian Raisch and Sebastian Krakowski.
\newblock Artificial intelligence and management: The automation--augmentation paradox.
\newblock {\em Academy of management review}, 46(1):192--210, 2021.

\bibitem[Rob19]{roberts2019behind}
Sarah~T Roberts.
\newblock {\em Behind the screen}.
\newblock Yale University Press, 2019.

\bibitem[SBP21]{sentenac2021decentralized}
Flore Sentenac, Etienne Boursier, and Vianney Perchet.
\newblock Decentralized learning in online queuing systems.
\newblock pages 18501--18512, 2021.

\bibitem[SGMV20]{ShahGMV20}
Virag Shah, Lennart Gulikers, Laurent Massouli{\'{e}}, and Milan Vojnovic.
\newblock Adaptive matching for expert systems with uncertain task types.
\newblock {\em Oper. Res.}, 68(5):1403--1424, 2020.

\bibitem[SGVM22]{singh2022feature}
Simrita Singh, Itai Gurvich, and Jan~A Van~Mieghem.
\newblock Feature-based priority queuing.
\newblock {\em Available at SSRN 3731865}, 2022.

\bibitem[SI21]{Facebook-nyt}
Adam Satariano and Mike Isaac.
\newblock The silent partner cleaning up facebook for \$ 500 million a year.
\newblock {\em The New York Times}, 2021.
\newblock \url{https://www.nytimes.com/2021/08/31/technology/facebook-accenture-content-moderation.html}. Accessed on June 22, 2025.

\bibitem[Sli19]{slivkins2019introduction}
Aleksandrs Slivkins.
\newblock Introduction to multi-armed bandits.
\newblock {\em Foundations and Trends{\textregistered} in Machine Learning}, 12(1-2):1--286, 2019.

\bibitem[SSF23]{SlivkinsSF23}
Aleksandrs Slivkins, Karthik~Abinav Sankararaman, and Dylan~J. Foster.
\newblock Contextual bandits with packing and covering constraints: {A} modular lagrangian approach via regression.
\newblock In {\em The Thirty Sixth Annual Conference on Learning Theory}, volume 195 of {\em Proceedings of Machine Learning Research}, pages 4633--4656. {PMLR}, 2023.

\bibitem[SSM19]{StahlbuhkSM19}
Thomas Stahlbuhk, Brooke Shrader, and Eytan~H. Modiano.
\newblock Learning algorithms for scheduling in wireless networks with unknown channel statistics.
\newblock {\em Ad Hoc Networks}, 85:131--144, 2019.

\bibitem[SSM21]{StahlbuhkSM21}
Thomas Stahlbuhk, Brooke Shrader, and Eytan~H. Modiano.
\newblock Learning algorithms for minimizing queue length regret.
\newblock {\em {IEEE} Trans. Inf. Theory}, 67(3):1759--1781, 2021.

\bibitem[Sto05]{stolyar2005maximizing}
Alexander~L Stolyar.
\newblock Maximizing queueing network utility subject to stability: Greedy primal-dual algorithm.
\newblock {\em Queueing Systems}, 50:401--457, 2005.

\bibitem[SZB10]{SayediZB10}
Amin Sayedi, Morteza Zadimoghaddam, and Avrim Blum.
\newblock Trading off mistakes and don't-know predictions.
\newblock In {\em 24th Annual Conference on Neural Information Processing Systems}, pages 2092--2100, 2010.

\bibitem[TE92]{tassiulas1992stability}
L~Tassiulas and A~Ephremides.
\newblock Stability properties of constrained queueing systems and scheduling policies for maximum throughput in multihop radio networks.
\newblock {\em IEEE Transactions on Automatic Control}, 37(12):1936--1948, 1992.

\bibitem[VBN23]{VermaBN23}
Rajeev Verma, Daniel Barrej{\'{o}}n, and Eric~T. Nalisnick.
\newblock Learning to defer to multiple experts: Consistent surrogate losses, confidence calibration, and conformal ensembles.
\newblock In {\em International Conference on Artificial Intelligence and Statistics}, volume 206 of {\em Proceedings of Machine Learning Research}, pages 11415--11434. {PMLR}, 2023.

\bibitem[VCL{\etalchar{+}}20]{VernadeCLZEB20}
Claire Vernade, Alexandra Carpentier, Tor Lattimore, Giovanni Zappella, Beyza Ermis, and Michael Br{\"{u}}ckner.
\newblock Linear bandits with stochastic delayed feedback.
\newblock In {\em Proceedings of the 37th International Conference on Machine Learning, {ICML}}, volume 119 of {\em Proceedings of Machine Learning Research}, pages 9712--9721. {PMLR}, 2020.

\bibitem[VCP17]{VernadeCP17}
Claire Vernade, Olivier Capp{\'{e}}, and Vianney Perchet.
\newblock Stochastic bandit models for delayed conversions.
\newblock In {\em Proceedings of the Thirty-Third Conference on Uncertainty in Artificial Intelligence, {UAI}}. {AUAI} Press, 2017.

\bibitem[VLSV23]{vishwakarma2023promises}
Harit Vishwakarma, Heguang Lin, Frederic Sala, and Ramya~Korlakai Vinayak.
\newblock Promises and pitfalls of threshold-based auto-labeling.
\newblock In {\em Thirty-seventh Conference on Neural Information Processing Systems}, 2023.

\bibitem[Wal14]{Walton14}
Neil~S. Walton.
\newblock Two queues with non-stochastic arrivals.
\newblock {\em Oper. Res. Lett.}, 42(1):53--57, 2014.

\bibitem[WDY14]{wang2014close}
Zizhuo Wang, Shiming Deng, and Yinyu Ye.
\newblock Close the gaps: A learning-while-doing algorithm for single-product revenue management problems.
\newblock {\em Operations Research}, 62(2):318--331, 2014.

\bibitem[WHK20]{wilder2020learning}
Bryan Wilder, Eric Horvitz, and Ece Kamar.
\newblock Learning to complement humans.
\newblock {\em arXiv preprint arXiv:2005.00582}, 2020.

\bibitem[WSLJ15]{WuSLJ15}
Huasen Wu, R.~Srikant, Xin Liu, and Chong Jiang.
\newblock Algorithms with logarithmic or sublinear regret for constrained contextual bandits.
\newblock In {\em Annual Conference on Neural Information Processing Systems}, pages 433--441, 2015.

\bibitem[WW22]{WuW22}
Han Wu and Stefan Wager.
\newblock Thompson sampling with unrestricted delays.
\newblock In {\em The 23rd {ACM} Conference on Economics and Computation}, pages 937--955. {ACM}, 2022.

\bibitem[WX21]{walton2021learning}
Neil Walton and Kuang Xu.
\newblock Learning and information in stochastic networks and queues.
\newblock In {\em Tutorials in Operations Research: Emerging Optimization Methods and Modeling Techniques with Applications}, pages 161--198. INFORMS, 2021.

\bibitem[{X C}23]{X-rule}
{X Corp.}
\newblock The {X} rules, 2023.
\newblock \url{https://help.twitter.com/en/rules-and-policies/x-rules}. Accessed on December 19, 2023.

\bibitem[Yeh23]{forbesAI}
Guy Yehiav.
\newblock Will ai augment or replace workers?, 2023.
\newblock \url{https://www.forbes.com/sites/forbestechcouncil/2023/08/08/will-ai-augment-or-replace-workers/}.

\bibitem[YSY23]{yang2023learning}
Zixian Yang, R~Srikant, and Lei Ying.
\newblock Learning while scheduling in multi-server systems with unknown statistics: Maxweight with discounted ucb.
\newblock In {\em International Conference on Artificial Intelligence and Statistics}, pages 4275--4312. PMLR, 2023.

\bibitem[ZBW22]{zhong2022learning}
Yueyang Zhong, John~R Birge, and Amy Ward.
\newblock Learning the scheduling policy in time-varying multiclass many server queues with abandonment.
\newblock {\em SSRN}, 2022.

\bibitem[ZC16]{ZhangC16}
Chicheng Zhang and Kamalika Chaudhuri.
\newblock The extended littlestone's dimension for learning with mistakes and abstentions.
\newblock In {\em Proceedings of the 29th Conference on Learning Theory, {COLT}}, volume~49, pages 1584--1616. JMLR.org, 2016.

\bibitem[ZN22]{Zhu022a}
Yinglun Zhu and Robert Nowak.
\newblock Efficient active learning with abstention.
\newblock In {\em NeurIPS}, 2022.

\end{thebibliography}

\appendix

\section{Supplementary materials on \textsc{BACID} (Section~\ref{sec:known})}\label{app:known}
\subsection{Proof of the Lower Bound (Theorem~\ref{thm:lower-bound})}\label{app:lower-bound}
Consider a setting with one type of jobs ($K = 1$) and one reviewer in any period. The arrival rate is $\lambda = 1 / 2$ and the service rate is $\mu= 1/2$. Since there is only one type, we omit the notational dependence on types. Suppose that the cost for any job $t$ is $C_t = \pm 1$ with equal probability, yielding a loss of $\ell = 0.5$. By \eqref{eq:loss-decompose}, the loss of any policy $\pi$ is equal to 
\begin{equation}\label{eq:loss-example}
\expect{\set{L}^{\pi}(T)} = 0.25T - 0.5\expect{\sum_{t=1}^T A(t)} + 0.5\expect{Q(T+1)}.
\end{equation} 
Let $S(t)$ be a Bernoulli random variable with mean $0.5$. The queue evolves as $Q(1) = 0$ and $Q(t+1) = (Q(t) + A(t) - S(t))^+$ for any $t > 1$. 

To prove the lower bound, we couple the queue under any policy with an auxiliary queue $\{\tilde{Q}(t)\}_{t \leq T+1}$ where every new job is admitted. That is, $\tilde{Q}(1) = 0$ and $\tilde{Q}(t+1) = (\tilde{Q}(t) + \Lambda(t) - S(t))^+$. The below lemma bounds the difference between the two queue length processes.
\begin{lemma}\label{lem:coupling}
For any $t \geq 1$, $0 \leq \tilde{Q}(t) - Q(t) \leq \sum_{\tau=1}^{t - 1} (\Lambda(\tau) - A(\tau)).$
\end{lemma}
\begin{proof}
We prove the result by induction. When $t =1 $, $\tilde{Q}(1) = Q(1) = 0$. As an induction hypothesis, we assume that  the result is true for $t$ and we prove the induction step for $t + 1$. First,
\begin{align*}
\tilde{Q}(t+1) - Q(t+1) &= (\tilde{Q}(t) + \Lambda(t) - S(t))^+ - (Q(t) + A(t) - S(t))^+  \\
&\geq (Q(t) + \Lambda(t) - S(t))^+ - (Q(t)+A(t) - S(t))^+ \geq 0,
\end{align*}
where we use the fact that $x^+$ is increasing in $x$, $\Lambda(t) \geq A(t)$, and the induction hypothesis that $\tilde{Q}(t) \geq Q(t)$. In addition,
\begin{align*}
\tilde{Q}(t+1) - Q(t+1) &= (\tilde{Q}(t) + \Lambda(t) - S(t))^+ - (Q(t) + A(t) - S(t))^+  \\ 
&\leq \left(\tilde{Q}(t) + \Lambda(t) - S(t)\right) - \left(Q(t) + A(t) - S(t)\right)  \\
&\leq \sum_{\tau=1}^{t-1} (\Lambda(\tau) - A(\tau)) + \Lambda(t) - A(t) = \sum_{\tau=1}^t (\Lambda(\tau) - A(\tau)),
\end{align*}
where we use (i) the fact that the function $x^+$ is $1-$Lipchitz; (2) the fact that $\tilde{Q}(t) \geq Q(t)$ and $\Lambda(t) \geq A(t)$; (3) the induction hypothesis that $\tilde{Q}(t) - Q(t) \leq \sum_{\tau=1}^{t-1} (\Lambda(\tau) - A(\tau)).$
\end{proof} We next lower bound the  the expected length of the auxiliary queue by $\Omega(\sqrt{T})$. To do so, let $\Phi(x)$ be the distribution of a standard normal distribution. We use the following version of Berry-Esseen theorem (Theorem 3.4.17 in \cite{durrett2019probability}).
\begin{fact}\label{fact:berry}
Let $X_1,X_2,\ldots$ be i.i.d. with zero mean, $\expect{X_i^2} = \sigma^2$ and $\expect{|X_i|^3} = \rho < \infty$. If $F_n(x)$ is the distribution of $(X_1+\cdots+X_n) / (\sigma\sqrt{n})$, then $|F_n(x) - \Phi(x)| \leq 3\rho / (\sigma^3 \sqrt{n}).$
\end{fact}
The following result lower bounds the auxiliary queue.
\begin{lemma}\label{lem:bound-queue}
For any $t \geq 10000$, $\expect{\tilde{Q}(t + 1)} \geq \sqrt{t}/100.$\end{lemma}
\begin{proof}
By induction, we can show $\tilde{Q}(t+1) \geq \sum_{\tau=1}^t (A(\tau) - S(\tau))$. Hence, with $\bar{A} = 2\sum_{\tau=1}^t (A(\tau)-0.5) / \sqrt{t}$ and $\bar{S} = 2\sum_{\tau=1}^t (S(\tau)-0.5)/ \sqrt{t},$ the queue length is lower bounded by $\tilde{Q}(t + 1) \geq \sqrt{t}(\bar{A} - \bar{S}) / 2$. Since $\{A(\tau)\}_{\tau \leq t}$ are i.i.d. Bernoulli random variables with mean $0.5$ for any $\tau$, Fact~\ref{fact:berry} applies, showing that the distribution of $\bar{A}$, $F_A(x)$, satisfies $|F_A(x) - \Phi(x)| \leq 3 / \sqrt{t}$ for any $x$. Similarly, the distribution of $\bar{S}$, $F_S(x)$, satisfies $|F_S(x) - \Phi(x)| \leq 3/\sqrt{t}$ for any $x$. Recall that $\Phi(0) = 0.5$ and $\Phi(-0.25) \geq 0.4.$ Then $\Pr\{\bar{A} \geq 0\} \geq 0.4 - 3/\sqrt{t}$, $\Pr\{\bar{S} \leq -0.25\} \geq 0.4 - 3/\sqrt{t}$, and thus:
\[
\expect{\tilde{Q}(t)} \geq \expect{\tilde{Q}(t) \mid \bar{A} \geq 0, \bar{S} \leq -0.25} \Pr\{\bar{A} \geq 0, \bar{S} \leq -0.25\} \geq \sqrt{t}\left(0.4-3/\sqrt{t}\right)^2 / 8 \geq 0.01\sqrt{t}, 
\]
where we use the fact that $\tilde{Q}(t) \geq 0$ and $t \geq 10000.$
\end{proof}
\begin{proof}[Proof of Theorem~\ref{thm:lower-bound}]
The loss of any policy $\pi$ satisfies
\begin{align*}
\expect{\set{L}^{\pi}(T)} &= 0.25T - 0.5\expect{\sum_{t=1}^T A(t)} + 0.5\expect{Q(T+1)} \tag{By \eqref{eq:loss-example}} \\
&\hspace{-0.5in}\geq 0.25T - 0.5\expect{\sum_{t=1}^T A(t)} + 0.5\left(\expect{\tilde{Q}(T+1)} - \expect{\sum_{\tau=1}^T (\Lambda(\tau) - A(\tau))}\right) \tag{By Lemma~\ref{lem:coupling}} \\
&\hspace{-0.5in}= 0.5\expect{\tilde{Q}(T+1)} \geq \sqrt{T} / 200,~\text{for $T \geq 10000$.} \tag{By Lemma~\ref{lem:bound-queue} and $\expect{\Lambda(\tau)} = 0.5$ for any $\tau$}
\end{align*}
\end{proof}

\subsection{Bounding idiosyncrasy loss by Lagrangian via drift analysis (Lemma~\ref{lem:connect-idio-lag})}\label{app:lem-connect-idio-lag}
\begin{proof}[Proof of Lemma~\ref{lem:connect-idio-lag}] The first inequality is because 
\begin{align*}
\expect{L(T+1)-L(1)}&=\beta\expect{\sum_{t=1}^{T}\ell_{\kappa(t)}(1-A(t))} + \frac{1}{2}\expect{\sum_{k \in \set{K}} Q_k^2(T+1)} \\
&\geq \beta\expect{\sum_{t=1}^{T}\ell_{\kappa(t)}(1-A(t))}
\end{align*}
since $Q_k(T+1) \geq 0$ and $Q_k(1) = 0$.

For the second inequality, recall that $\Lambda_k(t) = 1$ if and only if job $t$ is of type $k$ and $A_k(t) = 1$ if and only if this job is also admitted for human review. We bound the drift for a period $t$ by using $\ell_{\kappa(t)}(1-A(t)) = \sum_{k \in \set{K}} \ell_k(\Lambda_k(t)-A_k(t))$ and thus
\begin{align*}
\expect{L(t+1) - L(t)} &= \beta\expect{\sum_{k \in \set{K}} \ell_k(\Lambda_k(t) - A_k(t))} + \expect{\sum_{k \in \set{K}} \frac{1}{2}(Q_k(t+1)^2 - Q_k(t)^2)} \\
&\hspace{-1in}= \beta\sum_{k \in \set{K}} \ell_k \lambda_k(t) + \expect{-\beta\sum_{k \in \set{K}} \ell_k A_k(t) + \sum_{k \in \set{K}} \frac{1}{2}((Q_k(t) + A_k(t) - S_k(t))^2 - Q^2_k(t))} \\
&\hspace{-1in}= \beta\sum_{k \in \set{K}} \ell_k \lambda_k(t) + \expect{-\beta\sum_{k \in \set{K}} \ell_k A_k(t) + \frac{1}{2}\sum_{k \in \set{K}} (A_k(t) - S_k(t))^2 + \sum_{k \in \set{K}} Q_k(t)(A_k(t) - S_k(t))} \\
&\hspace{-1in}\leq 1 + \beta\sum_{k \in \set{K}} \ell_k \lambda_k(t) - \expect{\sum_{k \in \set{K}} A_k(t)(\beta \ell_k - Q_k(t)) + \sum_{k \in \set{K}} Q_k(t)S_k(t)} \\
&\hspace{-1in}= 1 + \beta\sum_{k \in \set{K}} \ell_k\lambda_k(t) - \expect{\sum_{k \in \set{K}} A_k(t)(\beta \ell_k - Q_k(t)) + \sum_{k \in \set{K}} Q_k(t)\psi_k(t)\mu_k N(t)} \\
&\hspace{-1in}= 1 + \expect{f_t(\bolds{A}(t),\bolds{\psi}(t), \bolds{Q}(t))}
\end{align*}
where the second equality uses $\expect{\Lambda_k(t)} = \lambda_k(t)$ and the queueing dynamic ($S_k(t) = 1$ if a type-$k$ job is reviewed); the first inequality uses that $A_k(t) = 1$ for at most one type and $S_k(t) = 1$ for at most one type; the second-to-last equality uses that condition on $\psi_k(t) = 1$ (i.e., type $k$ is chosen to review), $S_k(t) = 1$ with probability $\mu_k N(t)$. The result follows by telescoping over periods.
\end{proof}
\subsection{Connecting Lagrangian with Primal Objective (Lemma~\ref{lem:bound-lagrang})}\label{app:lem-bound-lagrang}
\begin{proof}[Proof of Lemma~\ref{lem:bound-lagrang}]
Expanding over each period, $f(\{\bolds{a}^\star(t)\}_{t \in [T]},\{\bolds{\nu}^\star(t)\}_{t \in [T]},\bolds{u})$ is given by
\[
\beta\sum_{t=1}^T\sum_{k \in \set{K}} \ell_k(\lambda_k(t)-a_k^\star(t)) - \sum_{t=1}^T \sum_{k \in \set{K}} u_{t,k}\left(\mu_k\nu_k^\star(t)N(t)-a_k^\star(t) \right).
\] As $\{\{\bolds{a}^\star(t)\}_{t \in [T]},\{\bolds{\nu}^\star(t)\}_{t \in [T]}\}$ is a feasible solution to \eqref{eq:fluid}, for any period $t$, it holds that $\mu_k\nu_k^\star(t)N(t)-a_k^\star(t) \geq 0$. Therefore, since $u_{t,k} \geq 0$, it follows that \[f(\{\bolds{a}^\star(t)\}_{t\in[T]},\{\bolds{\nu}^\star(t)\}_{t\in[T]},\bolds{u}) \leq \beta\sum_{t=1}^T \sum_{k \in \set{K}} \ell_k(\lambda_k(t)-a_k^\star(t)) = \beta \set{L}^\star(T).\]
\end{proof}

\subsection{\textsc{BACID} optimizes per-period Lagrangian (Lemma~\ref{lem:bacid-mw})}\label{app:lem-bacid-mw}
\begin{proof}[Proof of Lemma~\ref{lem:bacid-mw}]
For a period $t$, the expectation only involves the uncertainty in the arrival and the review so 
\begin{align*}
&\expect{f_t(\bolds{A}(t),\bolds{\psi}(t),\bolds{Q}(t)) \mid \bolds{Q}(t) = \bolds{q}} \\
&\hspace{0.2in}= \beta\sum_{k \in \set{K}}\ell_k\lambda_k(t) - \left(\sum_{k \in \set{K}} \lambda_k(t)\indic{\beta \ell_k \geq q_k}(\beta \ell_k-q_k) + \sum_{k \in \set{K}} \psi_k(t) q_k\mu_k N(t)\right) \\
&\hspace{0.2in}\leq \beta\sum_{k \in \set{K}}\ell_k\lambda_k(t) - \left(\sum_{k \in \set{K}} a^\star_k(t)(\beta \ell_k-q_k) + N(t)\sum_{k \in \set{K}} \nu^\star_k(t) q_k\mu_k\right)  = f_t(\bolds{a}^\star(t),\bolds{\nu}^\star(t), \bolds{q})
\end{align*}
where the equality is because $\bacid$ admits a type-$k$ job only if $\beta \ell_k \geq q_k$ and the inequality is because $0 \leq a_k^{\star}(t) \leq \lambda_k(t)$ and $\sum_{k \in \set{K}} \nu^\star_k \leq 1$ and the definition that $\psi_k(t) = 1$ for the type $k$ with the largest $q_k\mu_k$, i.e., $\sum_{k \in \set{K}} \nu_k^\star(t) q_k \mu_k \leq \sum_{k \in \set{K}} \psi_k(t) q_k \mu_k$. 
\end{proof}
\section{Supplementary materials on \textsc{OLBACID} (Section~\ref{sec:unknown})}\label{app:unknown}
\subsection{Optimism-only fails to learn classification decisions (Proposition~\ref{prop:fail-to-learn})}\label{app:fail-to-learn}
This appendix gives the proof of Proposition~\ref{prop:fail-to-learn} in Appendix~\ref{app:prop-fail-to-learn}. In Appendix~\ref{app:sim-fail-to-learn}, we simulate several alternatives of $\textsc{BACID.UCB}$ using a setting similar to the example in Proposition~\ref{prop:fail-to-learn}. We observe that these alternatives either fail to learn classification decisions, or incur much higher loss, supporting the robustness of our observation in Proposition~\ref{prop:fail-to-learn} and the efficiency of our algorithm.

\subsubsection{Proof of Proposition~\ref{prop:fail-to-learn}}\label{app:prop-fail-to-learn}
We recall the setting. There are two types ($K = 2$) and the cost distributions for each type are
\[
\set{F}_1 = \{\pm 1~\text{with probability 0.5}\}, \set{F}_2 = \{-0.01~\text{with probability 0.05};~+0.01~\text{with probability 0.95}\}.
\]
The algorithm is run with exact knowledge of $\set{F}_1$ and with  knowledge that $\ell_2 \leq 0.01$. We assume the algorithm sets $\bar{\ell}_2(t) \leq 0.01$ as a result. Arrival rates are such that $\lambda_1(t) = 1,\lambda_2(t) = 0$ for $t \leq T_1 \coloneqq \lfloor \frac{\beta}{2} \rfloor$, and $\lambda_1(t) = 5/6, \lambda_2(t) = 1/6$ for $T_1 < t \leq T$. The service rate is $\mu_1 = \mu_2 = 1/2$ with $N(t) = 1$ for any $t$. We know that $\ell_2^+ - \ell_2^- = (0.95 - 0.05)0.01 > 0$, meaning that wrongly accepting a type-$2$ job incurs a constant positive loss. We define $Q_{\max} = \ell_1\beta = 0.5\beta$. 

\begin{lemma}\label{lem:bound-q2}
For any period $t$, the number of type-2 jobs admitted for review is $Q_2(t) < \frac{Q_{\max}}{25}$.
\end{lemma}
\begin{proof}
It holds that $\ell_1 = \min(\ell_1^+, \ell_1^-) = 0.5$. In addition, the admission rule is that $A(t) = 1$ if $\beta \bar{\ell}_{\kappa(t)}(t) > Q_{\kappa(t)}(t)$. Since we assume the algorithm sets $\bar{\ell}_2(t) \leq 0.01$, it holds that the algorithm admits a type-$2$ job only if
$Q_2(t) \leq \beta / 100$, which, by the same analysis in the proof of Lemma~\ref{lem:known-delay}, implies that $Q_2(t) \leq \beta /100 + 1 < \frac{Q_{\max}}{25}$, where the last inequality is by $Q_{\max} = 0.5\beta$ and $\beta / 100 + 1 < \beta / 50$ when $\beta > 100$.
\end{proof}
We define the event $\set{E}$ where for all periods $t$ in $[T_1,T]$, we have $Q_1(t) \geq \frac{Q_{\max}}{25}$, i.e., $\set{E} = \{\forall t \in [T_1+1,T], Q_1(t) \geq Q_{\max}/25\}$. The following lemma shows event $\set{E}$ happens with high probability.
\begin{lemma}\label{lem:prob-event-e}
With high probability, there are at least $\frac{Q_{\max}}{25}$ type-1 admitted jobs:
$\Pr\{\set{E}\} \geq 1 - 2/T$.
\end{lemma}
\begin{proof}
Our proof strategy is as follows. We first show that by a concentration bound, the queue length of type-$1$ jobs must increase linearly in the interval, given that the queue length is less than $Q_{\max}$ in the entire interval. In this case, for any period $t$, either there is a period $\tau$ with $Q_1(\tau) \geq Q_{\max}$ in the last $Q_{\max}/4$ periods, which implies $Q_1(t) \geq Q_{\max}-Q_{\max}/4$; or $Q_1(\tau)$ grows linearly in the last $Q_{\max}/4$ periods, which also implies $Q_1(t) = \Omega(Q_{\max}).$

Formally, we define $\hat{S}_k(t)$ as the Bernoulli random variable indicating whether the review of a type-$k$ job will finish in this period, given that we schedule such a job. Then $\hat{S}_k(t)$ has mean $N(t)\mu_k=1/2$, and the true review outcome is $S_k(t) = \hat{S}_k(t)\psi_k(t)$. 
We denote $E = \lceil Q_{\max} / 4\rceil$. For $t > T_1 > E$, we define 
\[
\tilde{\set{E}}_t = \left\{\sum_{\tau= t - E}^{t-1} \Lambda_1(\tau) \geq \frac{5E}{6} - \sqrt{E\ln T}\right\} \cap \left\{\sum_{\tau= t - E}^{t-1} \hat{S}_1(t) \leq \frac{E}{2} + \sqrt{E\ln T}\right\}.
\]
By Hoeffding Inequality and union bound, we have $\Pr\{\tilde{\set{E}}_t\} \geq 1 - 2/T^2$ since $\expect{\Lambda_1(t)} \geq \frac{5}{6}$. We define $\tilde{\set{E}} = \cap_{t > T_1}\{\tilde{\set{E}}_t\}$. By union bound, we have $\Pr\{\tilde{\set{E}}\} \geq 1 - 2/T$. It remains to show that under $\tilde{\set{E}}$, it holds that $Q_1(t) \geq \frac{Q_{\max}}{25}$ for any $t > T_1$, and thus $\Pr\{\set{E}\} \geq \Pr\{\tilde{\set{E}}\} \geq 1 - 2/T.$

We fix a period $t > T_1$ and look at the interval $[t - E,t - 1]$. Since $Q_{\max} = 0.5\beta \geq 25$, it holds that $E < Q_{\max} / 2$. We consider two cases:
\begin{itemize}
\item there exists $\tau \in [t-E,t-1]$ such that $Q_1(\tau) \geq Q_{\max}$. In this case, it holds that \[Q_1(t) \geq Q_1(\tau) - (t-\tau) \geq Q_1(\tau) - E \geq Q_{\max} - E \geq Q_{\max} / 2\]
where the first inequality is because at most one type-$1$ job is reviewed per period.
\item for every $\tau \in [t-E,t-1]$, it holds that $Q_1(\tau) < Q_{\max}$. The admission rule and the assumption that $\set{F}_1$ is known perfectly imply that $A(\tau) = \indic{0.5\beta > Q_1(\tau)} = 1$ for every $\tau \in [t-E,t-1]$ with $\kappa(\tau) = 1$. As a result, conditioning on $\tilde{\set{E}}$,
\begin{align*}
Q_1(t) &\geq Q_1(t-E)+\sum_{\tau=t-E}^{t-1} \indic{\kappa(\tau)=1}A(\tau) - \sum_{\tau=t-E}^{t-1} \hat{S}_1(\tau) = \sum_{\tau=t-E}^{t-1} \Lambda_1(\tau) - \sum_{\tau=t-E}^{t-1} \hat{S}_1(\tau) \\
&\geq \frac{5E}{6} - \sqrt{E\ln T} - \frac{E}{2} - \sqrt{E\ln T} = \frac{E}{3}-2\sqrt{E\ln T} \overset{(a)}{\geq} \frac{E}{6} \geq \frac{Q_{\max}}{25},
\end{align*}
where (a) is because $T \leq \exp(Q_{\max} / 576) \leq \exp(E / 144)$.
\end{itemize}
Combining the above discussions then shows that, conditioning on $\tilde{\set{E}}$, it holds that $Q_1(t) \geq Q_{\max} / 25$, which finishes the proof.
\end{proof}

\begin{proof}[Proof of Proposition~\ref{prop:fail-to-learn}]
Let us condition on $\set{E}$, which happens with probability at least $1 - 2/T$ by Lemma~\ref{lem:prob-event-e}. Since the algorithm follows the scheduling rule that prioritizes the review of the type with the larger $\mu_k Q_k(t)$ in every period, it is guaranteed that only type-$1$ jobs get reviewed after $T_1$ because $\mu_1 = \mu_2$ and $Q_1(t) \geq \frac{Q_{\max}}{25} > Q_2(t)$ where the last inequality is by Lemma~\ref{lem:bound-q2}. But type-$2$ jobs only arrive after $T_1$. As a result, there is no review for type-$2$ jobs.
\end{proof}

\subsubsection{The failure of other exploration heuristics}\label{app:sim-fail-to-learn}
Although Proposition~\ref{prop:fail-to-learn} assumes a particular form of optimism-only approaches ($\textsc{BACID.UCB}$), this section simulates several alternatives and finds that they either still fail to learn correct classification or incur much higher loss than $\bacidol$. 
Similar to Proposition~\ref{prop:fail-to-learn}, we consider a setting with two types. Type-$1$ jobs have costs $+1$ with probability $0.49$ and costs $-1$ with probability $0.51$. Type-$2$ jobs have costs $+1$ with probability $0.3$ and costs $-0.3$ with probability $0.7$. Their arrival rates are $\lambda_1(t) = 0.5, \lambda_2(t) \equiv 0.5$. We set $\mu_1 = 0.4, \mu_2 = 0.1$ and $N(t)=1$. The time horizon is $T = 10^5$. With perfect knowledge of the cost distributions, the optimal classification decisions are to accept all type-1 jobs but reject all type-2 jobs. The algorithm should admit all type-$1$ jobs but zero type-$2$ jobs for human reviews. We simulate three variants of $\textsc{BACID.UCB}$: 
\begin{itemize}
\item Loss-weighted $\textsc{BACID.UCB}$: this algorithm  follows the same classification and admission decisions as $\textsc{BACID.UCB}$ in Section~\ref{sec:opti-fail}, but schedules a type-$k$ job with $k$ that maximizes $\bar{\ell}_k(t)\mu_k Q_k(t).$ That is, its prioritization rule also incorporates the optimistic estimate $\bar{\ell}_k(t)$. This may help prioritize the job type with higher uncertainty in the unknown loss $\ell_k.$ 
\item $\textsc{BACID}$ with discounted UCB: a potential approach to encourage exploration is to use a time-based exploration bonus. We consider a variant of the above loss-weighted $\textsc{BACID.UCB}$. For the admission decision  that checks whether $\beta \bar{\ell}_{\kappa(t)}(t)\geq Q_{\kappa(t)}(t)$, we modify the confidence bound estimate of $\bar{\ell}_{\kappa(t)}(t)$. Instead of following \eqref{eq:conf-r}, we discount samples collected in period $s$ by $(0.99)^{t-s}$ for the confidence bound in period $t$. Hence, the further a type of job was reviewed, the larger confidence bound this type will obtain. This discounted UCB approach was used in \cite{Avadhanula2022} to handle non-stationarity in cost distributions, but our purpose here is to understand whether it can bring sufficient exploration for classification decisions. 
\item $\textsc{InitExplore}$: this algorithm adds an initial exploration phase to $\textsc{BACID.UCB}$. Specifically, in the first $T_1 = T^{2/3}\ln(T)^{1/3}$ periods, it admits all jobs into the human review queue and schedules either type of jobs for review with equal probability. After the first $T_1$ exploration periods, the algorithm follows $\textsc{BACID.UCB}.$ We set $T_1$ according to the optimal exploration length for explore-then-commit policies in multi-armed bandits (see page 16 of \cite{slivkins2019introduction}).
\end{itemize}
We also simulate $\bacidol$, which uses label-driven admission on top of $\textsc{BACID.UCB}$ to enable fast learning of classification). We take $\gamma = (T/(2\ln (T)))^{-1/3}$ as per Algorithm~\ref{algo:bacidol}. For all $\bacid$-based approach, we take $\beta = \sqrt{T / 2}$. The  confidence bounds are set according to \eqref{eq:conf-h} and \eqref{eq:conf-r} without including the constants.

Our results contain $1000$ independent runs of the above environment. Recall that the algorithm rejects a type-2 job when its \emph{empirical mean} $\hat{c}_2(t)$ is positive for a period $t$. Let $\hat{c}_2(t,i)$ be this estimate for a run $i \leq 1000$. To understand how good the classification decisions are, for every period $t$ we count the percentage of runs with $\hat{c}_2(t,i) \leq 0$, i.e., the (ensemble-average) probability that this optimism-only approach will \emph{wrongly accept} a type-$2$ job for period $t$.

\begin{center}
\begin{minipage}[c]{0.45\textwidth}
  \centering
  
\includegraphics[width=3in]{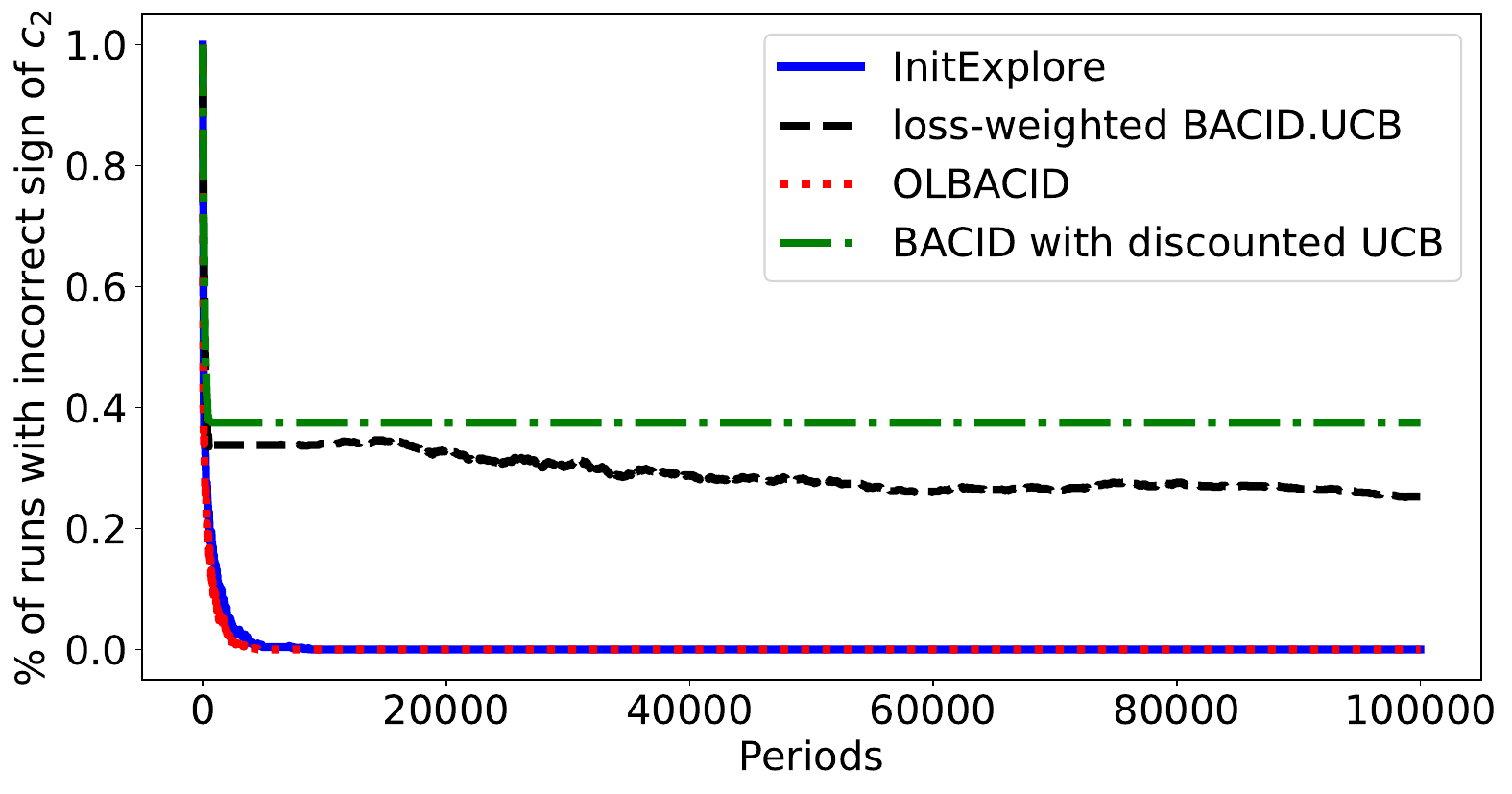}
  \captionof{figure}{Percentages of runs  in which an algorithm has an estimate $\hat{c}_2(t)$ with wrong signs for period $t$.}
    \label{fig:fail-to-learn}
\end{minipage}
\hspace{0.05\textwidth}
\begin{minipage}[c]{0.45\textwidth}
  \centering
  \begin{tabular}{|c|c|}
\hline
Algorithm                          & Average loss \\ \hline
Loss-weighted $\textsc{BACID.UCB}$ & $8375$                        \\ \hline
$\bacid$ with discounted UCB       & $8552$                        \\ \hline
$\textsc{InitExplore}$             & $9488$                        \\ \hline
$\bacidol$                         & $8148$                        \\ \hline
\end{tabular}
  \captionof{table}{Loss \eqref{eq:policy-loss} of algorithms averaged over $1000$ runs}
  \label{tab:mytable}
\end{minipage}
\end{center}

Figure~\ref{fig:fail-to-learn} shows the results. We observe that for the entire horizon, the optimism-only algorithms (loss-weighted $\textsc{BACID.UCB}$ and $\textsc{BACID}$ with discounted UCB) wrongly accept type-$2$ jobs for at least $20\%$ of the runs. This high failure probability indicates that the algorithm cannot (efficiently) learn how to classify type-$2$ jobs. What seems surprising is that $\textsc{BACID}$ with discounted UCB  may indeed lead to even fewer samples for type-$2$ than the non-discounted version (the former has a higher error probability than the latter in Figure~\ref{fig:fail-to-learn}). This happens because even though type 2 now has a larger confidence bound, type 1 also has a larger confidence bound with discounted UCB, which counteracts the incentive to review type 2.

In addition, although $\textsc{InitExplore}$ learns how to classify type-$2$ jobs in Figure~\ref{fig:fail-to-learn}, it incurs a much higher loss (Table~\ref{tab:mytable}) than $\bacidol$ due to its non-adaptive exploration phase. We also note that $\textsc{InitExplore}$ cannot accommodate a non-stationary setting where the arrival rates of types change across time because its exploration phase is limited to early arrivals. Compared to all these variants of $\textsc{BACID.UCB}$, our algorithm, $\bacidol$, learns to reject type-$2$ jobs for the classification decisions very soon and incurs the lowest loss.

This setting illustrates that simple variants of the optimism-only approaches do not ensure the algorithm to learn correct AI classification decisions in an efficient manner. Our algorithm addresses such shortcomings by using label-driven admission and forced scheduling.

\subsection{Regret guarantee for $\bacidol$ (Theorem~\ref{thm:bacidol})}\label{app:thm-bacidol}

\begin{proof}[Proof of Theorem~\ref{thm:bacidol}]
The regret $\reg^{\bacidol}(T)$ is upper bounded by $c_{\max}T$, which is less than the first term when $K \geq T$. Moreover, if $T \leq 2$, \[\reg^{\bacidol}(T) \leq \loss^{\bacidol}(T) \leq 2c_{\max} \lesssim c_{\max}\sqrt{T\ln T}.\] We thus focus on $T \geq 3$, $K \leq T$ and $
\beta = \sqrt{T / (Kc^2_{\max})}, \gamma = (\nicefrac{T}{K\ln T})^{-1/3}$, which is the setting Lemmas~\ref{lem:bacidol-delay},~\ref{lem:bacidol-class} and \ref{lem:bacidol-idio} are applicable.
Omitting dependence on $\sigma_{\max}, \hat{\mu}_{\min}$, constants and applying Lemmas~\ref{lem:bacidol-delay},~\ref{lem:bacidol-class} and \ref{lem:bacidol-idio} to the corresponding terms in \eqref{eq:learning-loss-decompose}, we obtain:
\begin{equation}\label{eq:bound-olbacid}
\begin{aligned}
\expect{\loss^{\bacidol}(T)} - \loss^\star(T) &\lesssim c^2_{\max}\beta K + \gamma\indic{\gamma > \eta} T + \frac{c_{\max}K\ln T}{\max(\eta,\gamma)^2} + T / \beta + Kc_{\max}\sqrt{T\ln T}\\
&\lesssim Kc_{\max}\sqrt{T\ln T} + \min\left(\frac{K\ln T}{\eta^2}, T^{2/3}(K\ln T)^{1/3}\right).
\end{aligned}
\end{equation}
\end{proof}

\subsection{Regret lower bound when cost distributions are unknown (Theorem~\ref{thm:bacidol-lowerbound})}\label{app:thm-bacidol-lowerbound}
The proof is similar to the standard approach of showing regret lower bound in multi-armed bandits. The proof uses Kullback-Leibler divergence (KL-divergence) defined as follows. For a finite sample space $\Omega$ and two of its probability distributions $p,q$, their KL-divergence is $\kl{p}{q} = \sum_{x \in \Omega} p(x)\ln \frac{p(x)}{q(x)}.$ Our proof uses the following properties of the KL-divergence.
\begin{fact}[theorem 2.2 of \cite{slivkins2019introduction}]\label{fact:KL-divergence}
The KL-divergence of distributions $p,q$ satisfies the next properties:
\begin{itemize}
\item chain rule for product distributions: suppose the sample space is a product $\Omega = \Omega_1\times \cdots \times \Omega_n$ and the two distributions satisfy $p = p_1 \times \cdots \times p_n, q = q_1 \times \cdots \times q_n$ with $p_j, q_j$ being distributions on $\Omega_j$. Then $\kl{p}{q} = \sum_{j=1}^n \kl{p_j}{q_j}.$
\item Pinsker's inequality: for any event $A \subseteq \Omega$, $2(p(A)-q(A))^2 \leq \kl{p}{q}.$
\end{itemize}
\end{fact}
\begin{fact}[fact 3 of \cite{freund2023quantifying}]\label{fact:bound-bernoulli}
The KL-divergence of two Bernoulli random variables with mean $g,q$ can be upper bounded by $\frac{(g-q)^2}{q(1-q)}.$
\end{fact}
Throughout, we suppose that for any type $k$, there is a sequence of costs $X_{k,1},X_{k,2},\ldots,X_{k,T}$ sampled i.i.d. from the cost distribution $\set{F}_k$. When humans successfully review a type-$k$ job, they observe the next cost on the cost sequence of type-$k$. This modification does not change the distribution of the original system since costs of type-$k$ jobs are i.i.d. from $\set{F}_k$. Under this modification, the decision of the algorithm in period $t$ must be a function of the types of arrived jobs $\{\kappa(1),\ldots,\kappa(t)\}$, the observed services $\{\set{S}(1),\ldots,\set{S}(t-1)\}$, and the costs of reviewed jobs $\{\{X_{k,i}\}_{i \leq N_k(t)}\}_{k \in \set{K}}$, where we denote $N_k(t)$ by the number of reviewed type-$k$ jobs before period $t$.
\begin{proof}[Proof of Theorem~\ref{thm:bacidol-lowerbound}]
Fix the time horizon $T$ and consider two settings, denoted by settings $(a)$ and $(b)$. In both settings there are two types such that $\lambda_1 = \lambda_2 = 0.5$ and $\mu_1 = \mu_2 = 0.5$ with $N(t) \equiv 1$. Type-1 jobs have costs $\pm 4$, each with probability $0.5$. Type-2 jobs have a setting-dependent distribution $\set{F}_2^{i}$ for $i \in \{(a),(b)\}$. Specifically, setting $i$ has a constant $p_i$ such that type-2 jobs have costs $-1$ with probability $p_{i}$ and costs $+1$ with probability $1-p_i$. We set $p_{(a)} = 0.5(1-\varepsilon)$ and $p_{(b)}=0.5(1+\varepsilon)$ for some $\varepsilon$ that we tune later. For $i \in \{(a),(b)\}$, We use $\expectsub{i}{\cdot}, \Prsub{i}{\cdot}$ to highlight on which setting the expectation and probability is evaluated.

We first calculate the loss of any feasible policy $\pi$ in either setting. For period $t$, let $\Lambda_2^A(t)$ be an indicator such that it is one when there is a type-$2$ job arrival and the policy accepts this job. Similarly, define an indicator $\Lambda_2^R(t)$ which is equal to one when the policy rejects a new type-$2$ job. In addition, let $S_2^A(t)$ indicate whether in period $t$ humans successfully review a type-$2$ job that was initially accepted. Similarly define $S_2^R(t)$ to indicate whether humans successfully review a type-$2$ job that was initially rejected. Let $N_2 = \sum_{t \leq T} S_2(t)$ be the number of reviewed type-$2$ jobs. Then for setting $i$, the expected loss $\expectsub{i}{\set{L}^{\pi}(T)}$ of policy $\pi$ is given by 
\[
2\expectsub{i}{\sum_{t\leq T} (\Lambda_1(t)-S_1(t))} + \expectsub{i}{\sum_{t\leq T} (\Lambda_2^A(t) - S_2^A(t))(1-p_i)} + \expectsub{i}{\sum_{t\leq T} (\Lambda_2^R(t) - S_2^R(t))p_i}, 
\]
where the first term captures that any unreviewed type-$1$ jobs have expected loss $2$; the second term is that any unreviewed type-$2$ jobs that were initially accepted incur expected loss $1 - p_i$ since type-$2$ jobs have costs $+1$ with probability $1-p_i$; the third term means that any unreviewed type-$2$ jobs that were initially rejected incur expected loss $p_i$.

Noting that $\expectsub{i}{\Lambda_1(t)} = 0.5$ for any $t$ and $\expectsub{i}{\sum_{t \leq T}(S_1(t) + S_2^A(t) + S_2^R(t))} \leq T / 2$, 
\begin{align}
\expectsub{i}{\set{L}^{\pi}(T)} &\geq T - \left(T - 2\expectsub{i}{\sum_{t \leq T}( S_2^A(t) + S_2^R(t))}\right) \nonumber\\
&~~+ \expectsub{i}{\sum_{t\leq T} (\Lambda_2^A(t) - S_2^A(t))(1-p_i)} + \expectsub{i}{\sum_{t\leq T} (\Lambda_2^R(t) - S_2^R(t))p_i} \nonumber\\
&\hspace{-1in}= (2-\max(p_i,1-p_i))\expectsub{i}{\sum_{t \leq T} (S_2^A(t) + S_2^R(t))} + (1-p_i)\expectsub{i}{\sum_{t\leq T} \Lambda_2^A(t)} + p_i\expectsub{i}{\sum_{t\leq T} \Lambda_2^R(t)} \nonumber\\
&\hspace{-1in}\geq \expectsub{i}{\sum_{t \leq T} S_2(t) + (1-p_i)\sum_{t\leq T} \Lambda_2^A(t) + p_i\sum_{t\leq T} \Lambda_2^R(t)} \nonumber\\
&\hspace{-1in}= \expectsub{i}{\sum_{t \leq T} S_2(t) + p_i \sum_{t\leq T} \Lambda_2(t) + (1-2p_i)\sum_{t\leq T} \Lambda_2^A(t)}=\frac{p_i T}{2}+\expectsub{i}{\sum_{t \leq T} S_2(t) + (1-2p_i)\sum_{t\leq T} \Lambda_2^A(t)}.
\label{eq:lower-bound-loss}
\end{align}
Denote the total number of reviewed type-$2$ jobs by $N_2 = \sum_{t \leq T} S_2(t)$. Conditioning on $N_2 = n$, in setting $i \in \{(a),(b)\}$, the indicators $\Lambda_2^A(t), \Lambda_2^R(t)$ are random variables of the sample space formed by $\Omega^i = \Omega' \times X^{i}_{2,1} \times \cdots X^{i}_{2,n}$, where  $\Omega'$ denotes the sample space unrelated to type-$2$ jobs' costs and is thus common for setting $(a)$ and $(b)$; for $j \leq n$,  $X^{i}_{2,j}$ denotes the costs of the $j$-th reviewed type-$2$ jobs and has i.i.d. distribution of $\set{F}_2^{i}.$ Using Fact~\ref{fact:KL-divergence}, for any $t \leq T$,
\begin{align*}
2\left|\Prsub{(a)}{\Lambda_2^A(t) = 1 \mid N_2 = n} -\Prsub{(b)}{\Lambda_2^A(t) = 1 \mid N_2 = n} \right|^2 &\leq \kl{\Prsub{(a)}{\cdot \mid N_2 = n}}{\Prsub{(b)}{\cdot \mid N_2 = n}} \\
&\leq \sum_{j=1}^n \kl{\set{F}_2^{(a)}}{\set{F}_2^{(b)}} \leq  \frac{n\varepsilon^2}{(1-\varepsilon)^2/4},
\end{align*}
where the first two inequalities use Fact~\ref{fact:KL-divergence} and the last inequality uses Fact~\ref{fact:bound-bernoulli}. As a result, 
\begin{equation}\label{eq:bound-dist-diff}
\begin{aligned}
|\expectsub{(a)}{\Lambda_2^A(t) \mid N_2 = n} - \expectsub{(b)}{\Lambda_2^A(t) \mid N_2 = n}| &\leq \frac{\sqrt{2n}\varepsilon}{1-\varepsilon}.
\end{aligned}
\end{equation}
For setting $i$, define the \emph{pseudo-loss} $\tilde{L}^i \coloneqq \sum_{t \leq T} S_2(t) + (1-2p_i)\sum_{t\leq T} \Lambda_2^A(t)$, which is the random variable taken expectation on the right hand side of \eqref{eq:lower-bound-loss}. The sum of the pseudo-losses satisfies
\begin{align*}
&~~\expectsub{(a)}{\tilde{L}^{(a)} \mid N_2 = n} + \expectsub{(b)}{\tilde{L}^{(b)} \mid N_2 = n} \\
&= 2n+\varepsilon\sum_{t \leq T} \expectsub{(a)}{\Lambda_2^A(t) \mid N_2 = n}-\varepsilon \sum_{t \leq T} \expectsub{(b)}{\Lambda_2^A(t) \mid N_2 = n}\\
&\geq 2n - \varepsilon \sum_{t \leq T} \left|\expectsub{(a)}{\Lambda_2^A(t) \mid N_2 = n} - \expectsub{(b)}{\Lambda_2^A(t) \mid N_2 = n}\right| \geq 2n - \frac{\sqrt{2n}\varepsilon^2 T}{1-\varepsilon}.
\end{align*}
Letting $\varepsilon = T^{-1/3}$, which is at most $0.25$ by the assumption that $T \geq 64$, the above inequality simplifies to 
$\expectsub{(a)}{\tilde{L}^{(a)} \mid N_2 = n} + \expectsub{(b)}{\tilde{L}^{(b)} \mid N_2 = n} \geq 2n - \frac{4}{3}\sqrt{2n}T^{1/3}$. The right hand side is minimized when $n = \frac{2}{9}\cdot T^{2/3}$. Therefore,
\[
\expectsub{(a)}{\tilde{L}^{(a)}} + \expectsub{(b)}{\tilde{L}^{(b)}} \geq \min_{n} \left(\expectsub{(a)}{\tilde{L}^{(a)} \mid N_2 = n} + \expectsub{(b)}{\tilde{L}^{(b)} \mid N_2 = n}\right) \geq -\frac{4}{9}T^{2/3}.
\]
Plugging this bound into \eqref{eq:lower-bound-loss} gives
\[
\expectsub{(a)}{\set{L}^{\pi}(T)} + \expectsub{(b)}{\set{L}^{\pi}(T)} = \frac{T}{2} + \expectsub{(a)}{\tilde{L}^{(a)}} + \expectsub{(b)}{\tilde{L}^{(b)}} \geq \frac{T}{2} - \frac{4}{9}\cdot T^{2/3}.
\]
The optimal fluid benchmark of \eqref{eq:fluid} gives $\set{L}^\star(T) = \lambda_2\min(p_i,1-p_i)T = 0.25(1-\varepsilon)T$ for both settings. Therefore,
\begin{align*}
\max\left(\expectsub{(a)}{\set{L}^{\pi}(T)}, \expectsub{(b)}{\set{L}^{\pi}(T)}\right) - \set{L}^\star(T) &\geq \frac{\expectsub{(a)}{\set{L}^{\pi}(T)} + \expectsub{(b)}{\set{L}^{\pi}(T)}}{2} - \set{L}^\star(T) \\
&\hspace{-1in}\geq \frac{T}{4} - \frac{2 T^{2/3}}{9} - \frac{T(1-\varepsilon)}{4} \geq \frac{T^{2/3}}{4} - \frac{2 T^{2/3}}{9} \geq \frac{T^{2/3}}{36}.
\end{align*}
We have thus shown that for any feasible policy $\pi$, there exists a setting (either $(a)$ or $(b)$) in which the policy incurs $\Omega(T^{2/3})$ regret.
\end{proof} 

\subsection{Number of jobs admitted to label-driven queue (Lemma~\ref{lem:connect-et-qe})}\label{app:lem-connect-et-qe}
\begin{proof}[Proof of Lemma~\ref{lem:connect-et-qe}]
We define $S^{\ld}_k(t) \in \{0,1\}$ to be equal to one if and only if $S_k(t) = 1$ and $Q^{\ld}(t) = 1$, i.e., humans review a type-$k$ job from the label-driven queue (whose length is at most one). This implies that:
$$Q^{\ld}(t+1) - Q^{\ld}(t) = E(t) - \sum_{k \in \set{K}}  S_k^{\ld}(t) = E(t) - \sum_{k \in \set{K}} S_k^{\ld}(t)Q^{\ld}(t).$$ We then define an indicator $\set{I}_{k,t}$ which is equal to $1$ if $\set{Q}^{\ld}(t)$ contains a type-$k$ job. Then conditioning on $Q^{\ld}(t) = 1$, the expectation of $\sum_{k \in \set{K}}S_k^{\ld}(t)$ is lower bounded by
\begin{align}
\expect{\sum_{k \in \set{K}} S_k^{\ld}(t) \mid Q^{\ld}(t) = 1} &= \sum_{k \in \set{K}} \Pr\{\set{I}_{k,t} = 1 | Q^{\ld}(t) = 1\}\expect{S_k^{\ld}(t) \mid Q^{\ld}(t) = 1, \set{I}_{k,t} = 1} \nonumber\\
&= \sum_{k \in \set{K}} \Pr\{\set{I}_{k,t} = 1 | Q^{\ld}(t) = 1\}N_k(t)\mu_k \geq \hat{\mu}_{\min}.\label{eq:lowerbound-service}
\end{align}
The second equality is because $\bacidol$ prioritizes jobs in the label-driven queue (Line~\ref{line:bacidol-prio} in Algorithm~\ref{algo:bacidol}); the last inequality uses $N_k(t)\mu_k \geq \hat{\mu}_{\min}$ and $\sum_{k \in \set{K}} \Pr\{\set{I}_{k,t} = 1 | Q^{\ld}(t) = 1\} = 1$. Therefore,
\[
\expect{Q^{\ld}(t+1) - Q^{\ld}(t)} = \expect{E(t)} - \expect{Q^{\ld}(t)\sum_{k \in \set{K}} S_k^{\ld}(t)} \geq \expect{E(t)} - \hat{\mu}_{\min}\expect{Q^{\ld}(t)}
\]
where we use \eqref{eq:lowerbound-service} for the inequality. Telescoping for $t = 1,\ldots,T$, we obtain $\hat{\mu}_{\min}\sum_{t=1}^T \expect{Q^{\ld}(t)} \leq \expect{\sum_{t=1}^T E(t)}$, which finishes the proof.
\end{proof}

\subsection{Confidence bounds are valid with high probability (Lemma~\ref{lem:prob-good-event})}\label{app:lem-prob-good-event}
We first need the following result that a Lipschitz function of a sub-Gaussian random variable is still sub-Gaussian. 
\begin{lemma}\label{lem:sub-gaussian}
Given a sub-Gaussian random variable $X$ with variance proxy $\sigma^2$ and a $1-$Lipschitz function $f(\cdot)$, the random variable $f(X)$ is sub-Gaussian with variance proxy $2\sigma^2.$
\end{lemma}
\begin{proof}
The proof follows a simplified version of the proof of \cite[Theorem~1]{kontorovich2014concentration}. Define $V = f(X) - \expectsub{X'}{f(X')}$ where $X'$ is an independent copy of $X$. By definition, showing that $f(X)$ is sub-Gaussian with variance proxy $2\sigma^2$ is equivalent to showing that $\expectsub{X}{\exp(sV)} \leq \exp(s^2(2\sigma^2)/2)$ for any $s \in \mathbb{R}.$ Fixing $s \in \mathbb{R}$ and expanding the term gives
\begin{align*}
\expectsub{X}{e^{sV}} &= \expectsub{X}{e^{s\expectsub{X'}{f(X)-f(X')}}} \\
&\leq \expectsub{X,X'}{e^{s(f(X)-f(X'))}} \tag{Jensen's inequality} \\
&= \frac{1}{2}\expectsub{X,X'}{e^{s(f(X)-f(X'))} + e^{s(f(X')-f(X))}} \\
&\overset{(a)}{\leq} \frac{1}{2}\expectsub{X,X'}{e^{s|X-X'|} + e^{-s|X-X'|}} \\
&\overset{(U=X-X')}{=} \frac{1}{2}\expectsub{U}{e^{sU} + e^{-sU}} = \expectsub{U}{e^{sU}}, \tag{By the symmetry of $U$}
\end{align*}
where inequality (a) uses the fact that $e^{a}+e^{-a} = 2\mathrm{cosh}(a) \leq 2\mathrm{cosh}(b) = e^b + e^{-b}$ for any $b \geq |a|.$ Since $-X'$ is sub-Gaussian with variance proxy $\sigma^2$ (Fact~\ref{fact:negative-subgaussian}) and $X,X'$ are independent, we have $U = X+(-X')$ sub-Gaussian with variance proxy $2\sigma^2$ by Fact~\ref{fact:addition-subGaussian}. As a result, $\expectsub{U}{e^{sU}} \leq \exp(s^2(2\sigma^2)/2)$ and $\expectsub{X}{e^{sV}} \leq \expectsub{U}{e^{sU}} \leq \exp(s^2(2\sigma^2)/2)$ for any $s \in \mathbb{R}$, showing that $f(X)$ is sub-Gaussian with variance proxy $2\sigma^2$.
\end{proof}

\begin{proof}[Proof of Lemma~\ref{lem:prob-good-event}]
Fix $k \in \set{K}, t \in [T]$. The number $n_k(t)$ of type-$k$ jobs in the dataset $\set{D}_t$ is upper bounded by $t$. Conditioning on $n_k(t) = n \leq t$, there are $n$ jobs whose costs are i.i.d. samples from $\set{F}_k.$ We first show that $c_k \in [\ubar{c}_k(t),\bar{c}_k(t)]$ with probability at least $1 - 2t^{-4}.$ Denote the observed costs by $X_1,\ldots,X_n.$ By definition, 
\[
\hat{c}_k(t) = \hat{\ell}_k^+(t) - \hat{\ell}_k^-(t)  = \frac{\sum_{i=1}^n X_i^+ - \sum_{i=1}^n X_i^-}{n} = \frac{\sum_{i=1}^n X_i}{n}.
\]
Since $\set{F}_k$ is sub-Gaussian with variance proxy $\sigma_{\max}^2$, applying Chernoff bound (Fact~\ref{fact:chernoff}) gives
\[
\Pr\left\{|\hat{c}_k(t) - c_k| > \sigma_{\max}\sqrt{\frac{8\ln t}{n}} \mid| n_k(t) = n \right \} \leq 2\exp\left(-4\ln t\right) = 2t^{-4}.
\]
Combining it with the definitions of $\ubar{c}_k(t)$ and $\bar{c}_k(t)$ in \eqref{eq:conf-h} then gives 
\begin{equation}\label{eq:prob-h}
\Pr\left\{c_k \in [\ubar{c}_k(t),\bar{c}_k(t)] \mid n_k(t) = n  \right \} \geq 1 - 2t^{-4}.
\end{equation}
We next show $\ell_k \leq \bar{\ell}_k(t)$ with probability at least $1-2t^{-4}$ conditioned on $n_k(t) = n$. By definition, $\hat{\ell}_k^+(t) = \frac{\sum_{i=1}^n X_i^+}{n}.$ In addition, since $\{X_i\}_{i \leq n}'$s are i.i.d. sub-Gaussian random variables with variance proxy $\sigma^2_{\max}$, we also have $\{X_i^+\}_{i \leq n}$'s are i.i.d. sub-Gaussian random variables with variance proxy $2\sigma^2_{\max}$ by Lemma~\ref{lem:sub-gaussian}. Applying Chernoff bound (Fact~\ref{fact:chernoff}) and the fact that $\ell_k^+ = \expect{X_i^+}$ gives
\[
\Pr\left\{|\ell_k^+ - \hat{\ell}_k^+(t)| > 4\sigma_{\max}\sqrt{\frac{\ln t}{n}} \mid n_k(t) = n \right \} \leq 2\exp\left(-4\ln t\right) = t^{-4}.
\]
By symmetry, we also have 
\[
\Pr\left\{|\ell_k^- - \hat{\ell}_k^-(t)| > 4\sigma_{\max}\sqrt{\frac{\ln t}{n}} \mid| n_k(t) = n \right \} \leq 2\exp\left(-4\ln t\right) = t^{-4}.
\]
By union bound, the above two inequalities imply that 
\[
\Pr\left\{\left|\min(\ell_k^+,\ell_k^-) - \min(\hat{\ell}_k^+(t),\hat{\ell}_k^-(t))\right| > 4\sigma_{\max}\sqrt{\frac{\ln t}{n}} \mid n_k(t) = n \right \} \leq 4\exp\left(-4\ln t\right) = 2t^{-4}.
\]
Using the definition that $\ell_k = \min(\ell_k^+,\ell_k^-)$ and that of $\bar{\ell}_k(t)$ in \eqref{eq:conf-r} shows \[\Pr\{\ell_k \leq \bar{\ell}_k(t) | n_k(t) = n \} \geq 1 - 2t^{-4}.\] Combining it with \eqref{eq:prob-h} gives $\Pr\{\set{E}_{k,t} | n_k(t) = n\} \geq 1 - 4t^{-4}$ by union bound. The lemma then follows by applying union bound over $n = 0,\ldots,t-1.$
\end{proof}

\subsection{Bounding number of admitted jobs to label-driven queue (Lemma~\ref{lem:bound-sum-ek})}\label{app:lem-bound-sum-ek}
\begin{proof}[Proof of Lemma~\ref{lem:bound-sum-ek}]
Abusing the notation, define $E_k(t) = \Lambda_k(t) E(t)$ which is equal to one if and only if $\kappa(t) = k$ and $E(t) = 1$. We decompose the number of admitted jobs $\expect{\sum_{t=1}^T E(t)}$ by
\[
\expect{\sum_{t=1}^T\sum_{k \in \set{K}} E_k(t)(\indic{\set{E}_{k,t} + \set{E}_{k,t}^c})} \leq \expect{\sum_{t=1}^T\sum_{k \in \set{K}} E_k(t)\indic{\set{E}_{k,t}}} + \sum_{t=1}^T \frac{4}{t^3} \leq \expect{\sum_{t=1}^T\sum_{k \in \set{K}} E_k(t)\indic{\set{E}_{k,t}}} + 5
\]
where the first inequality uses Lemma~\ref{lem:prob-good-event} and the fact that $\sum_{k\in\set{K}} E_k(t) \leq 1$, and the second inequality uses $\sum_{t=1}^{\infty} 1/t^3 \leq 1.2$. It remains to bound the first sum.

Fixing a type $k$, we consider the last time period $T_k$ such that $E_k(T_k)\indic{\set{E}_{k,T_k}} = 1$. By Line~\ref{line:bacidol-lda} in Algorithm~\ref{algo:bacidol}, $E_k(T_k) = 1$ if and only if $\ubar{c}_k(T_k) < -\gamma < \gamma < \bar{c}_k(T_k)$ and $Q^{\ld}(T_k) = 0$. In addition, the setting of the confidence interval, \eqref{eq:conf-h}, shows that $\bar{c}_k(T_k) - \ubar{c}_k(T_k) \leq 2\sigma_{\max}\sqrt{\frac{8\ln T_k}{n_k(T_k)}}$
where $n_k(T_k)$ is the number of type-$k$ jobs in the dataset $\set{D}(T_k)$. We next bound $n_k(T_k)$ for the case of $c_k \geq 0$. Since $\indic{\set{E}_{k,T_k}} = 1$, we have $c_k \in [\ubar{c}_k(T_k), \bar{c}_k(T_k)]$. By definition, 
\[
\ubar{c}_k(T_k) = \bar{c}_k(T_k) - (\bar{c}_k(T_k) - \ubar{c}_k(T_k)) \geq c_k - 2\sigma_{\max}\sqrt{\frac{8\ln T_k}{n_k(T_k)}} \geq \eta - 2\sigma_{\max}\sqrt{\frac{8\ln T_k}{n_k(T_k)}}
\]
where the last inequality is by the margin definition that $\eta \leq |c_k| = c_k$. Since $\ubar{c}_k(T_k) < -\gamma$, we have $-\gamma \geq \eta - 2\sigma_{\max}\sqrt{\frac{8\ln T_k}{n_k(T_k)}}$. As a result, $n_k(T_k) \leq \frac{32\sigma^2_{\max}\ln T}{\max(\eta,\gamma)^2}$. The same analysis applies for $c_k < 0$.

Recall that $E_k(t)=1$ implies that $Q^{\ld}(t) = 0$, i.e., all previous type-$k$ jobs added to the label-driven queue have received human reviews when we admit a new job to the label-driven queue. Hence, the number of samples collected for type $k$ is at least $n_k(T_k) \geq \sum_{t=1}^{T_k-1} E_k(t)$. Since $\eta, \gamma \leq 1$,
\[
\sum_{t=1}^T E_k(t)\indic{\set{E}_{k,t}} \leq \sum_{t=1}^{T_k-1} E_k(t) + 1 \leq n_k(T_k) + 1 \leq \frac{32\sigma_{\max}^2\ln T}{\max(\eta,\gamma)^2} + 1 \leq \frac{33\sigma_{\max}^2\ln T}{\max(\eta,\gamma)^2}.
\]
Summing across all types gives $\sum_{t=1}^T\sum_{k\in\set{K}} E_k(t)\indic{\set{E}_{k,t}} \leq \frac{33K\sigma^2_{\max}\ln T}{\max(\eta,\gamma)^2}$. Therefore,
\[
\expect{\sum_{t=1}^T\sum_{k \in \set{K}} E_k(t)(\indic{\set{E}_{k,t} + \set{E}_{k,t}^c})}  \expect{\sum_{t=1}^T\sum_{k \in \set{K}} E_k(t)\indic{\set{E}_{k,t}}} + 5 \leq  \frac{33K\sigma^2_{\max}\ln T}{\max(\eta,\gamma)^2} + 5 \leq  \frac{38K\sigma^2_{\max}\ln T}{\max(\eta,\gamma)^2}.
\]
\end{proof}

\subsection{Bounding per-period classification loss (Lemma~\ref{lem:bacidol-class-err})}\label{app:lem-bacidol-class-err}
\begin{proof}[Proof of Lemma~\ref{lem:bacidol-class-err}]
We fix a period $t$ and a type $k$ with $\Lambda_k(t) = 1$. If $\Lambda_k(t) = 0$, this directly implies that $Z_k(t) = 0$. We focus on bounding $Z_k(t)$ conditioning on $\set{E}_{k,t}$, which implies $c_k \in [\ubar{c}_k(t), \bar{c}_k(t)]$. We consider three cases:
\begin{itemize}
\item $\ubar{c}_k(t) \geq -\gamma$: We first assume that $\hat{c}_k(t) > 0$; in this case we set $Y(t) = -1$ and thus $Z_k(t) \leq \ell_k^- - \ell_k$. We proceed the proof by discussing whether $c_k \geq 0$:
\begin{itemize}
\item If $c_k \geq 0$, it holds that $Z_k(t) = 0$ since $\ell_k^+ - \ell_k^- = c_k \geq 0$ and thus $\ell_k = \min(\ell_k^+, \ell_k^-) = \ell_k^-$;
\item if $c_k < 0$, since the confidence interval is valid (we condition on $\set{E}_{k,t}$), $\ubar{c}_k(t) \leq c_k$ and thus $-\gamma \leq \ubar{c}_k(t) \leq c_k \leq 0$. The margin assumption ($|c_k| \geq \eta$) then implies $\gamma \geq \eta$. As a result,
\begin{equation}\label{eq:bound-z-k-positive}
Z_k(t) \leq |\ell_k^- - \ell_k^+|\indic{\gamma \geq \eta} = |-c_k|\indic{\gamma \geq \eta} \leq \gamma \indic{\gamma \geq \eta}.
\end{equation} 
\end{itemize}
We next consider the scenario when $\hat{c}_k(t) \leq 0$; in this scenario, we set $Y(t) = 1$. By definition of the confidence interval \eqref{eq:conf-h} and the case assumption that $\ubar{c}_k(t) \geq -\gamma$,
\[\bar{c}_k(t) \leq \hat{c}_k(t) + (\hat{c}_k(t) - \ubar{c}_k(t))  \leq 0 - \ubar{c}_k(t) \leq \gamma.\] 
 Since the confidence bound is valid conditioning on $\set{E}_{k,t}$, the ground truth $c_k$ satisfies $-\gamma \leq \ubar{c}_k(t) \leq c_k \leq \bar{c}_k(t) \leq \gamma$. \eqref{eq:bound-z-k-positive} again holds true and gives $Z_k(t) \leq \gamma\indic{\gamma \geq \eta}$. 
\item $\bar{c}_k(t) \leq \gamma$; This case is symmetric to the above case. Following the same analysis gives $Z_k(t) \leq \gamma \indic{\gamma \geq \eta}$ as well.
\item $\ubar{c}_k(t) \leq -\gamma \leq \gamma \leq \bar{c}_k(t)$. In this case, our label-driven admission wishes to send the job to the review job. If the label-driven queue is empty, it does so by setting $E(t) = 1$ and thus $Z_k(t) = 0$. Otherwise, the algorithm sets $Y(t) = -1$ and $E(t) = 0$. Therefore, for this case, $Z_k(t) \leq c_{\max} Q^{\ld}(t)$.
\end{itemize}
Summarizing the above three cases finishes the proof.
\end{proof}

\subsection{Bounding idiosyncrasy loss by Lagrangian via drift analysis (Lemma~\ref{lem:bacidol-idio-lag})}\label{app:lem-bacidol-idio-lag}
The proof is similar to that of Lemma~\ref{lem:connect-idio-lag}, and we include it for completeness.
\begin{proof}[Proof of Lemma~\ref{lem:bacidol-idio-lag}]
By definition, $L(1) = 0$ and $L(T+1) \geq \beta\sum_{t=1}^T \ell_{\kappa(t)}(1 - A(t) - E(t)).$ That is, the Idiosyncrasy Loss of  $\bacidol$  (as given in \eqref{eq:learning-loss-decompose}), is upper bounded by $\frac{1}{\beta}\expect{L(T+1) - L(1)}$. We thus prove the first inequality in the lemma.

To show the second inequality, using the fact that for any period $t$, $\ell_{\kappa(t)} (1-A(t)-E(t))$ is equal to $\sum_{k \in \set{K}} \ell_k\Lambda_k(t)(1-A_k(t)-E(t))$, we obtain
\begin{align*}
\expect{L(t+1) - L(t)} 
&=\beta\expect{\sum_{k \in \set{K}} \ell_k\Lambda_k(t)(1 - A_k(t) - E(t))} + \expect{\sum_{k \in \set{K}}\frac{1}{2}(Q_k(t+1)^2 - Q_k(t)^2)} \\
&\hspace{-1in}= \expect{\beta\sum_{k \in \set{K}} \ell_k \Lambda_k(t)(1 - A_k(t) - E(t)) + \frac{1}{2}\sum_{k \in \set{K}} (A_k(t) - S_k(t))^2 + \sum_{k \in \set{K}} Q_k(t)(A_k(t) - S_k(t))} \\
&\hspace{-1in}\leq 1 + \beta\sum_{k \in \set{K}} \ell_k\lambda_k(t) + \expect{-\beta\sum_{k \in \set{K}}\ell_k(A_k(t)+\Lambda_k(t)E(t)) + \sum_{k \in \set{K}} Q_k(t)(A_k(t) - S_k(t))} \\
&\hspace{-1in}\leq 1 + \beta\sum_{k \in \set{K}} \ell_k\lambda_k(t) + \expect{-\beta\sum_{k \in \set{K}}\ell_k\bar{A}_k(t) + \sum_{k \in \set{K}} Q_k(t)(\bar{A}_k(t) - S_k(t))},
\end{align*}
where the last inequality is because $A_k(t) \leq \bar{A}_k(t) \leq A_k(t) + \Lambda_k(t) E(t)$. We know that, conditioned on $\psi_k(t) 
 = 1$, the expectation of $S_k(t)$ is $\mu_kN(t)$. Therefore,
 \begin{align*}
 \expect{L(t+1) - L(t)} 
 &\leq 1 + \beta\sum_{k \in \set{K}} \ell_k\lambda_k(t) - \expect{\sum_{k \in \set{K}} \bar{A}_k(t)\left(\beta \ell_k-Q_k(t)\right) + \sum_{k \in \set{K}} Q_k(t)\psi_k(t)\mu_k N(t)} \\
 &= 1 + \expect{f_t(\bolds{\bar{A}(t)}, \bolds{\psi(t)},\bolds{Q}(t))}.
 \end{align*}
 Telescoping across $t=1,\ldots,T$, we get
\[
\expect{L(T+1) - L(1)} \leq T + \sum_{t=1}^T \expect{f_t(\bolds{\bar{A}(t)}, \bolds{\psi(t)},\bolds{Q}(t))}.
\]
\end{proof}

\subsection{Bounding the Regret in Admission for $\bacidol$ (Lemma~\ref{lem:bacidol-rega})}\label{app:lem-bacidol-rega}
Our proof structure is similar to that of bandits with knapsacks when restricting to the admission component \cite{agrawal2019bandits}. A difference in the settings is that works on bandits with knapsacks also need to learn the uncertain resource consumption whereas the consumption is known in our setting. However, our setting presents an additional challenge: job labels are delayed by an endogenous queueing effect. In particular, a key step in the proof is Lemma~\ref{lem:bacidol-admit-error}, which bounds the estimation error for all admitted jobs where the estimation error is proxied by $1/\sqrt{n_k(t)}$ and $n_k(t)$ is the number of reviewed type-$k$ jobs, and is a finite-type version of  \cite[Lemma~3]{AgrawalD16} (we later prove a similar result in the contextual setting; see Lemma~\ref{lem:conbacid-error}). To prove Lemma~\ref{lem:bacidol-admit-error}, we need to address the feedback delay, which adds the first term in our bound. The key insight is that, whenever we admit a new job of one type, the number of labels from that type we have not received is exactly the number of waiting jobs of that type. By the optimistic admission rule (Line~\ref{line:bacidol-admit} of Algorithm~\ref{algo:bacidol}), this number is upper bounded by $\beta c_{\max}$, which enables us to bound the effect of the delay.
\begin{lemma}\label{lem:bacidol-admit-error}
It holds that $\sum_{t=1}^T\bar{A}_k(t)\min\left(c_{\max},8\sigma_{\max}\sqrt{\frac{\ln t}{n_k(t)}}\right) \leq 4\beta c^2_{\max}+16\sigma_{\max}\sqrt{T\ln T}$.
\end{lemma}
\begin{proof}
Recall that $A_k(t)$ captures whether the algorithm admits a type-$k$ job into the review queue. Let $E_k(t)$ be $1$ if the platform admits a type-$k$ job into the label-driven queue (and $0$ otherwise). By Line~\ref{line:bacidol-admit} of Algorithm~\ref{algo:bacidol}, we know that $A_k(t) = \bar{A}_k(t)(1-E_k(t))$ and thus $\bar{A}_k(t) \leq A_k(t)+E_k(t)$. We thus aim to upper bound $\sum_{t=1}^T(A_k(t)+E_k(t))\min\left(c_{\max},8\sigma_{\max}\sqrt{\frac{\ln t}{n_k(t)}}\right)$. Denoting the sequence of periods with $A_k(t)+E_k(t) = 1$ as $t_1,\ldots,t_Y$ where $Y = \sum_{t=1}^T \left(A_k(t) + E_k(t)\right)$ and setting $n'(y) = n_k(t_y)$, we then need to upper bound $\sum_{y=1}^Y \min\left(c_{\max}, 8\sigma_{\max}\sqrt{\frac{\ln T}{n'(y)}}\right).$

We next connect $y$ and $n'(y)$. We fix a period $t$. Since $n_k(t)$ is the number of type-$k$ jobs in the dataset $\set{D}(t)$, we have $n_k(t)$ equal to the number of periods that a type-$k$ job is successfully served. In addition, the number of type-$k$ jobs in the label-driven queue $\set{Q}^{\ld}(t)$  and the review queue $\set{Q}(t)$ is given by the number of admitted type-$k$ jobs deducted by the number of reviewed type-$k$ jobs. As a result, $Q^{\ld}(t) + Q_k(t) \geq \sum_{\tau=1}^{t-1} (A_k(\tau) + E_k(\tau)) - n_k(t)$. As in the proof of Lemma~\ref{lem:bacidol-delay}, we have $Q^{\ld}(t) + Q_k(t) \leq 1 + 2\beta c_{\max} \leq 3\beta c_{\max}$ and thus $n_k(t) \geq \sum_{\tau=1}^{t-1} (A_k(\tau) + E_k(\tau)) - 3\beta c_{\max}$. For any $y$, we consider $t = t_y$ in the above analysis. Then we know that $n'(y) \geq y - 1 - 3\beta c_{\max} \geq y - 4\beta c_{\max}$. Therefore, recalling that we view $a / 0 = +\infty$ for any positive $a$, 
\begin{align*}
\sum_{t=1}^T\bar{A}_k(t)\min\left(c_{\max},8\sigma_{\max}\sqrt{\frac{\ln t}{n_k(t)}}\right) &\leq \sum_{y=1}^Y \min\left(c_{\max}, 8\sigma_{\max}\sqrt{\frac{\ln T}{n'(y)}}\right) \\
&\leq \sum_{y=1}^Y \min\left(c_{\max}, 8\sigma_{\max}\sqrt{\frac{\ln T}{\max(0, y - 4\beta c_{\max})}}\right) \\
&\hspace{-2in}\overset{u = y - 4\beta c_{\max}}{\leq} \sum_{y=1}^{4\beta c_{\max}} c_{\max} + \sum_{u=1}^T  8\sigma_{\max}\sqrt{\frac{\ln T}{u}} \leq 4\beta c^2_{\max}+16\sigma_{\max}\sqrt{T\ln T},
\end{align*}
where we use the fact that $\sum_{u=1}^T \sqrt{1/u} \leq 1 + \int_1^T x^{-1/2}\mathrm{d}x \leq 2\sqrt{T}$. 
\end{proof}
We now upper bound the regret in admission.
\begin{proof}[Proof of Lemma~\ref{lem:bacidol-rega}]
Recall the good event $\set{E}_{k,t}$ where $c_k \in [\ubar{c}_k(t),\bar{c}_k(t)]$. We first decompose $\textsc{RegA}$ by
\begin{align}
\textsc{RegA}(T) &= \expect{\sum_{t=1}^T \sum_{k\in\set{K}} (A_k^{\bacid}(t) - \bar{A}_k(t))(\beta \ell_k-Q_k(t))\left(\indic{\set{E}_{k,t}} + \indic{\set{E}_{k,t}^c}\right)} \nonumber\\
&\leq\expect{\sum_{t=1}^T \sum_{k\in\set{K}} (A_k^{\bacid}(t) - \bar{A}_k(t))(\beta \ell_k-Q_k(t))\indic{\set{E}_{k,t}}} + \sum_{t=1}^T \sum_{k \in \set{K}} (\beta \ell_k + Q_k(t))\Pr\{\set{E}_{k,t}^c\} \nonumber\\
&\leq \expect{\sum_{t=1}^T \sum_{k\in\set{K}} (A_k^{\bacid}(t) - \bar{A}_k(t))(\beta \ell_k-Q_k(t))\indic{\set{E}_{k,t}}} + 8\beta K c_{\max}+8K \label{eq:bacidol-regA-decomp}
\end{align}
where the first inequality uses $A_k^{\bacid}(t) - \bar{A}_k(t) \in \{-1,+1\}$; the second inequality uses $Q_k(t) \leq t$, Lemma~\ref{lem:prob-good-event} that $\Pr\{\set{E}_{k,t}\} \geq 1 - 4t^{-3}$, and that $\sum_{t=1}^T t^{-2} \leq 2$. 

We fix a type $k$ and upper bound $\sum_{t=1}^T (A_k^{\bacid}(t) - \bar{A}_k(t))(\beta \ell_k-Q_k(t))\indic{\set{E}_{k,t}}$. Fixing a period $t$ and assuming $\set{E}_{k,t}$ holds, we have $\ell_k \leq \bar{\ell}_k(t)$ and 
$
\bar{\ell}_k(t) - \ell_k \leq 8\sigma_{\max}\sqrt{\frac{\ln t}{n_k(t)}}$
by the definition of $\bar{\ell}_k(t)$ in \eqref{eq:conf-r}. In addition, if $\Lambda_k(t) = 0$, we have $\bar{A}_k(t) = A_k^{\bacid}(t) = 0$; otherwise,
\[\bar{A}_k(t) = \indic{\beta \bar{\ell}_k(t) \geq Q_k(t)} \geq \indic{\beta \ell_k \geq Q_k(t)} \geq A_k^{\bacid}(t).\]
Therefore, $(A_k^{\bacid}(t) - \bar{A}_k(t))(\beta \ell_k-Q_k(t))$ is positive if and only if $A_k^{\bacid}(t) = 0$ and $\bar{A}_k(t) = 1$. In this case,
\begin{align*}
(A_k^{\bacid}(t) - \bar{A}_k(t))(\beta \ell_k-Q_k(t)) = Q_k(t) - \beta \ell_k &= Q_k(t) - \beta \bar{\ell}_k(t) + \beta(\bar{\ell}_k(t) - \ell_k) \\
&\hspace{-1in}\leq \beta(\bar{\ell}_k(t) - \ell_k) \leq \beta\min\left(c_{\max},8\sigma_{\max}\sqrt{\frac{\ln t}{n_k(t)}}\right)
\end{align*}
where $Q_k(t) \leq \beta \bar{\ell}_k(t)$ when $\bar{A}_k(t) = 1$; the last inequality uses the definition of $\bar{\ell}_k(t)$ in \eqref{eq:conf-r}. Hence,
\begin{align*}
\sum_{t=1}^T (A_k^{\bacid}(t) - \bar{A}_k(t))(\beta \ell_k-Q_k(t))\indic{\set{E}_{k,t}} &\leq \beta\sum_{t=1}^T\bar{A}_k(t)\min\left(c_{\max},8\sigma_{\max}\sqrt{\frac{\ln t}{n_k(t)}}\right) \\
&\hspace{-1in}\leq 4\beta^2c^2_{\max}+16\beta \sigma_{\max}\sqrt{T\ln T} \qquad \text{by Lemma~\ref{lem:bacidol-admit-error}}.
\end{align*}
Using \eqref{eq:bacidol-regA-decomp} gives
\begin{align*}
\textsc{RegA}(T) &\leq \expect{\sum_{t=1}^T \sum_{k\in\set{K}} (A_k^{\bacid}(t) - \bar{A}_k(t))(\beta \ell_k-Q_k(t))\indic{\set{E}_{k,t}}} + 8\beta K c_{\max}+8K \\
&\leq 4K\beta^2c^2_{\max}+16K\beta \sigma_{\max}\sqrt{T\ln T} + 8\beta K c_{\max}+8K \\
&\leq 20K\beta^2c^2_{\max}+16K\beta \sigma_{\max}\sqrt{T\ln T},
\end{align*} 
where we use the assumption that $\beta \geq 1/c_{\max}$ and $c_{\max} \geq 1$.
\end{proof}

\subsection{Bounding the Regret in Scheduling for $\bacidol$ (Lemma~\ref{lem:bacidol-regs})}\label{app:lem-bacidol-regs}
\begin{proof}[Proof of Lemma~\ref{lem:bacidol-regs}]
By the scheduling rule in $\bacidol$, if $Q^{\ld}(t) = 0$ for a period $t$, we have $\psi_k(t) = \indic{k = \arg\max_{k' \in \set{K}} Q_{k'}(t)\mu_{k'}} = \psi_k^{\bacid}(t)$. If $Q^{\ld}(t) = 1$, then $\psi_k(t) = 0$ for any $k$ because of forced scheduling, which schedules the job in the label-driven queue when there is any (recall that $Q_k(t)$ refers to the \emph{non-}label-driven queues). Therefore, 
\begin{align*}
\textsc{RegS}(T) &= \expect{\sum_{t=1}^T \sum_{k \in \set{K}} \psi_k^{\bacid}(t)Q_k(t)\mu_kN(t)Q^{\ld}(t)} \\
&\leq 2\expect{\sum_{t=1}^T \sum_{k \in \set{K}} \psi_k^{\bacid}(t)\beta c_{\max}Q^{\ld}(t)} \\
&\leq 2\beta c_{\max}\expect{\sum_{t=1}^T Q^{\ld}(t)} \leq \frac{76\beta c_{\max}K\sigma^2_{\max}\ln T}{\max(\eta,\gamma)^2\hat{\mu}_{\min}}
\end{align*}
where the first inequality is similar to the proof of Lemma~\ref{lem:bacidol-delay} and the assumption that $N(t)\mu_k \leq 1$; the second inequality is by $\sum_{k\in \set{K}} \psi_k^{\bacid}(t) \leq 1$; the last inequality is by Lemma~\ref{lem:bound-sum-ek}.
\end{proof}
\section{Supplementary materials on \textsc{COLBACID} (Section~\ref{sec:contextual})}\label{app:contextual}
\subsection{Comparison with Literature of Linear Contextual Bandits}\label{app:comparison}
Contrasting to  linear contextual bandits, e.g. \cite{Abbasi-YadkoriPS11}, and linear contextual bandits with delays, such as \cite{VernadeCP17,blanchet2023delay}, there are four main distinctions.

The first distinction is that our work needs to estimate the cost of accepting a job ($\ell_k^+ = \expectsub{k}{C_t^+}$) and the cost of rejecting a job ($\ell_k^- = \expectsub{k}{-C_t^-}$) while papers in multi-armed bandits only need to estimate the mean reward of an arm. The reason is that we need to identify the type with high error of classification by AI \textemdash ~in multi-armed bandits, one only cares about types with high average cost (or reward). Indeed, since $\ell_k^+ - \ell_k^- = c_k$, our setting requires to learn more information of the underlying cost distributions. This requirement adds challenge to our concentration bound analysis since it is apriori unclear if $C_t^+$ (or $C_t^-$) is still sub-Gaussian when the cost $C_t$ is sub-Gaussian. Fortunately, our result in Lemma~\ref{lem:sub-gaussian} affirmatively answers this question and shows that a Lipschitz function of a sub-Gaussian random variable is still sub-Gaussian but with twice variance proxy.

The second distinction is that those works use a different estimator $\hat{\btheta}$ by including all arrived data points into the dataset (unobserved labels are set as $0$). In contract, our confidence set is constructed based on all \emph{reviewed} jobs in \eqref{eq:regression}, which is a subset of all \emph{arrived} jobs \emph{endogenously} determined by our algorithm. 

The third distinction, as a result of the second, is that conditioned on the set of arrived jobs, the matrix $\bar{V}_t$ in \eqref{eq:regression} is a fixed matrix in their settings, but is a random matrix in our setting because it is constructed from the set of reviewed jobs. We should thus deal with the intricacy of the randomness in the regression.

The fourth distinction lies in the analysis with queueing-delayed feedback.  The analysis of prior work crucially relies on an independence assumption, i.e., the event of observing feedback for a job within a particular delay period is independent from other jobs, which enables a concentration bound on the number of observed feedback.  In our setting, this is no longer the case as the delay in one job implies that other jobs wait in the queue and thus delays could correlate with each other. We thus resort to the more intuitive estimator and analyze it via properties of our queueing systems. 

We note that \cite{blanchet2023delay} consider a setting where delays are correlated because they are generated from a Markov chain. That said, the proof requires concentration and stationary properties of the Markov chain, which are not available in our non-Markovian setting.

\subsection{Regret guarantee of \textsc{COLBACID} (Theorem~\ref{thm:cbacidol})}\label{app:thm-cbacidol}
\begin{proof}[Proof of Theorem~\ref{thm:cbacidol}]
By definition of $B_{\delta}(T)$ in \eqref{eq:contextual-def-conf}, it holds that $B_{\delta}(T) \lesssim \sqrt{d\ln (T / \delta)} \lesssim \sqrt{d\ln T}$ where the last inequality is because $\delta^{-1} \lesssim T$ by its definition in Algorithm~\ref{algo:cbacidol}. Applying Lemmas~\ref{lem:contextual-delay},~\ref{lem:contextual-class} and \ref{lem:contextual-idio} to \eqref{eq:contextual-loss-decompose}, we have the loss of $\conbacid$ upper bounded by
\begin{align*}
\expect{\set{L}^{\conbacid}(T)} &= \textsc{Classification Loss} + \textsc{Idiosyncrasy Loss} + \textsc{Relaxed Delay Loss} \\
&\hspace{-1.2in}\lesssim \set{L}^\star(T) + \Delta(\set{K}_{\set{G}})T + \gamma\indic{\gamma \geq \eta}T + \frac{d^{2.5}\ln^{2} T}{\max(\eta,\gamma)^2}+\frac{T}{\beta} + Gd^{1.5}\beta\sqrt{\ln T} + d\sqrt{GT\ln T} \\
&\hspace{-1.2in}\lesssim \set{L}^\star(T)+ \Delta(\set{K}_{\set{G}})T+\min\left(\frac{d^{2.5}\ln^{2} T}{\eta^2}, d^{5/6}T^{2/3}\ln(T)^{2/3}\right)+ d\sqrt{GT\ln T},
\end{align*}
where we use the setting that $\beta = \sqrt{T / (Gd^{1.5})}, \gamma = (T / (d^{2.5}\ln^2 T))^{-1/3}.$ We finish the proof by
\begin{align*}
\textsc{Reg}^{\conbacid}(T) &= \set{L}^{\conbacid}(T) - \set{L}^\star(T) \\
&\hspace{-0.4in}\lesssim \Delta(\set{K}_{\set{G}})T+\min\left(\frac{d^{2.5}\ln^{2} T}{\eta^2}, d^{5/6}T^{2/3}\ln(T)^{2/3}\right)+ d\sqrt{GT\ln T}.
\end{align*}
\end{proof}

\subsection{Confidence Sets are Valid with High Probability (Lemma~\ref{lem:contextual-prob-good-event})}\label{app:lem-contextual-prob-good-event}
\begin{proof}[Proof of Lemma~\ref{lem:contextual-prob-good-event}]
The proof uses \cite[Theorem~2]{Abbasi-YadkoriPS11}, which we restate in Fact~\ref{fact:abbasi-bound}. Recall that $\set{C}^+_t$ is the set of vectors $\btheta^+$ and $\set{C}^-_t$ is the set of vectors $\btheta^-$ in the confidence sets as defined in \eqref{eq:contextual-def-conf}. The proof proceeds by showing that $\btheta^{\star,+} \in \set{C}^+_t$ and $\btheta^{\star,-} \in \set{C}^-_t$ with probability $1 - 2\delta$, which implies $\bTheta^\star \in \set{C}_t^+ \times \set{C}_t^-$ with probability $1-2\delta$.

We first show that $\btheta^{\star,+} \in \set{C}^+_t$ with probability $1-\delta$. Recall that the datapoint for $\hat{\btheta}^+$ in period~$t$ is $C^+_{\set{S}(t)}$. Define $\varepsilon_{\set{S}(t)} = C^+_{\set{S}(t)} - \expectsub{\kappa(\set{S}(t))}{C^+_{\set{S}(t)}} = C^+_{\set{S}(t)} - \bphi_{\kappa(\set{S}(t))}^{\trans}\btheta^{\star,+}$. To apply Fact~\ref{fact:abbasi-bound}, it suffices to show that the sequence $\{\varepsilon_{\set{S}(t)}\}$ is conditionally sub-Gaussian with variance proxy $2\sigma^2_{\max}$ for the filtration $\set{F}_{t-1} = \sigma(\bolds{X}_1,\ldots,\bolds{X}_t,\varepsilon_{\set{S(1)},\ldots,\varepsilon_{\set{S}(t - 1)})}$, i.e., $\expect{e^{u\varepsilon_{\set{S}(t)}} \mid \set{F}_{t-1}} \leq \exp(u^2(2\sigma_{\max}^2)/2)$ for any $u \in \mathbb{R}$. Given that we assume that job costs are independent samples, it holds that $\expect{e^{u\varepsilon_{\set{S}(t)}} \mid \set{F}_{t-1}} = \expect{e^{u\varepsilon_{\set{S}(\tau)}}}$. Since we assume $C_{\set{S}(t)}$ is sub-Gaussian with variance proxy $\sigma_{\max}^2$ conditioned on the type of the job $\kappa(\set{S}(t)) = k$ and $\max(x,0)$ is a $1-$Lipschitz function, Lemma~\ref{lem:sub-gaussian} shows that $C_{\set{S}(t)}^+$ is sub-Gaussian with variance proxy $2\sigma_{\max}^2$. As a result, $\varepsilon_{\set{S}(t)}$ is also sub-Gaussian with variance proxy $2\sigma_{\max}^2$ and $\expect{e^{u\varepsilon_{\set{S}(t)}}} \leq \exp(u^2(2\sigma_{\max}^2) / 2)$, showing that thus $\{\varepsilon_{\set{S}(t)}\}$ is conditionally sub-Gaussian with variance proxy $2\sigma_{\max}^2$. Hence the conditions of Fact~\ref{fact:abbasi-bound} are met. Applying Fact~\ref{fact:abbasi-bound} and the definition of $\set{C}^{+}_t$  then gives $\btheta^{\star,+} \in \set{C}_t^+$ with probability at least $1-\delta$.

By symmetry, the above argument applies to $\btheta^{\star,-}$ and thus $\btheta^{\star,-} \in \set{C}_t^-$ with probability at least $1 - \delta$. Using union bound gives that $\bTheta^\star \in \set{C}_t^+ \times \set{C}_t^-$ with probability at least $1-2\delta$.
\end{proof}

\subsection{Bounding the Number of Label-Driven Admissions  (Lemma~\ref{lem:bound-norm-selective})}\label{app:lem-bound-norm-selective}
We define $\set{T}_E$ as the set of periods where a new job is admitted into the label-driven queue, i.e., $\set{T}_E = \{t\leq T\colon E(t) = 1\}$, and $\expect{\sum_{t=1}^T E(t)} = \expect{|\set{T}_E|}$. The following lemma connects $\expect{|\set{T}_E|}$ with the estimation error $\|\bphi_{\kappa(t)}\|_{\bar{\bolds{V}}_{t-1}^{-1}}^2$ for jobs admitted into the label-driven queue.
\begin{lemma}\label{lem:conntect-te-norm}
For $\delta \in (0,1)$, it holds that $\expect{|\set{T}_E|} \leq \frac{16B_{\delta}^2(T)\expect{\sum_{t \in \set{T}_E} \|\bphi_{\kappa(t)}\|_{\bar{\bolds{V}}_{t-1}^{-1}}^2}}{\max(\eta,\gamma)^2} + 2T\delta$.
\end{lemma}
\begin{proof}[Proof of Lemma~\ref{lem:conntect-te-norm}]
By the law of total expectation, the fact that $|\set{T}_E| \leq T$, and Lemma~\ref{lem:contextual-prob-good-event} that $\Pr\{\set{E}^c\} \leq 2\delta$, it holds that 
\begin{equation}\label{eq:te-condition-expect}
    \expect{|\set{T}_E|} = \expect{|\set{T}_E| \mid \set{E}}\Pr\{\set{E}\} + \expect{|\set{T}_E \mid \set{E}^c|}\Pr\{\set{E}^c\} \leq \expect{|\set{T}_E| \mid \set{E}} + 2T\delta.
\end{equation}
It remains to upper bound $\expect{|\set{T}_E| \mid \set{E}}$. We condition on $\set{E}$, which implies that $\bTheta^\star \in \set{C}_t^+ \times \set{C}_t^-$ for any $t \geq 0$. Focusing on jobs that are admitted in the label-driven queue $t \in \set{T}_E$ and letting $k = \kappa(t)$ be the type of the job arrived in period $t$, it holds that $E(t) = 1$ and that $\ubar{c}_k(t) < -\gamma< \gamma < \bar{c}_k(t)$ by the label-driven admission rule (Line~\ref{line:conbacid-lda} of Algorithm~\ref{algo:cbacidol}). In addition, since $\bTheta^\star \in \set{C}_t^+ \times \set{C}_t^-$, it holds that $\ubar{c}_k(t) \leq c_k \leq \bar{c}_k(t)$. By the margin assumption $|c_k(t)| \geq \eta$, we must have $\bar{c}_k(t) - \ubar{c}_k(t) \geq \eta + \gamma \geq \max(\eta,\gamma)$. To see this, if $c_k \geq 0$, then $\bar{c}_k(t) \geq \eta$ and $\bar{c}_k(t) - \ubar{c}_k(t) \geq \eta - (-\gamma) = \eta + \gamma$; same analysis holds for $c_k < 0$. In addition, recall $\hat{\bolds{c}}(t-1) = \hat{\btheta}^+(t-1) - \hat{\btheta}^-(t-1)$. For any $\bolds{c} \in \set{C}_{t-1}$, there exists $\btheta^+ \in \set{C}^+_{t-1}$ and $\btheta^- \in \set{C}^-_{t-1}$ such that $\bolds{c} = \btheta^+ - \btheta^-$. Furthermore,
\begin{align*}
\left|\bphi_k^{\trans}\bolds{c} - \bphi_k^{\trans}\hat{\bolds{c}}_{t-1}\right| = \left|\bphi_k^{\trans}(\bolds{c} - \hat{\bolds{c}}_{t-1})\right| &\leq \|\bolds{c}-\hat{\bolds{c}}_{t-1}\|_{\bar{\bolds{V}}_{t-1}}\|\bphi_k\|_{\bar{\bolds{V}}_{t-1}^{-1}} \\
&= \|\btheta^+ - \btheta^- - (\hat{\btheta}^+(t-1) - \hat{\btheta}^-(t-1))\|_{\bar{\bolds{V}}_{t-1}}\|\bphi_k\|_{\bar{\bolds{V}}_{t-1}^{-1}} \\
&\hspace{-1in}\leq \left(\|\btheta^+- \hat{\btheta}^+(t-1)\|_{\bar{\bolds{V}}_{t-1}} + \|\btheta^- - \hat{\btheta}^-(t-1)\|_{\bar{\bolds{V}}_{t-1}}\right)\|\bphi_k\|_{\bar{\bolds{V}}_{t-1}^{-1}} \\
&\leq 2B_{\delta}(t-1)\|\bphi_k\|_{\bar{\bolds{V}}_{t-1}^{-1}}
\end{align*}
where the first inequality is by Cauchy Inequality; the second inequality is by triangle inequality; and the last inequality is by the definition of $\set{C}^+_{t-1}, \set{C}^-_{t-1}$ in \eqref{eq:contextual-def-conf}.
Combining this result with the definition of confidence intervals in \eqref{eq:conf-h-feature}, we have
\[
\max(\eta,\gamma)\leq \bar{c}_k(t) - \ubar{c}_k(t)
\leq \max_{\bolds{c} \in \set{C}_{t-1}} \bphi_k^{\trans}\bolds{c} - \min_{\bolds{c} \in \set{C}_{t-1}} \bphi_k^{\trans}\bolds{c} \leq 4B_{\delta}(t-1)\|\bphi_k\|_{\bar{\bolds{V}}_{t-1}^{-1}},
\]
so $16B_{\delta}^2(T)\|\bphi_k\|_{\bar{\bolds{V}}_{t-1}^{-1}}^2 \geq \max(\eta,\gamma)^2$ and  $16B_{\delta}^2(T)\sum_{t \in |\set{T}_E|} \|\bphi_k\|_{\bar{\bolds{V}}_{t-1}^{-1}}^2 \geq |\set{T}_E|\max(\eta,\gamma)^2$ conditioned on~$\set{E}$. This shows that $\expect{|\set{T}_E| \mid \set{E}} \leq \frac{16B_{\delta}^2(T)\sum_{t \in |\set{T}_E|} \|\bphi_k\|_{\bar{\bolds{V}}_{t-1}^{-1}}^2}{\max(\eta,\gamma)^2}$, which finishes the proof by \eqref{eq:te-condition-expect}.
\end{proof}
To upper bound the estimation error, the following result shows that the estimation error for a feature vector is larger when using a subset of data points.
\begin{lemma}\label{lem:norm-subset}
Given $\xi > 0$ and two subsets $\set{T}_1,\set{T}_2$ of time periods such that $\set{T}_1 \subseteq \set{T}_2$, we define $\bolds{V}_1 = \xi \bI + \sum_{t \in \set{T}_1} \bphi_{\kappa(\set{S}(t))}\bphi_{\kappa(\set{S}(t))}^{\trans}$ and $\bolds{V}_2 = \xi \bI + \sum_{t \in \set{T}_2} \bphi_{\kappa(\set{S}(t))}\bphi_{\kappa(\set{S}(t))}^{\trans}$ coming from the data points in periods $\set{T}_1,\set{T}_2$ respectively. Then for any vector $\bolds{u} \in \mathbb{R}^d$, it holds that $\|\bolds{u}\|_{\bolds{V}_1^{-1}} \geq \|\bolds{u}\|_{\bolds{V}_2^{-1}}$.
\end{lemma}
\begin{proof}
Let $\bV' = \bV_2 - \bV_1$. We have $\bV' = \sum_{t \in \set{T}_2 \setminus \set{T}_1} \|\bphi_{\kappa(\set{S}(t))}\|_{\bar{\bolds{V}}_{t-1}^{-1}}^2$, which is positive-semi definite (PSD). Also, $\bV_1^{-1/2}$ exists because $\bV_1$ is PSD. As a result, for any $\bolds{u} \in \mathbb{R}^d$, denoting the minimum eigenvalue of a matrix $\bolds{A}$ as $\lambda_{\min}(\bolds{A})$, we have 
\begin{align*}
\|\bolds{u}\|_{\bV_2^{-1}}^2 = \bolds{u}^{\trans}\bV_2^{-1}\bolds{u} = \bolds{u}^{\trans}(\bV_1+\bV')^{-1}\bolds{u}&=\bolds{u}^{\trans}\bV_1^{-1/2}\left(\bI + \bV_1^{1/2}\bV'\bV_1^{1/2}\right)^{-1}\bV_1^{-1/2}\bolds{u} \\
&\hspace{-1in}\leq \frac{\|\bV_1^{-1/2}\bolds{u}\|_2^2}{\lambda_{\min}\left(\bI + \bV_1^{1/2}\bV'\bV_1^{1/2}\right)} \leq \|\bV_1^{-1/2}\bolds{u}\|_2^2 = \|\bolds{u}\|_{\bV_1^{-1}}^2
\end{align*}
where the first inequality is by the Courant-Fischer theorem and the fact that eigenvalues of a matrix inverse are inverses of the matrix; the second inequality is because $\bolds{V}_1,\bolds{V}'$ are PSD and thus $\lambda_{\min}\left(\bI + \bV_1^{1/2}\bV'\bV_1^{1/2}\right) \geq 1.$ 
\end{proof}
We next bound $\expect{|\set{T}_E|}$ by bounding $\sum_{t \in \set{T}_E} \|\bphi_{\kappa(t)}\|_{\bar{\bolds{V}}_{t-1}^{-1}}^2$.
\begin{proof}[Proof of Lemma~\ref{lem:bound-norm-selective}]
We denote elements in the set $\set{T}_E$ by $1 \leq t_1 < \ldots < t_M$ with $M = |\set{T}_E|$, and define $\widetilde{\bolds{V}}_m = \xi \bI + \sum_{i \leq m} \bphi_{\kappa(t_i)}\bphi^{\trans}_{\kappa(t_i)}$ for any $m \leq M$, which resembles the definition of $\bar{\bolds{V}}_t$ in \eqref{eq:regression} but is restricted to data collected from jobs sent to the label-driven queue. We fix $m \leq M$ and consider the norm $\|\bphi_{\kappa(t_m)}\|_{\bar{\bolds{V}}_{t_m-1}^{-1}}$. Since the label-driven admission rule (Line~\ref{line:conbacid-lda} in Algorithm~\ref{algo:cbacidol}) only admits a job to the label-driven queue when it is empty, it holds that in period $t_m$, jobs that arrived in periods $\{t_1,\ldots,t_{m-1}\}$ are already reviewed as they were admitted into the label-driven queue and the queue is now empty. As a result, jobs that arrived in periods $\{t_1,\ldots,t_{m-1}\}$ are in the dataset $\set{D}_{t_{m} - 1}$. By Lemma~\ref{lem:norm-subset}, it holds that $\|\bphi_{\kappa(t_m)}\|_{\bar{\bolds{V}}_{t_m-1}^{-1}}^2 \leq \|\bphi_{\kappa(t_m)}\|_{\widetilde{\bolds{V}}_{m-1}^{-1}}^2$ and thus
\[
\sum_{t \in \set{T}_E} \|\bphi_{\kappa(t)}\|_{\bar{\bolds{V}}_{t-1}^{-1}}^2 \leq \sum_{m=1}^M \|\bphi_{\kappa(t_m)}\|_{\widetilde{\bolds{V}}_{m-1}^{-1}}^2. 
\]
We can thus bound $\sum_{t \in \set{T}_E} \|\bphi_{\kappa(t)}\|_{\bar{\bolds{V}}_{t-1}^{-1}}^2$ by bounding
\[
\sum_{m=1}^M \|\bphi_{\kappa(t_m)}\|_{\widetilde{\bolds{V}}_{m-1}^{-1}}^2 \leq 2\ln \frac{\mathrm{det}(\widetilde{\bolds{V}}_M)}{\mathrm{det}(\xi \bI)} \leq 2\ln\frac{(\xi + MU^2/d)^d}{\xi^d} = 2d\ln(1+MU^2/(d\xi))\leq 2d\ln(1+T/d),
\]
where the first inequality is by Fact~\ref{fact:ellipsoid-bound} (with $\xi \geq \max(1,U^2)$); the second inequality is by Fact~\ref{fact:matrix-norm-bound}; the last inequality is by $\xi \geq U^2, M \leq T$. 

Applying Lemma~\ref{lem:conntect-te-norm} gives $\expect{|\set{T}_E|} \leq \frac{32B_{\delta}^2(T)d\ln(1+T/d)}{\max(\eta,\gamma)^2} + 2T\delta \leq \frac{34B_{\delta}^2(T)d\ln(1+T/d)}{\max(\eta,\gamma)^2}$ where the last inequality is because $\max(\eta,\gamma) \leq 1$, $T\delta \leq 1 \leq B_{\delta}^2(T)d\ln(1+T/d)$ as $B_{\delta}^2(T) \geq \xi U^2 \geq 1$ (we assume $U \geq 1$) and $d\ln(1+T/d) \geq d\ln(1+3/d) \geq 1$ by $T \geq 3$ and Fact~\ref{fact:ln-property}.
\end{proof}

\subsection{Bounding Per-Period Classification Loss (Lemma~\ref{lem:con-period-class})}\label{app:lem-con-period-class}
The proof is similar to that of Lemma~\ref{lem:bacidol-class-err} and we provide it here for completeness.
\begin{proof}[Proof of Lemma~\ref{lem:con-period-class}]
 Let us condition on $\set{E}$ (the event that the confidence set is correct) and bound 
 \[
 Z_k(t) = \Lambda_k(t)(Y(t)^+(\ell_k^+ - \ell_k) + Y(t)^-(\ell_k^- -\ell_k))(1-A(t)-E(t)).
 \]
 We fix a period $t$ and a type $k$ such that a type-$k$ job arrives in period $t$ and it is not admitted into the label-driven queue, i.e., $\Lambda_k(t) = 1$ and $E(t) = 0$; otherwise $Z_k(t) = 0$. Since $\set{E}$ holds, we have $\bolds{c}^\star \in \set{C}_{t-1}$ and thus $c_k \in [\ubar{c}_k(t),\bar{c}_k(t)]$ for any period $t$. We consider the three possible cases when the algorithm makes the classification decision $Y(t)$:
\begin{itemize}
\item $\ubar{c}_k(t) \geq -\gamma$: We first assume that $\bphi_k^{\trans}\hat{\bolds{c}}(t-1) > 0$ where $\hat{\bolds{c}}(t-1) = \hat{\btheta}^+(t-1) - \hat{\btheta}^-(t-1)$. In this case we classify the job by $Y(t) = \-1$ (Line~\ref{line:conbacid-classify} in Algorithm~\ref{algo:cbacidol}), so $Z_k(t) \leq (\ell_k^- - \ell_k)$. 
\begin{itemize}
    \item Then, if $c_k > 0$, it holds that $Z_k(t) = 0$ since $\ell_k^- = \ell_k^+ - c_k \leq \ell_k^+$ and thus $\ell_k = \ell_k^-$.
    \item If $c_k \leq 0$, it holds that $-\gamma \leq \ubar{c}_k(t) \leq c_k \leq 0$ and thus $|c_k|\leq \gamma$. Given the margin assumption ($|c_k|\geq \eta$), this implies that $\gamma\geq \eta$ and thus $Z_k(t) \leq \gamma \indic{\gamma \geq \eta}$ as in~\eqref{eq:bound-z-k-positive}. 
\end{itemize}
We next consider the scenario where  $\bphi_k^{\trans}\hat{\bolds{c}}(t-1)\leq 0$. 
\begin{itemize}
    \item If $\bar{c}_k(t) \leq \gamma$ then $c_k \in [\ubar{c}_k(t),\bar{c}_k(t)] \subseteq [-\gamma,\gamma]$. Given the margin assumption  ($|c_k|\geq \eta$), this again implies that $\gamma\geq \eta$ and thus  $Z_k(t) \leq \gamma\indic{\gamma \geq \eta}$.
    \item If $\bar{c}_k(t) > \gamma$, then there exists $\bolds{c} \in \set{C}_{t-1}$ such that $\bphi_k^{\trans}\bolds{c} > \gamma$. By definition of $\set{C}_{t-1}$, there exist $\btheta^+ \in \set{C}^+_{t-1}$ and $\btheta^- \in \set{C}^-_{t-1}$ such that $\bolds{c} = \btheta^+ - \btheta^-$. 
    We consider the symmetric points of $\btheta^+$ and $\btheta^-$ in the confidence sets:  $\btheta^{',+} \coloneqq 2\hat{\btheta}^+(t-1) - \btheta^+$ and $\btheta^{',-} \coloneqq 2\hat{\btheta}^-(t-1) - \btheta^-$. The corresponding difference, e.g., the symmetric point for $\bolds{c}$, is $\bolds{c}^{'} \coloneqq \btheta^{',+} - \btheta^{',-}.$  We know $\btheta^{',+} \in \set{C}^+_{t-1}$ since $\btheta^{',+} - \hat{\btheta}^+(t-1) = -(\btheta^{+} - \hat{\btheta}(t-1))$. Similarly  $\btheta^{',-} \in \set{C}^-_{t-1}$ and thus $\bolds{c}^{'} \in \set{C}_{t-1}.$ However, recall that we are in the scenario where $\bphi_k^{\trans}\hat{\bolds{c}}(t-1)\leq 0$ and thus $\bphi_k^{\trans}\bolds{c}^{'} = \bphi_k^{\trans}(2\hat{\bolds{c}}(t-1) - \bolds{c}) < 0 - \gamma = -\gamma$. This contradicts the assumption that $\ubar{c}_k(t) = \min_{\bolds{c} \in \set{C}_{t-1}} \bphi_k^{\trans}\bolds{c} \geq -\gamma$. Therefore, $\bar{c}_k(t) > \gamma$ is impossible.
\end{itemize}
\item $\bar{c}_k(t) \leq \gamma$: We also have $Z_k(t) \leq \gamma \indic{\gamma \geq \eta}$ by a symmetric argument of the above case;
\item $\ubar{c}_k(t) \leq -\gamma < \gamma \leq \bar{c}_k(t)$: in this case the label-driven admission will admit the job unless $Q^{\ld}(t) = 1$. Since $E(t) = 0$, we must have $Q^{\ld}(t) = 1$, and thus $Z_k(t) \leq c_{\max} Q^{\ld}(t)$.
\end{itemize}
Combining the above cases, $Z_k(t) \leq (\gamma\indic{\gamma \geq \eta} +c_{\max}Q^{\ld}(t))\Lambda_k(t)$ conditioned on $\set{E}$. 
\end{proof}

\subsection{Idiosyncrasy Loss and Lagrangians with Type-Aggregated Duals (Lemma~\ref{lem:cbacidol-idio-lag})}\label{app:lem-cbacidol-idio-lag}
The proof is close to that of Lemmas~\ref{lem:connect-idio-lag} and \ref{lem:bacidol-idio-lag}, but relies on the analysis of a different Lyapunov function.
\begin{proof}
By definition, $\tilde{L}(1) = 0$ and $\expect{\tilde{L}(T + 1) / \beta}$ upper bounds the idiosyncrasy loss, which proves the first inequality. For the second inequality, for any period $t$:
\begin{align*}
&\hspace{0.2in}\expect{\tilde{L}(t+1) - \tilde{L}(t)} \\
&=\beta\expect{\ell_{\kappa(t)}(1 - A(t) - E(t))}+\expect{\sum_{g \in \set{G}}\frac{1}{2}\left((\tilde{Q}_g(t+1))^2 - (\tilde{Q}_g(t))^2\right)} \\
&= \expect{\beta\ell_{\kappa(t)}(1 - A(t) - E(t))}+\frac{1}{2}\expect{\sum_{g \in \set{G}}\left(\sum_{k:g(k) = g} (A_k(t) - S_k(t))\right)^2} \\
&\hspace{0.3in} + \expect{\sum_{g \in \set{G}} \tilde{Q}_g(t)\left(\sum_{k:g(k) = g} (A_k(t) - S_k(t))\right)} \\
&\overset{(a)}{\leq} 1 + \beta\sum_{k \in \set{K}} \ell_k \lambda_k(t) + \expect{-\beta \ell_{\kappa(t)}(A(t) + E(t)) + \sum_{k \in \set{K}} \tilde{Q}_{g(k)}(t)\left(A_k(t) - S_k(t)\right)} \\
&\leq 1 + \beta\sum_{k \in \set{K}} \ell_k\lambda_k(t) + \expect{-\beta\sum_{k \in \set{K}} \ell_k \bar{A}_k(t)+ \sum_{k \in \set{K}} \tilde{Q}_{g(k)}(t)\left(\bar{A}_k(t) - S_k(t)\right)}
\end{align*}
where inequality (a) is because there is at most one type with $A_k(t) = 1$ and at most one type with $S_k(t) = 1$; the last inequality is because $A_k(t) \leq \bar{A}_k(t) \leq A_k(t) + E(t)$. In addition, $\expect{S_k(t) \mid \psi_k(t)} = \mu_kN(t)\psi_k(t)$. As a result, recalling the definition of $f_t$ from \eqref{eq:per-period-lag},
\begin{align*}
\expect{\tilde{L}(t+1) - \tilde{L}(t)} 
&\leq 1 + \beta\sum_{k \in \set{K}} \ell_k\lambda_k(t) - \expect{\sum_{k \in \set{K}} \bar{A}_k(t)\left(\beta \ell_k -\tilde{Q}_{g(k)}(t)\right) + \sum_{k \in \set{K}} \tilde{Q}_{g(k)}(t)\psi_k(t)\mu_k N(t)} \\
&= 1 + \expect{f_t(\bolds{\bar{A}}(t),\bolds{\psi}(t), \bolds{Q}^{\ta}(t))}
\end{align*}
where $Q^{\ta}_k(t) = \tilde{Q}_{g(k)}(t)$. We obtain the desired result by telescoping from $t = 1$ to $t = T$.
\end{proof}

\subsection{Lagrangians with Type-Aggregated Duals and Fluid Benchmark (Lemma~\ref{lem:context-benchmark-lagran})}\label{app:lem-context-benchmark-lagran}
Denote the optimal admission and service vectors to \eqref{eq:fluid} by $\{\bolds{a}^\star(t),\bolds{\nu}^\star(t)\}_{t \in [T]}$. The proof follows a similar structure with the proof of Lemma~\ref{lem:lagrang-bacid} (Section~\ref{sec:lagrang-bacid}) but uses a different dual $\bolds{Q}^{\ta}$. We define a new vector of dual variables $\bolds{u}^\star$ such that the dual $u^\star_{t,k}$ takes the corresponding queue length in $\bolds{Q}^{\ta}$, i.e.,  $u^\star_{t,k} = Q^{\ta}_k(t) = \tilde{Q}_{g(k)}(t)$ for $t \in [T], k \in \set{K}$. Then
\begin{align}
&\hspace{0.1in}\sum_{t=1}^T \expect{f_t(\bolds{A}^{\tabacid}(t),\bolds{\psi}^{\tabacid}(t),\bolds{Q}^{\ta}(t))} - \beta\set{L}^\star(T) \nonumber\\
&= \expect{f(\{\bolds{a}^\star(t)\}_t,\{\bolds{\nu}^\star(t)\}_t,\bolds{u}^\star)}-\beta\set{L}^\star(T)\label{eq:context-primary-lagrang} \\
&\hspace{-0.5in}+\sum_{t=1}^T \left(\expect{f_t(\bolds{A}^{\tabacid}(t),\bolds{\psi}^{\tabacid}(t),\bolds{Q}^{\ta}(t))} - \expect{f_t(\bolds{a}^\star(t),\bolds{\nu}^\star(t),\bolds{Q}^{\ta}(t))}\right) \label{eq:context-per-period-subopt}.
\end{align}
The first term (difference between a Lagrangian and the primal) is non-positive by Lemma~\ref{lem:bound-lagrang}. 

The next lemma is a per-period bound of  \eqref{eq:context-per-period-subopt}.
\begin{lemma}\label{lem:cbacid-optimal}
For any $t$, 
\[\expect{f_t(\bolds{A}^{\tabacid}(t),\bolds{\psi}^{\tabacid}(t),\bolds{Q}^{\ta}(t))} \leq  \expect{f_t(\bolds{a}^\star(t),\bolds{\nu}^\star(t),\bolds{Q}^{\ta}(t))} + 2\beta c_{\max}\Delta(\set{K}_{\set{G}}).\]
\end{lemma}
\begin{proof}
For any $\bolds{\tilde{q}}$ such that $\tilde{q}_g \leq 2\beta c_{\max}$, we expand the expectation on the left hand side, conditioning on $\bolds{\tilde{Q}}(t) = \bolds{\tilde{q}}$, by 
\begin{align}
&\hspace{0.1in}\expect{f_t(\bolds{A}^{\tabacid}(t),\bolds{\psi}^{\tabacid}(t),\bolds{Q}^{\ta}(t)) \mid \bolds{\tilde{Q}}(t) = \bolds{\tilde{q}}} \nonumber\\
&=\beta\sum_{k \in \set{K}} \ell_k\lambda_k(t) - \left(\sum_{k \in \set{K}} \lambda_k(t)\indic{\beta \ell_k \geq \tilde{q}_{g(k)}}(\beta \ell_k - \tilde{q}_{g(k)}) + \sum_{k \in \set{K}} \psi^{\tabacid}_k(t)\tilde{q}_{g(k)}\mu_k N(t)\right) \nonumber\\
&\overset{(a)}{\leq} \beta\sum_{k \in \set{K}} \ell_k\lambda_k(t) - \sum_{k \in \set{K}} a_k^\star(t)(\beta \ell_k - \tilde{q}_{g(k)}) - \sum_{k \in \set{K}} \psi^{\tabacid}_k(t)\tilde{q}_{g(k)}\mu_k N(t) \nonumber\\
&\overset{(b)}{\leq} \beta\sum_{k \in \set{K}} \ell_k\lambda_k(t) - \sum_{k \in \set{K}} a_k^\star(t)(\beta \ell_k - \tilde{q}_{g(k)}) - \sum_{g \in \set{G}} \tilde{\mu}_g \tilde{q}_g N(t)\sum_{k\in \set{K}_g} \psi^{\tabacid}_k(t) \nonumber \\
&\overset{(c)}{\leq} \beta\sum_{k \in \set{K}} \ell_k\lambda_k(t) - \sum_{k \in \set{K}} a_k^\star(t)(\beta \ell_k - \tilde{q}_{g(k)}) - \sum_{g \in \set{G}} \tilde{\mu}_g \tilde{q}_g N(t)\sum_{k\in \set{K}_g} \nu^\star_k(t) \nonumber\\
&\overset{(d)}{\leq} \beta\sum_{k \in \set{K}} \ell_k\lambda_k(t) - \sum_{k \in \set{K}} a_k^\star(t)(\beta \ell_k - \tilde{q}_{g(k)}) - \sum_{g \in \set{G}} \tilde{q}_g \sum_{k\in \set{K}_g} \left(N(t)\mu_k - \Delta(\set{K}_{\set{G}})\right)\nu^\star_k(t) \nonumber\\
&\overset{(e)}{\leq} \beta\sum_{k \in \set{K}} \ell_k\lambda_k(t) - \sum_{k \in \set{K}} a_k^\star(t)(\beta \ell_k - \tilde{q}_{g(k)}) - \sum_{g \in \set{G}} \tilde{q}_g N(t)\sum_{k\in \set{K}_g} \mu_k\nu^\star_k(t) + 2\beta c_{\max}\Delta(\set{K}_{\set{G}}) \nonumber\\
&= f_t(\bolds{a}^\star,\bolds{\nu}^\star, \bolds{Q}^{\ta}(t)) + 2\beta c_{\max}\Delta(\set{K}_{\set{G}}), \label{eq:bound-lag-gap}
\end{align}
where the first equality is by the admission rule of $\tabacid$; Inequality $(a)$ is because $a_k^\star(t) \leq \lambda_k(t)$ and $(\beta \ell_k-\tilde{q}_{g(k)}) \leq \indic{\beta \ell_k\geq\tilde{q}_{g(k)}}(\beta \ell_k-\tilde{q}_{g(k)})$; Inequality (b) is because $\tilde{\mu}_g = \min_{k \in \set{K}_g} \mu_k$; Inequality (c) is by the scheduling rule of $\tabacid$ that picks a group maximizing $\tilde{\mu}_g\tilde{Q}_g(t)$ for review; Inequality (d) is by the definition of aggregation gap in Section~\ref{sec:aggregate}; and inequality (e) is because $\tilde{q}_g \leq 2\beta c_{\max}$ for any group $g$ and $\sum_{k \in \set{K}} \nu^\star_k(t) \leq 1$. 

As in the proof of Lemma~\ref{lem:bacidol-delay}, we always have $\tilde{Q}_g(t) \leq 2\beta c_{\max}$. We then finish the proof by
\begin{align*}
&\hspace{0.1in}\expect{f_t(\bolds{A}^{\tabacid}(t),\bolds{\psi}^{\tabacid}(t),\bolds{Q}^{\ta}(t))} \\
&= \expect{f_t(\bolds{A}^{\tabacid}(t),\bolds{\psi}^{\tabacid}(t),\bolds{Q}^{\ta}(t))\indic{\forall g,\tilde{Q}_g(t) \leq 2c_{\max}}}  \\
&\overset{\eqref{eq:bound-lag-gap}}{\leq} \expect{(f_t(\bolds{a}^\star,\bolds{\nu}^\star, \bolds{Q}^{\ta}(t)) + 2\beta c_{\max}\Delta(\set{K}_{\set{G}}))\indic{\forall g, \tilde{Q}_g(t) \leq 2\beta c_{\max}}} \\
&= \expect{f_t(\bolds{a}^\star(t),\bolds{\nu}^\star(t),\bolds{Q}^{\ta}(t))} + 2\beta c_{\max}\Delta(\set{K}_{\set{G}}).
\end{align*}
\end{proof}
\begin{proof}[Proof of Lemma~\ref{lem:context-benchmark-lagran}]
By the discussion in the beginning of this section, it holds that: 
\begin{align*}
\sum_{t=1}^T \expect{f_t(\bolds{A}^{\tabacid}(t),\bolds{\psi}^{\tabacid}(t),\bolds{Q}^{\ta}(t))} - \beta\set{L}^\star(T) &= \eqref{eq:context-primary-lagrang} + \eqref{eq:context-per-period-subopt} \\
&\leq 0 + 2\beta c_{\max}\Delta(\set{K}_{\set{G}})T,
\end{align*}
where we use Lemmas~\ref{lem:bound-lagrang} and \ref{lem:cbacid-optimal} to bound the two terms respectively.
\end{proof}

\subsection{Bounding the regret in scheduling (Lemma~\ref{lem:cbacid-regS})}\label{app:lem-cbacid-regS}
The proof is similar to that of Lemma~\ref{lem:bacidol-regs} and we include it for completeness.
\begin{proof}
For a period $t$, by the scheduling decision of $\conbacid$ and $\tabacid$, we have $\psi_k(t) = \psi_k^{\tabacid}(t)$ for any $k$ if $Q^{\ld}(t) = 0$. Otherwise, if $Q^{\ld}(t) = 1$, then $\psi_k(t) = 0$ because of forced scheduling. As a result,
\begin{align*}
\textsc{RegS}(T) &= \expect{\sum_{t=1}^T \sum_{k \in \set{K}} \psi_k^{\tabacid}(t)\tilde{Q}_{g(k)}(t)\mu_k N(t) Q^{\ld}(t)} \\
&\leq 2\beta c_{\max}\expect{\sum_{t=1}^T \sum_{k \in \set{K}} \psi_k^{\tabacid}(t)Q^{\ld}(t)} \\
&\leq 2\beta c_{\max}\expect{\sum_{t=1}^T Q^{\ld}(t)} \leq \frac{68\beta c_{\max}B_{\delta}^2(T)d\ln(1+T/d)}{\max(\eta,\gamma)^2\hat{\mu}_{\min}},
\end{align*}
where the first inequality is because $\tilde{Q}_g(t) \leq 2\beta c_{\max}$; the second one is by $\sum_{k \in \set{K}} \psi_k^{\tabacid}(t) \leq 1$; the last one is by Lemmas~\ref{lem:connect-et-qe} and \ref{lem:bound-norm-selective} that together bound $\expect{\sum_{t=1}^T Q^{\ld}(t)}$.
\end{proof}

\subsection{Bounding the Regret in Admission (Lemma~\ref{lem:conbacid-rega})}\label{app:lem-conbacid-rega}
We first need a lemma that bounds $\bar{\ell}_k(t) - \ell_k$.
\begin{lemma}\label{lem:bound-r-k-bar}
Conditioned on $\set{E}$, it holds that $\bar{\ell}_k(t) - \ell_k \leq 2B_{\delta}(T)\|\bphi_k\|_{\bar{V}_{t-1}^{-1}}$ for any type $k$.
\end{lemma}
\begin{proof}
Conditioning on $\set{E}$ gives $\btheta^{\star,+} \in \set{C}^+_{t-1}$ and $\btheta^{\star,-} \in \set{C}^-_{t-1}$. As a result, for any $\btheta^+ \in \set{C}^+_{t-1}$, 
\begin{align*}
|\bphi_k^{\trans}\btheta^+ - \ell_k^+| = |\bphi_k^{\trans}\btheta^+ - \bphi_k^{\trans}\btheta^{\star,+}| &\leq |\bphi_k^{\trans}\btheta^+ - \bphi_k^{\trans}\hat{\btheta}^+(t-1)| + |\bphi_k^{\trans}\hat{\btheta}^+(t-1) - \bphi_k^{\trans}\btheta^{\star,+}| \\
&\leq \|\bphi_k\|_{\bar{V}_{t-1}^{-1}}\left(\|\btheta^+ - \hat{\btheta}^+(t-1)\|_{\bar{V}_{t-1}} + \|\btheta^{\star,+} - \hat{\btheta}^+(t-1)\|_{\bar{V}_{t-1}}\right) \tag{Cauchy-Schwarz inequality} \\
&\leq 2B_{\delta}(T)\|\bphi_k\|_{\bar{V}_{t-1}^{-1}} \tag{By \eqref{eq:contextual-def-conf} and that $\btheta^{\star,+},\btheta^+ \in \set{C}^+_{t-1}$}.
\end{align*}
Similarly, we also have $|\bphi_k^{\trans}\btheta^- - \ell_k^-| \leq 2B_{\delta}(T)\|\bphi_k\|_{\bar{V}_{t-1}^{-1}}$ for any $\btheta^- \in \set{C}_{t-1}^-.$ As a result, for any $\btheta^+ \in \set{C}_{t-1}^+, \btheta^- \in \set{C}_{t-1}^-$,
\[
\min(\bphi_k^{\trans}\btheta^+,\bphi_k^{\trans}\btheta^-) \leq \min(\ell_k^+ + 2B_{\delta}(T)\|\bphi_k\|_{\bar{V}_{t-1}^{-1}},\ell_k^- + 2B_{\delta}(T)\|\bphi_k\|_{\bar{V}_{t-1}^{-1}}) = \min(\ell_k^+,\ell_k^-) + 2B_{\delta}(T)\|\bphi_k\|_{\bar{V}_{t-1}^{-1}}.
\]
The definition of $\bar{\ell}_k(t)$ in \eqref{eq:conf-r-feature} then gives
\[
\bar{\ell}_k(t) - \ell_k = \max_{\btheta^+ \in \set{C}_{t-1}^+, \btheta^- \in \set{C}_{t-1}^-} \min(\bphi_k^{\trans}\btheta^+,\bphi_k^{\trans}\btheta^-) - \min(\ell_k^+,\ell_k^-) \leq 2B_{\delta}(T)\|\bphi_k\|_{\bar{V}_{t-1}^{-1}}.
\]
\end{proof}

\begin{proof}[Proof of Lemma~\ref{lem:conbacid-rega}]
Recall that the $\tabacid$ and optimistic admission decisions are \[A_k^{\tabacid}(t) = \Lambda_k(t)\indic{\beta \ell_k \geq \tilde{Q}_{g(k)}(t)}\quad\text{and}\quad~\bar{A}_k(t) = \Lambda_k(t)\indic{\beta \bar{\ell}_k(t) \geq \tilde{Q}_{g(k)}(t)}.\] Recall also the good event $\set{E}$ for which $\bolds{\Theta}^\star \in \set{C}_t^+ \times \set{C}_t^-$ for any $t \geq 0$. The regret in admission is
\begin{align}
\textsc{RegA}(T)&=\expect{\sum_{t=1}^T \sum_{k\in\set{K}} (A_k^{\tabacid}(t) - \bar{A}_k(t))(\beta \ell_k-\tilde{Q}_{g(k)}(t))} \nonumber\\
&= \expect{\sum_{t=1}^T \sum_{k\in\set{K}} (A_k^{\tabacid}(t) - \bar{A}_k(t))(\beta \ell_k-\tilde{Q}_{g(k)}(t))\left(\indic{\set{E}} + \indic{\set{E}^c}\right)} \nonumber\\
&\leq \expect{\sum_{t=1}^T \sum_{k\in\set{K}} (A_k^{\tabacid}(t) - \bar{A}_k(t))(\beta \ell_k-\tilde{Q}_{g(k)}(t))\indic{\set{E}}} + \Pr\{\set{E}^c\}\sum_{t=1}^T 3\beta c_{\max} \nonumber\\
&\leq \expect{\sum_{t=1}^T \sum_{k\in\set{K}} (A_k^{\tabacid}(t) - \bar{A}_k(t))(\beta \ell_k-\tilde{Q}_{g(k)}(t))\indic{\set{E}}} + 3\beta c_{\max} \label{eq:con-rega-decomp}
\end{align}
where the first inequality is by $\sum_k A_k^{\tabacid}(t) \leq 1, \sum_k \bar{A}_k(t) \leq 1$ and $\tilde{Q}_{g(k)}(t) \leq 2\beta c_{\max}$; the last one is by Lemma~\ref{lem:contextual-prob-good-event} that $\Pr\{\set{E}^c\} \leq 2\delta$ and the assumption that~$\delta \leq 1 / (2T)$. 

We next upper bound the first term in \eqref{eq:con-rega-decomp}. Conditioning on $\set{E}$, we bound $(A_k^{\tabacid}(t) - \bar{A}_k(t))(\beta \ell_k-\tilde{Q}_{g(k)}(t))$ for a fixed period $t$ and a fixed type $k$. Since $\bTheta^\star \in \set{C}^+_{t-1} \times \set{C}^-_{t-1}$ by $\set{E}$, it holds that $\bar{\ell}_k(t) \geq \ell_k$. If there is no type-$k$ arrival, i.e., $\Lambda_k(t) = 0$, we have that $A_k^{\tabacid}(t) = \bar{A}_k(t) = 0$ and thus the term is zero. Otherwise, since $\bar{\ell}_k(t) \geq \ell_k$, $\bar{A}_k(t) = \indic{\beta \bar{\ell}_k(t) \geq \tilde{Q}_{g(k)}(t)} \geq \indic{\beta \ell_k \geq \tilde{Q}_{g(k)}(t)} = A_k^{\tabacid}(t)$. Therefore, $(A_k^{\tabacid}(t) - \bar{A}_k(t))(\beta \ell_k-\tilde{Q}_{g(k)}(t))$ is positive only if $A_k^{\tabacid}(t) = 0$ and $\bar{A}_k(t) = 1$, under which we have
\begin{align}
(A_k^{\tabacid}(t) - \bar{A}_k(t))(\beta \ell_k-\tilde{Q}_{g(k)}(t)) &= (\tilde{Q}_{g(k)}(t) - \beta \ell_k)\bar{A}_k(t) \nonumber\\
&= (\tilde{Q}_{g(k)}(t) - \beta \bar{\ell}_k(t) + \beta \bar{\ell}_k(t) - \beta \ell_k)\bar{A}_k(t) \nonumber\\
&\leq \beta \left(\bar{\ell}_k(t) - \ell_k\right)\bar{A}_k(t), \label{eq:cont-trans-reward-gap}
\end{align}
where the last inequality is because the optimistic admission rule only admits a job ($\bar{A}_k(t) = 1$) when $\tilde{Q}_{g(k)}(t) \leq \beta \bar{\ell}_k(t)$. We complete the proof by
\begin{align*}
&\hspace{0.1in}\expect{\sum_{t=1}^T \sum_{k\in\set{K}} (A_k^{\tabacid}(t) - \bar{A}_k(t))(\beta \ell_k-\tilde{Q}_{g(k)}(t))} \\
&\leq \expect{\sum_{t=1}^T \sum_{k \in \set{K}} \beta (\bar{\ell_k}(t)-\ell_k)\bar{A}_k(t)\indic{\set{E}}} + 3\beta c_{\max} \tag{By \eqref{eq:con-rega-decomp} and \eqref{eq:cont-trans-reward-gap}}  \\
&\leq 3\beta c_{\max}B_{\delta}(T)\left(\expect{\sum_{t=1}^T \sum_{k \in \set{K}} \|\bphi_k\|_{\bar{V}_{t-1}^{-1}}\bar{A}_k(t)} + 1\right) \tag{By Lemma~\ref{lem:bound-r-k-bar}} \\
&\leq 3\beta c_{\max}B_{\delta}(T)\left(d\ln(1+T/d)\left(4GQ_{\max} + \frac{34B_{\delta}^2(T)}{\max(\eta,\gamma)^2}\right) + \sqrt{2GTd\ln(1+T/d)}\right)
\end{align*} where we use Lemma~\ref{lem:conbacid-error} for the last inequality.
\end{proof}

\subsection{Bounding estimation error with fixed delay (Lemma~\ref{lem:bound-delayed-error})}\label{app:lem-bound-delayed-error}
\begin{proof}
We first bound the case when there is no delay, i.e., $q = 1$, which is the case in \cite{Abbasi-YadkoriPS11}. In particular, assuming that we have $M$ data points, by Cauchy-Schwarz Inequality
\begin{align}
\sum_{i=1}^M \|\hat{\bphi}_i\|_{\hat{\bV}_{i-1}^{-1}} \leq \sqrt{M\sum_{i=1}^M \|\hat{\bphi}_i\|_{\hat{\bV}_{i-1}^{-1}}^2} \leq \sqrt{2M\ln(\det(\hat{\bV}_n) / \xi^d))} &\leq  \sqrt{2Md\ln(1+MU^2/(d\xi))} \nonumber\\
&\leq \sqrt{2Md\ln(1+M/d)}, \label{eq:bound-q=1}
\end{align}
where we use Fact~\ref{fact:ellipsoid-bound} for the second inequality, Fact~\ref{fact:matrix-norm-bound} for the third inequality, and $\xi \geq U^2$ (assumption in the lemma) for the last inequality. To bound for arbitrary $q$, we observe that 
\begin{equation}\label{eq:trans-q-to-1}
\sum_{i=q}^M \|\hat{\bphi}_i\|_{\hat{\bV}_{i-q}^{-1}} \leq \sum_{i=q}^M \|\hat{\bphi}_i\|_{\hat{\bV}_{i-q}^{-1}} + \sum_{i=1}^{q-1} \|\hat{\bphi}_i\|_{\hat{\bV}_{i-1}^{-1}} =  \sum_{i=1}^M\|\hat{\bphi}_i\|_{\hat{\bV}_{i-1}^{-1}} + \sum_{i=q}^M \left(\|\hat{\bphi}_i\|_{\hat{\bV}_{i-q}^{-1}} - \|\hat{\bphi}_i\|_{\hat{\bV}_{i-1}^{-1}}\right).
\end{equation}
Note that the term $\sum_{i=1}^M\|\hat{\bphi}_i\|_{\hat{\bV}_{i-1}^{-1}}$ is bounded by \eqref{eq:bound-q=1}. The term $\sum_{i=q}^M \left(\|\hat{\bphi}_i\|_{\hat{\bV}_{i-q}^{-1}} - \|\hat{\bphi}_i\|_{\hat{\bV}_{i-1}^{-1}}\right)$ measures the difference between the estimation errors when there is no delay and when there is a delay of $q$. We next upper bound the second term for a fixed data point $i$ by 
\begin{align}
\|\hat{\bphi}_i\|_{\hat{\bV}_{i-q}^{-1}} - \|\hat{\bphi}_i\|_{\hat{\bV}_{i-1}^{-1}} &= \left(\sqrt{\hat{\bphi}_i^{\trans}\hat{\bV}_{i-q}^{-1}\hat{\bphi_i}} - \sqrt{\hat{\bphi}_i^{\trans}\hat{\bV}_{i-1}^{-1}\hat{\bphi_i}}\right) \nonumber\\
&\leq \frac{\hat{\bphi}_i^{\trans}\left(\hat{\bV}_{i-q}^{-1}-\hat{\bV}_{i-1}^{-1}\right)\hat{\bphi_i}}{\|\hat{\bphi}_i\|_{\hat{\bV}_{i-q}^{-1}}} \tag{$\forall a,b > 0$, $\sqrt{a} - \sqrt{b} =\frac{a-b}{\sqrt{a}+\sqrt{b}} \leq \frac{a-b}{\sqrt{a}}$} \nonumber\\
&\hspace{-1in}=\frac{\hat{\bphi}_i^{\trans}\hat{\bV}_{i-1}^{-1}\hat{\bV}_{i-1}\left(\hat{\bV}_{i-q}^{-1}-\hat{\bV}_{i-1}^{-1}\right)\hat{\bV}_{i-q}\hat{\bV}_{i-q}^{-1}\hat{\bphi}_i}{\|\hat{\bphi}_i\|_{\hat{\bV}_{i-q}^{-1}}} = \frac{\hat{\bphi}_i^{\trans}\hat{\bV}_{i-1}^{-1}\left(\sum_{\tau=i-q+1}^{i-1} \hat{\bphi}_{\tau}\hat{\bphi}_{\tau}^{\trans}\right)\hat{\bV}_{i-q}^{-1}\hat{\bphi}_i}{\|\hat{\bphi}_i\|_{\hat{\bV}_{i-q}^{-1}}}. \label{eq:expand-last}
\end{align}
Further expanding \eqref{eq:expand-last}, we have
\begin{align*}
\|\hat{\bphi}_i\|_{\hat{\bV}_{i-q}^{-1}} - \|\hat{\bphi}_i\|_{\hat{\bV}_{i-1}^{-1}} &\leq \frac{\sum_{\tau=i-q+1}^{i-1} \hat{\bphi}_i^{\trans}\hat{\bV}_{i-1}^{-1}\hat{\bphi}_{\tau}\hat{\bphi}_{\tau}^{\trans}\hat{\bV}_{i-q}^{-1}\hat{\bphi}_i}{\|\hat{\bphi}_i\|_{\hat{\bV}_{i-q}^{-1}}} \\
&\hspace{-1in}\leq \frac{\sum_{\tau=i-q+1}^{i-1} \|\hat{\bphi}_i\|_{\hat{\bV}_{i-1}^{-1}}\|\hat{\bphi}_{\tau}\|_{\hat{\bV}_{i-1}^{-1}}\|\hat{\bphi}_{\tau}\|_{\hat{\bV}_{i-q}^{-1}}\cancel{\|\hat{\bphi}_i\|_{\hat{\bV}_{i-q}^{-1}}}}{\cancel{\|\hat{\bphi}_i\|_{\hat{\bV}_{i-q}^{-1}}}} \\
&\hspace{-1in}\leq \sum_{\tau=i-q+1}^{i-1} \|\hat{\bphi}_i\|_{\hat{\bV}_{i-1}^{-1}}\|\hat{\bphi}_{\tau}\|_{\hat{\bV}_{i-1}^{-1}}
\end{align*}
where the second inequality is because \[\hat{\bphi}_i^{\trans}\hat{\bV}_{i-1}^{-1}\hat{\bphi}_{\tau} = \left(\hat{\bV}_{i-1}^{-1/2}\hat{\bphi}_i\right)^{\trans}\left(\hat{\bV}_{i-1}^{-1/2}\hat{\bphi}_{\tau}\right) \leq \|\hat{\bV}_{i-1}^{-1/2}\hat{\bphi}_i\|_2 \|\hat{\bV}_{i-1}^{-1/2}\hat{\bphi}_{\tau}\|_2 = \|\hat{\bphi}_i\|_{\hat{\bV}_{i-1}^{-1}}\|\hat{\bphi}_{\tau}\|_{\hat{\bV}_{i-1}^{-1}}\]
by Cauchy-Schwarz Inequality (same argument for $\hat{\bphi}_{\tau}^{\trans}\hat{\bV}_{i-q}^{-1}\hat{\bphi}_i$); the last inequality is because $\|\hat{\bphi}_{\tau}\|_{\hat{\bV}_{i-q}^{-1}} \leq 1$ under the assumption that $\xi \geq U^2$. As a result, 
\begin{align}
\sum_{i=q}^M \left(\|\hat{\bphi}_i\|_{\hat{\bV}_{i-q}^{-1}} - \|\hat{\bphi}_i\|_{\hat{\bV}_{i-1}^{-1}}\right) &\leq \sum_{i=q}^M \left(\sum_{\tau=i-q+1}^{i-1} \|\hat{\bphi}_i\|_{\hat{\bV}_{i-1}^{-1}}\|\hat{\bphi}_{\tau}\|_{\hat{\bV}_{i-1}^{-1}}\right) \nonumber\\
&\hspace{-1in}\leq \frac{1}{2}\sum_{i=q}^M \sum_{\tau=i-q+1}^{i-1} \left(\|\hat{\bphi}_i\|_{\hat{\bV}_{i-1}^{-1}}^2 + \|\hat{\bphi}_{\tau}\|_{\hat{\bV}_{i-1}^{-1}}^2\right) \tag{$ab \leq (a^2+b^2) / 2$} \nonumber\\
&\leq q\sum_{i=1}^M \|\hat{\bphi}_i\|_{\hat{\bV}_{i-1}^{-1}}^2 \leq 2qd\ln(1+M/d),\label{eq:bound-diff}
\end{align}
where the third inequality is because every term is counted for $q$ times; and the last two inequalities again follow again from Facts~\ref{fact:ellipsoid-bound} and \ref{fact:matrix-norm-bound}. Combining \eqref{eq:bound-q=1}, \eqref{eq:trans-q-to-1}, and \eqref{eq:bound-diff}, we complete the proof by
\[
\sum_{i=q}^M \|\hat{\bphi}_i\|_{\hat{\bV}_{i-q}^{-1}} \leq \sqrt{2Md\ln(1+M/d)}+2qd\ln(1+M/d).
\]
\end{proof}

\subsection{Bounding Estimation Error with Queueing Delays (Lemma~\ref{lem:conbacid-error})}\label{app:lem-conbacid-error}
Our proof starts by bounding $\sum_{t=1}^T \sum_k \|\bphi_k\|_{\bar{\bV}_{t-1}^{-1}}A_k(t)$ where recall that $A_k(t) = \bar{A}_k(t)(1-E(t))$ denotes whether a job is admitted by the optimistic admission rule (Line~\ref{line:conbacid-admit}). We bound the estimation error separately for different groups. In particular, fixing a group $g$, suppose that in the first $T$ periods there are $M_g$ group-$g$ jobs such that $A_k(t) = 1$ where $k$ is the type of this job and $t$ is the period it arrives. We denote the sequence of these jobs by $\set{J}_g$, their arrival periods by $t_{g,1}^A<\ldots<t_{g,M_g}^A$, and the time these jobs are reviewed by $t_{g,1}^R,\ldots, t_{g,M_g}^R$, where we denote $t_{g,i}^R = T+i$ if the $i-$th job is not reviewed by period $T$. 

Denoting $\hat{\bphi}_{g,i} = \bphi_{\kappa(t_{g,i}^A)}$, which is the feature vector of the $i-$ arrival in group $g$ ($\kappa(t_{g,i}^A)$ is its type), the estimation error of the $i-$th job on sequence $\set{J}_g$ is $\|\hat{\bphi}_{g,i}\|_{\bar{\bV}_{t_{g,i}^A - 1}^{-1}}$. Akin to the matrix $\bar{\bV}_t$ in \eqref{eq:regression} but only with the first $i$ jobs in $\set{J}_g$, we define $\widehat{\bV}_{g,i} = \xi \bI+\sum_{i'=1}^i \hat{\bphi}_{g,i'}\hat{\bphi}_{g,i'}^{\trans}$. Our first result connects the estimation error of a job in a setting with queueing delayed feedback, to a setting where the feedback delay is fixed to $Q_{\max}=2\beta c_{\max}$, the maximum number of group-$g$ jobs in the review queue.
\begin{lemma}\label{lem:connect-fix-delay}
For any group $g$ and $i \geq Q_{\max}$, the estimation error is $\|\hat{\bphi}_{g,i}\|_{\bar{\bV}_{t_{g,i}^A - 1}^{-1}} \leq \|\hat{\bphi}_{g,i}\|_{\widehat{\bV}_{g,i - Q_{\max}}}$.
\end{lemma}
\begin{proof}
Fix a group $g$. Since $\conbacid$ follows a first-come-first-serve scheduling rule for jobs of the same group in the review queue (Line~\ref{line:conbacid-schedule} in Algorithm~\ref{algo:cbacidol}), it holds that $t_{g,1}^R < \ldots < t_{g,M_g}^R$. Since $Q_{\max}$ upper bounds the number of group-$g$ jobs in the review queue by Lemma~\ref{lem:contextual-delay}, for the $i-$th job with $i \geq Q_{\max}$, the first $(i-Q_{\max})$ jobs must have already been reviewed before the admission of job $i$. Otherwise there would be $Q_{\max}+1$ group-$g$ jobs in the review system when we admit job $i$, which contradicts the fact that $Q_{\max}$ upper bounds the queue length. As a result, when we admit the $i-$th job, the first $i-Q_{\max}$ jobs in $\set{J}_g$ are already reviewed by the first-come-first-serve scheduling and are all accounted in $\bar{\bV}_{t_{g,i}^A - 1}$. By Lemma~\ref{lem:norm-subset}, it holds that $\|\hat{\bphi}_{g,i}\|_{\bar{\bV}_{t_{g,i}^A - 1}^{-1}} \leq \|\hat{\bphi}_{g,i}\|_{\widehat{\bV}_{g,i - Q_{\max}}}$. 
\end{proof}
As a result of Lemma~\ref{lem:connect-fix-delay}, the estimation error of jobs in sequence $\set{J}_g$ behaves as if having a fixed feedback delay. This allows the use of Lemma~\ref{lem:bound-delayed-error} to bound their total estimation error. Aggregating the error across groups gives the following result.
\begin{lemma}\label{lem:error-total}
The total estimation error of jobs admitted by the optimistic admission rule is bounded by $3GQ_{\max}d\ln(1+T/d) + \sqrt{2GTd\ln(1+T/d)}$.
\end{lemma}
\begin{proof}
The estimation error can be decomposed into groups by 
\begin{align*}
\sum_{t=1}^T \sum_k
\|\bphi_k\|_{\bar{\bV}_{t-1}^{-1}}A_k(t) = \sum_{g \in \set{G}} \sum_{i=1}^{M_g} \|\hat{\bphi}_{g,i}\|_{\bar{\bV}_{t_{g,i}^A - 1}^{-1}} &\leq  \sum_{g\in\set{G}} \sum_{i=1}^{Q_{\max}}  \|\hat{\bphi}_{g,i}\|_{(\xi \bI)^{-1}} + \sum_{g \in \set{G}} \sum_{i=Q_{\max}}^{M_g} \|\hat{\bphi}_{g,i}\|_{\bar{\bV}_{t_{g,i}^A - 1}^{-1}} \\
&\hspace{-1in}\leq GQ_{\max} + \sum_{g \in \set{G}} \sum_{i=Q_{\max}}^{M_g} \|\hat{\bphi}_{g,i}\|_{\widehat{\bV}_{g,i-Q_{\max}}} \tag{By $\xi \geq U^2$ and Lemma~\ref{lem:connect-fix-delay}} \\
&\hspace{-2in}\leq GQ_{\max} + \sum_{g\in\set{G}}\left(\sqrt{2M_gd\ln(1+T/d)}+2Q_{\max}d\ln(1+T/d)\right) \tag{By Lemma~\ref{lem:bound-delayed-error}} \\
&\hspace{-1in}\leq 3GQ_{\max}d\ln(1+T/d) + \sqrt{2GTd\ln(1+T/d)}
\end{align*}
where the last inequality is by $\sum_{g \in \set{G}} M_g \leq T$ and Cauchy-Schwarz Inequality.
\end{proof}

\begin{proof}[Proof of Lemma~\ref{lem:conbacid-error}]
Recall that $\bar{A}_k(t) = \Lambda_k(t)\indic{\beta \bar{\ell}_k(t) \geq \tilde{Q}_{g(k)}(t)}$ captures whether the optimistic admission would have admitted a type-$k$ job, $E(t)$ captures whether the label-driven admission admits a new job, and $A_k(t) = \bar{A}_k(t)(1-E(t))$ captures whether the optimistic admission ends up admitting a type-$k$ job into the review queue. We finish the proof by combining Lemma~\ref{lem:error-total} with the impact of label-driven admissions:
\begin{align*}
\expect{\sum_{t=1}^T \sum_k \|\bphi_k\|_{\bar{\bV}_{t-1}^{-1}}\bar{A}_k(t)} &\leq 
\expect{\sum_{t=1}^T \sum_k \|\bphi_k\|_{\bar{\bV}_{t-1}^{-1}} (A_k(t) + E(t)\Lambda_k(t))} \tag{$A_k(t) = \bar{A}_k(t)(1-E(t))$} \\
&\leq \expect{\sum_{t=1}^T \sum_k \|\bphi_k\|_{\bar{\bV}_{t-1}^{-1}} A_k(t)} + \expect{\sum_{t=1}^T E(t)} \tag{$\|\bphi_k\|_{\bar{\bV}_{t-1}^{-1}} \leq 1$} \\
&\hspace{-1in}\leq 3GQ_{\max}d\ln(1+T/d) + \sqrt{2GTd\ln(1+T/d)} + \frac{34B_{\delta}^2(T)d\ln(1+T/d)}{\max(\eta,\gamma)^2} \tag{By Lemmas~\ref{lem:bound-norm-selective} and \ref{lem:error-total}}  \\
&\leq d\ln(1+T/d)\left(3G Q_{\max} + \frac{34B_{\delta}^2(T)}{\max(\eta,\gamma)^2}\right) + \sqrt{2GTd\ln(1+T/d)}.
\end{align*} 
\end{proof}

\section{Extension to holding costs (Remark~\ref{remark:holding_loss})}\label{app:holding-cost}
To capture a scenario where an admitted job incurs holding cost while waiting for human review, this appendix considers an alternative objective to \eqref{eq:policy-loss} by including a holding cost function for every job type. In particular, for every type $k$, we assume there is a (cumulative) holding cost function $h_k(\cdot)$ such that for $d \geq 0$, $h_k(d)$ denotes the holding cost for a type-$k$ job if humans review this job after $d$ periods since its arrival. Naturally, we impose $h_k(0) = 0$ and $h_k(d + 1) \geq h_k(d)$.  Let $D(t)$ be the (random) number of periods between the admission of job $t$ and the period humans review this job. If the job is not admitted, we take $D(t) = 0$; if the job is admitted by not reviewed by period $T$, we take $D(t) = T + 1 - t$. We consider the below \emph{holding-adjusted} loss of a policy $\pi$:
\begin{equation}\label{eq:holding-loss}
\set{L}_{\bolds{h}}^{\pi}(T) \coloneqq \set{L}^{\pi}(T) + \sum_{t \leq T} h_{\kappa(t)}(D(t)),
\end{equation}
where the first term is identical to the \eqref{eq:policy-loss} and counts misclassification loss; the second term is the new component considering the holding cost for every job.

To minimize the holding-adjusted loss of a policy, we consider its \emph{holding-adjusted} regret, which captures the additional holding-adjusted loss compared to the fluid benchmark in \eqref{eq:fluid} (our results similarly extend to the $w-$fluid benchmark in Appendix~\ref{app:w-fluid})): 
\begin{equation}
\reg_{\bolds{h}}^{\pi}(T) = \left(\expect{\set{L}_{\bolds{h}}^{\pi}(T)} - \set{L}^\star(T)\right)^+.
\end{equation}
Our focus is on the learning setting in Section~\ref{sec:unknown} and the goal is to again obtain sublinear regret. However, since the benchmark completely ignores holding cost, such a goal is not possible even when the holding cost has a simple form: a job has a cost of $1$ for the first waiting period and then zero holding cost. The proof of the below result is in the next subsection.
\begin{proposition}\label{prop:lower-bound-holding}
Consider a one-type stationary setting with $N(t) \equiv 1, \lambda_1 = \mu_1 = \ell_1=0.5$ and the holding cost function is $h_1(0) = 0, h_1(d) = 1$ for $d \geq 1$. Then for any horizon $T$, any policy $\pi$ must incur a linear holding-adjusted regret: $\reg_{\bolds{h}}^{\pi}(T) \geq 0.25T$.
\end{proposition}
Proposition~\ref{prop:lower-bound-holding} shows that sublinear holding-adjusted regret is not achievable when the marginal increase in holding cost is a constant. As a result, we focus on holding cost functions that have \emph{vanishing} cost increase. Specifically, for some known constant $\omega \in (0,1]$, we assume that for any $k \in \set{K}$ and $d \geq 1$, the marginal holding cost satisfies $h_k(d+1) - h_k(d) \leq T^{-\omega}$. As $T$ becomes large, the maximum marginal holding cost vanishes. For such holding cost functions, we show that $\bacidol$ with suitable hyper-parameters achieve sublinear holding-adjusted regret.

\begin{proposition} $\bacidol$ with $\beta = T^{\omega/2}$ and $\gamma = (T/(K\ln T))^{-1/3}$ has 
\[\reg_{\bolds{h}}^{\bacidol}(T) \lesssim K\sqrt{T\ln T} + KT^{1-\omega/2} + \min\left(\frac{K\ln T}{\eta^2}, T^{2/3}(K\ln T)^{1/3}\right) = o(T).\]
\end{proposition}
\begin{proof}
The proof is similar to the proof of Theorem~\ref{thm:bacidol}. We first upper bound the expected holding cost under $\bacidol$ by
\begin{equation}\label{eq:bound-holding}
\expect{\sum_{t \leq T} h_{\kappa(t)}(D(t))} \leq \expect{\sum_{t \leq T} T^{-\omega} D(t)} = T^{-\omega}\expect{\sum_{t \leq T + 1} |\set{Q}(t)|},
\end{equation}
where the first inequality uses the assumption that $h_k(0) = 0$ and $h_k(d+1)-h_k(d) \leq T^{-\omega}$ for any $k \in \set{k}, d \geq 0$; the equality uses Little's Law.
Using the same notations as in \eqref{eq:learning-loss-decompose} and plugging \eqref{eq:bound-holding} into \eqref{eq:holding-loss}, we bound the holding-adjusted loss of $\bacidol$ by
\begin{align*}
\expect{\set{L}^{\pi}_{\bolds{h}}(T)} &\leq \expect{\set{L}^{\pi}(T)} + T^{-\omega}\expect{\sum_{t\leq T+1} |\set{Q}(t)|} \\
&\leq \text{Classification Loss} + \text{Idiosyncrasy Loss} + \text{Relaxed Delay Loss} + T^{-\omega}\expect{\sum_{t =1}^{T+1} \sum_{k \in \set{K}} Q_k(t)}.
\end{align*}
Since $\beta \geq 1$, Lemmas~\ref{lem:bacidol-delay}, \ref{lem:bacidol-class}, \ref{lem:bacidol-idio} apply and bound the first three terms. Using \eqref{eq:bound-olbacid} (and omitting $c_{\max}$) then gives 
\[
\expect{\set{L}^{\pi}_{\bolds{h}}(T)} \lesssim \set{L}^\star(T) + \beta K + T / \beta + \min\left(\frac{K\ln T}{\eta^2}, T^{2/3}(K\ln T)^{1/3}\right) + K\sqrt{T\ln T} + T^{-\omega}\expect{\sum_{t =1}^{T+1} \sum_{k \in \set{K}} Q_k(t)}.
\]
For the last term, the proof of Lemma~\ref{lem:bacidol-delay} shows that $Q_k(t) \leq \beta c_{\max}+1$ for any $k \in \set{K}, t \in [T+1].$ Plugging this bound gives 
\begin{align*}
\expect{\set{L}^{\pi}_{\bolds{h}}(T)} &\lesssim \set{L}^\star(T) + \beta K + T / \beta + \min\left(\frac{K\ln T}{\eta^2}, T^{2/3}(K\ln T)^{1/3}\right) + K\sqrt{T\ln T} + \beta K T^{1-\omega} \\
&\lesssim K\sqrt{T\ln T} + KT^{1-\omega/2} + \min\left(\frac{K\ln T}{\eta^2}, T^{2/3}(K\ln T)^{1/3}\right),
\end{align*}
where the last inequality uses $\beta = T^{\omega / 2}.$
\end{proof}

\subsection{Necessity of vanishing marginal holding cost (Proposition~\ref{prop:lower-bound-holding})}\label{sec:pf-lowerbound-holding}
\begin{proof}[Proof of Proposition~\ref{prop:lower-bound-holding}]
Since there is only one type, this proof omits the dependence on the type for all notations. Because $\lambda = \mu$, the fluid benchmark \eqref{eq:fluid} has zero loss by admitting all jobs, i.e., $\set{L}^\star(T) = 0.$ It remains to show $\expect{\set{L}_{h}^{\pi}(T)} \geq 0.25T$ for any feasible policy $\pi$. Using the definition \eqref{eq:holding-loss} and recalling $\Lambda(t) = 1$ if there is a (type-1) job,
\begin{align}
\expect{\set{L}_{\bolds{h}}^{\pi}(T)} &= \expect{\set{L}^{\pi}(T)}+\expect{\sum_{t \leq T} \Lambda(t) h(D(t))} \nonumber\\
&\geq  \expect{\sum_{t = 1}^T \ell\Lambda(t)(1 - A(t))} + \expect{\sum_{t \leq T} \Lambda(t) A(t)\indic{D(t) \geq 1}} \nonumber\\
&\overset{\ell=0.5}{=} 0.25 T + \expect{\sum_{t \leq T} \Lambda(t) A(t)(\indic{D(t) \geq 1} - 0.5)}. \label{eq:holding-lower-bound}
\end{align}
where the first inequality is by applying \eqref{eq:loss-decompose} to $\expect{\set{L}^{\pi}(T)}$ and by using the fact that $h(d) = 1$ for $d \geq 1$, which implies that job $t$ incurs a holding cost of $1$ if and only if it is admitted ($A(t)=1$) and it has positive delay ($D(t) \geq 1$). For the second term on the right hand side of \eqref{eq:holding-lower-bound}, we claim that it is non-negative. To see this, when the random service for period $t$, $S(t)$, is equal to zero (i.e., humans do not finish the review), then the admitted job (if any) must have a wait time of at least one period. Therefore, $\Lambda(t)A(t)\indic{D(t) \geq 1} \geq \Lambda(t)A(t)\indic{S(t)=0}=\Lambda(t)A(t)(1 - S(t)).$ Plugging this inequality to \eqref{eq:holding-lower-bound}, we obtain
\begin{align*}
\expect{\set{L}_{\bolds{h}}^{\pi}(T)} &\geq 0.25T + \expect{\sum_{t \leq T} \Lambda(t)A(t)(0.5-S(t))} \\
&= 0.25T + \expect{\sum_{t \leq T} \Lambda(t)A(t)\left(0.5-\expect{S(t) \mid \Lambda(t)A(t)}\right)} \\
&\geq 0.25T + \expect{\sum_{t \leq T} \Lambda(t)A(t)\left(0.5-0.5\right)} = 0.25T,
\end{align*}
where the first equality is using the tower property of conditional expectation and the second inequality is by the fact that conditional on $\Lambda(t)A(t)$, the random variable $S(t)$ is either deterministically zero (the platform does not schedule a job for review), or is a Bernoulli random variable of mean $N(t)\mu = 0.5$ because our model assumes independent service completion. The proof finishes by using $\reg_{\bolds{h}}^{\pi}(T) = (\expect{\set{L}_{\bolds{h}}^{\pi}(T)} - \set{L}^\star(T))^+ = (\expect{\set{L}_{\bolds{h}}^{\pi}(T)})^+ \geq 0.25T.$ 
\end{proof}
\section{Extension to $w-$fluid benchmarks (Remark~\ref{remark:w-fluid})}\label{app:w-fluid}
To motivate our $w-$fluid benchmarks, which include \eqref{eq:fluid} as a special case, we start from the stationary case where $\lambda_k(t) \equiv \lambda_k$ and $N(t) \equiv N$ across any period $t \leq T$. The classical fluid benchmark corresponds to a linear program (LP) as in \eqref{eq:T-fluid}. In particular, the objective is the expected total loss and the first constraint captures the capacity constraint requiring all admitted posts are reviewed eventually in expectation. An alternative formulation is to satisfy the capacity constraint for \emph{each period} as in \eqref{eq:fluid}, which is restated in \eqref{eq:1-fluid}. With the stationary property that $\lambda_k(t) \equiv \lambda_k$ and $N(t) \equiv N$, the two LPs are equivalent. 
 
\noindent
\begin{minipage}[t]{0.5\textwidth}
\begin{equation}\label{eq:T-fluid}
\begin{aligned}
&\min_{\{a_k(t),\nu_k(t)\}}\sum_{t=1}^T \sum_{k \in \set{K}} \ell_k(\lambda_k(t) - a_k(t)),\text{s.t.} \\
\quad& \bolds{\sum_{t=1}^{T} a_k(t) \leq \mu_k\sum_{t=1}^{T} N(t)\nu_k(t),\forall k \in \set{K}} \\
&a_k(t) \leq \lambda_k(t),~\nu_k(t) \geq 0,~\forall k \in \set{K}, t\in [T] \\
& \sum_{k\in \set{K}} \nu_k(t) \leq 1,~\forall t \in [T].
\end{aligned}
\end{equation}
\end{minipage}
\begin{minipage}[t]{0.5\textwidth}
\begin{equation}\label{eq:1-fluid}
\begin{aligned}
&\min_{\{a_k(t),\nu_k(t)\}}\sum_{t=1}^T \sum_{k \in \set{K}} \ell_k(\lambda_k(t) - a_k(t)),\text{s.t.} \\
\quad& \bolds{a_k(t) \leq \mu_k N(t)\nu_k(t),\forall k \in \set{K}, t \in [T]} \\
&a_k(t) \leq \lambda_k(t),~\nu_k(t) \geq 0,~\forall k \in \set{K}, t\in [T] \\
& \sum_{k\in \set{K}} \nu_k(t) \leq 1,~\forall t \in [T].
\end{aligned}
\end{equation}
\end{minipage}
Extending the fluid benchmark to a non-stationary setting is non-trivial. In particular, the benchmark in \eqref{eq:T-fluid} allows the decision maker to use human capacity of the entire time horizon to review any job and can greatly underestimate the loss any policy must incur, as illustrated in Example~\ref{ex:fluid}. Alternatively, using the benchmark in \eqref{eq:1-fluid} can overestimate the loss of an optimal policy. It requires an admitted job to be reviewed in the same period it arrives, even though an admitted job can wait in the review queue and be reviewed by later human capacity.
\begin{example}\label{ex:fluid}
Consider a setting with one type of jobs. Arrival rates and review capacity are time-varying such that $\lambda(t) = 0, \mu N(t) = 1$ for $t \leq T / 2$ and $\lambda(t) = 1, \mu N(t) = 0$ for $t > T / 2$. Benchmark \eqref{eq:T-fluid} will always give a loss of $0$. However, the uneven capacity indeed makes it a ``difficult'' setting where no job can be reviewed.
\end{example}
Given that \eqref{eq:T-fluid} is too strong and \eqref{eq:1-fluid} is too weak as the benchmark, we consider a series of fluid benchmarks that interpolate between them, which we call the \emph{$w-$fluid benchmarks}. In particular, we assume that the interval $\{1,\ldots, T\}$ can be partitioned into consecutive windows of sizes at most~$w$, such that the admission mass of the benchmark is no larger than the service capacity in each window. In other words, a $w-$fluid benchmark limits the range of human capacity used to review an admitted job and the available information: a larger window size allows an admitted job to be reviewed by later human capacity and allows the decision maker to know more information. Benchmarks \eqref{eq:T-fluid} and \eqref{eq:1-fluid} are special cases where $w=T$ and $w = 1$ respectively. 

Formally, consider the set of feasible window partitions \[\set{P}(w) = \left\{\bolds{\tau}=(\tau_1,\ldots,\tau_{I+1})\colon 1 = \tau_1 < \cdots < \tau_{I+1}=T+1;~\tau_{i+1}-\tau_i < w,~\forall i \leq I;~I \in \mathbb{N}\right\}.\]
The $w$-fluid benchmark minimizes the expected total loss over admission vector $\{a_k(t)\}_{k \in \set{K}, t \in [T]}$, partition $\bolds{\tau}$ and probabilistic service vector $\{\nu_k(t)\}_{k \in \set{K}, t \in [T]}$ that satisfy $w$-capacity constraints:
\begin{equation}\label{eq:w-fluid}
\tag{$w$-fluid}
\begin{aligned}
\set{L}^\star(w,T) \coloneqq &\min_{\substack{\bolds{\tau} \in \set{P}(w)\\\{a_k(t)\}_{k \in \set{K}, t \in [T]},\{\nu_k(t)\}_{k \in \set{K},t \in [T]}}} \sum_{t=1}^T \sum_{k \in \set{K}} \ell_k(\lambda_k(t) - a_k(t)) \\
\text{s.t.}\quad& \bolds{\sum_{t=\tau_i}^{\tau_{i+1}-1} a_k(t) \leq \mu_k\sum_{t=\tau_i}^{\tau_{i+1}-1} N(t)\nu_k(t),~\forall i \leq I, k \in \set{K}} \\
&a_k(t) \leq \lambda_k(t),~\nu_k(t) \geq 0,~\forall k \in \set{K}, t\in [T] \\
& \sum_{k\in \set{K}} \nu_k(t) \leq 1,~\forall t \in [T].
\end{aligned}
\end{equation}
We note that the idea of enforcing capacity constraints in consecutive windows is widely used in defining capacity regions for non-stationary (adversarial) queueing systems; see, e.g., \cite{borodin2001adversarial,liang2018minimizing,yang2023learning,nguyen2023learning}. 

As discussed above, if arrival rates and review capacity are stationary, then $\set{L}^\star(1,T) = \cdots=\set{L}^\star(T,T)$ and thus the choice of $w$ does not matter. For a general non-stationary system, we aim to design a feasible policy~$\pi$ that is robust to different choices of window size $w$. More formally, we define the regret for a window size $w$ as
\begin{equation}\label{eq:w-regret}
\reg^\pi(w,T) \coloneqq \expect{\set{L}^\pi(T)} - \set{L}^\star(w,T).
\end{equation}

\subsection{Regret guarantee for the $w-$fluid benchmarks}
Our regret guarantees for \eqref{eq:fluid} extend to guarantees for the $w-$fluid benchmarks. In a nutshell, the terms $\sqrt{KT}$ and $\sqrt{GT}$ in Theorems~\ref{thm:bacid}, \ref{thm:bacidol} and \ref{thm:cbacidol} become $w\sqrt{KT}$ and $w\sqrt{GT}$ when the regret is for a window size $w$, i.e., \eqref{eq:w-regret}. If $w$ is known, then by optimizing the hyperparameter $\beta \approx \sqrt{T / w}$ we can reduce the linear dependence on $w$ to be $\sqrt{w}$.

We illustrate how to derive such a guarantee for $\bacid$; the derivation is similar for $\bacidol$ and $\conbacid$. The crux is establishing Lemma~\ref{lem:window-known-idiosyncrasy}, a general version of Lemma~\ref{lem:known-idiosyncrasy} based on $w-$fluid benchmarks. Replacing Lemma~\ref{lem:known-idiosyncrasy} with Lemma~\ref{lem:window-known-idiosyncrasy} in the proof of Theorem~\ref{thm:bacid} then gives $\reg^{\bacid}(w,T) \lesssim w\sqrt{KT} + K$.
 \begin{lemma}\label{lem:window-known-idiosyncrasy}
\textsc{BACID}'s idiosyncrasy loss is
$\expect{\sum_{t=1}^T \ell_{\kappa(t)}(1 - A(t))} \leq \loss^\star(w,T)+\frac{wT}{\beta}$.
\end{lemma}
The proof of  Lemma~\ref{lem:window-known-idiosyncrasy} follows a similar structure with the proof of Lemma~\ref{lem:known-idiosyncrasy} in Section~\ref{ssec:bacid-idiosyncrasy}. Fix a window size $w$ and let $\bolds{\tau}^\star = (\tau^\star_1,\ldots,\tau^\star_I), \{a^\star_k(t), \nu^\star_k(t)\}_{k \in \set{K},t \in [T]}$ be the optimal solution to the $w$-fluid benchmark \eqref{eq:w-fluid}. Scaling the objective in \eqref{eq:w-fluid} by $\beta$, we let 
\[f(\{\bolds{a}(t)\}_t,\{\bolds{\nu}(t)\}_t,\bolds{u}) = \beta\sum_{t=1}^T \sum_{k \in \set{K}} \ell_k(\lambda_k(t) - a_k(t))-\sum_{i=1}^I\sum_{k\in\set{K}}u_{i,k}\left(\mu_k\sum_{t=\tau^\star_i}^{\tau^\star_{i+1}-1} N(t)\nu_k(t) - \sum_{t=\tau^\star_i}^{\tau^\star_{i+1}-1} a_k(t)\right),\]
where $\bolds{u} = (u_{i,k})_{i \in \set{I}, k \in \set{K}} \geq 0$ are dual variables for the capacity constraints. The Lagrangian for the given window partitions $\bolds{\tau}^\star$ is 
\begin{equation}
\begin{aligned}
f(\bolds{u})\coloneqq&\min_{a_k(t),\nu_k(t)} f(\{\bolds{a}(t)\}_t,\{\bolds{\nu}(t)\}_t,\bolds{u})  \\
\text{s.t.}\quad& a_k(t) \leq \lambda_k(t),~\nu_k(t) \geq 0,~\sum_{k' \in \set{K}} \nu_{k'}(t) \leq 1,~\forall k \in \set{K}, t\in [T].
\end{aligned}
\end{equation}
Recall the per-period Lagrangian in \eqref{eq:per-period-lag}, the following lemma, extending Lemma~\ref{lem:lagrang-bacid}, shows that the expected Lagrangian of \textsc{BACID}, $\sum_{t=1}^T \expect{f_t(\bolds{A}(t), \bolds{\psi}(t), \bolds{Q}(t))}$, is close to the $w-$fluid benchmark. 
\begin{lemma}\label{lem:window-lagrang-bacid}
For any window size $w$, the expected Lagrangian of \textsc{BACID} is upper bounded by:
\[
\sum_{t=1}^T \expect{f_t(\bolds{A}(t), \bolds{\psi}(t), \bolds{Q}(t))}  \leq \beta \loss^\star(w,T) + (w-1)T.
\]
\end{lemma}
\begin{proof}[Proof sketch.]
We provide a proof sketch of the result. The proof structure is similar with the proof of Lemma \ref{lem:lagrang-bacid} in Section~\ref{sec:lagrang-bacid}.

To connect the left hand side of Lemma~\ref{lem:window-lagrang-bacid} with the primal, we select a vector of dual variables $\bolds{u}^\star$ consisting of the queue lengths in the first period of each window: $u^\star_{i,k} = Q_k(\tau^\star_i)$. Then we have
\begin{align}
\sum_{t=1}^T \expect{f_t(\bolds{A}(t), \bolds{\psi}(t), \bolds{Q}(t))} - \beta \loss^\star(w,T) &= \expect{f(\{\bolds{a}^\star(t)\}_t,\{\bolds{\nu}^\star(t)\}_t,\bolds{u}^\star)} - \beta \loss^\star(w,T) \label{eq:window-lagran-relax}\\
&\hspace{-1in}+\sum_{t=1}^T\left(\expect{f_t(\bolds{A}(t), \bolds{\psi}(t), \bolds{Q}(t))} -\expect{f_t(\bolds{a}^\star(t), \bolds{\nu}^\star(t), \bolds{Q}(t))}\right) \label{eq:window-bacid-step-subopt} \\
&\hspace{-1in}+\sum_{i=1}^I\sum_{t=\tau_i^\star}^{\tau_{i+1}^\star-1}\left(\expect{f_t(\bolds{a}^\star(t), \bolds{\nu}^\star(t), \bolds{Q}(t))} - \expect{f_t(\bolds{a}^\star(t), \bolds{\nu}^\star(t), \bolds{Q}(\tau^\star_{i}))}\right)\label{eq:bacid-dual-subopt}
\end{align}
Compared to the decomposition at the top of Section~\ref{sec:lagrang-bacid}, there is a new term, \eqref{eq:bacid-dual-subopt}, that occurs due to the $w-$fluid benchmark.

The proof proceeds by bounding each of the three terms. First, a slight change to the proof of Lemma~\ref{lem:bound-lagrang} can  upper bound \eqref{eq:window-lagran-relax} by $0.$ Second, Lemma~\ref{lem:bacid-mw} upper bounds \eqref{eq:window-bacid-step-subopt} by $0$. Third, \eqref{eq:bacid-dual-subopt} is no larger than $(w-1)T$.
This is because the queue length  changes  at most linearly within a window and thus the difference in the Lagrangian is bounded by the window size when using either the per-period queue length or the initial queue length of a window as the duals. Summing up the above three bounds gives the desired result.
\end{proof}
\begin{proof}[Proof of Lemma~\ref{lem:window-known-idiosyncrasy}]
Combining Lemmas~\ref{lem:connect-idio-lag} and \ref{lem:window-lagrang-bacid} gives $\beta \expect{\sum_{t=1}^{T}\ell_{\kappa(t)}(1-A(t))} \leq wT + \beta \loss^\star(w,T)$, which finishes the proof by dividing both sides by $\beta$.
\end{proof}
\section{Additional Plots (Section~\ref{sec:content_moderation})}\label{app:numerics}
Following the same set-up as Figure~\ref{fig:80}, Figure~\ref{fig:four-pdfs} includes the percentage of misclassifications for the considered algorithms with different preset auto-deletion threshold. Moreover, we include a new algorithm, loss-weighted scheduling, that modifies $\conbacid$ by prioritizing admitted content with largest estimated loss $\bar{\ell}$ for review. As shown in these figures, our algorithm $\conbacid$ consistently outperforms $\staticTh$, a heuristic following \cite{Avadhanula2022}, and $\bacid$ w. OfflineML (when the review ratio is larger than $0.02$). In addition, loss-weighted scheduling can further improve the efficiency of $\conbacid$.
\begin{figure}[htbp]
  \centering
  \begin{subfigure}[b]{0.24\textwidth}
    \includegraphics[width=\linewidth]{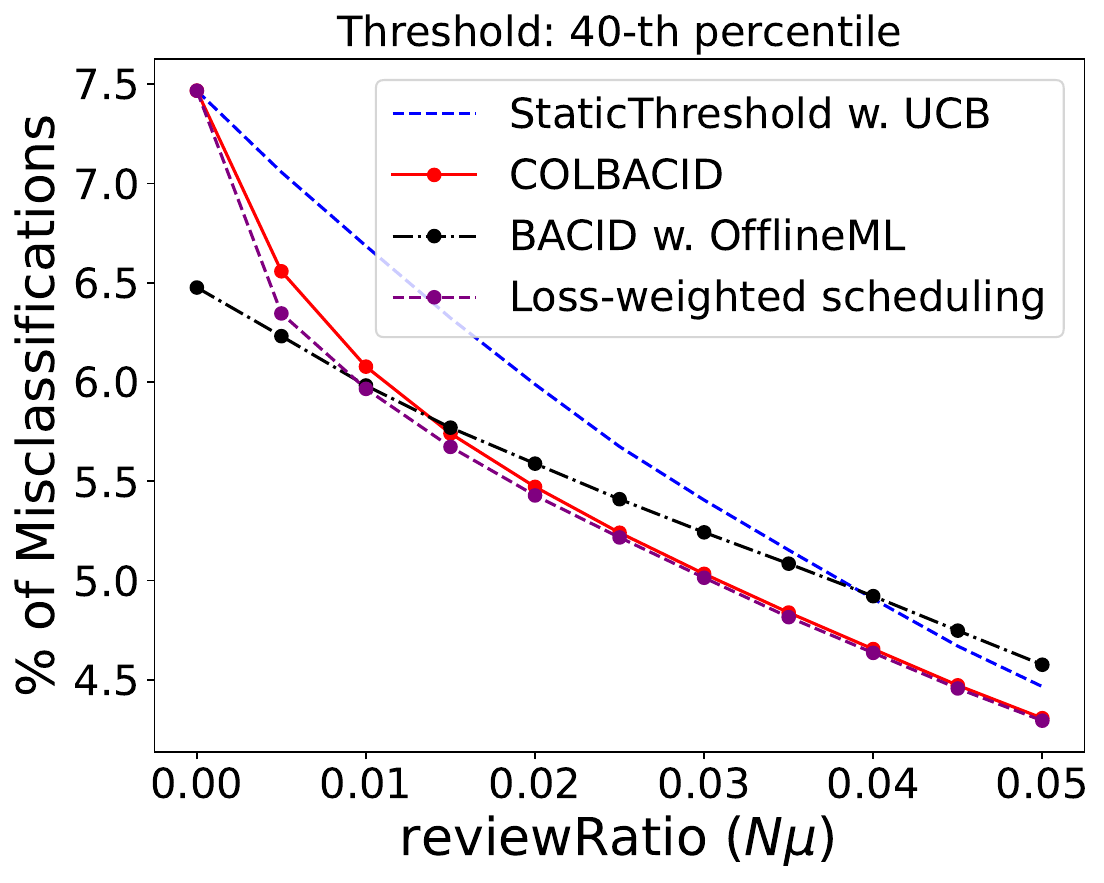}
  \end{subfigure}\hfill
  \begin{subfigure}[b]{0.24\textwidth}
    \includegraphics[width=\linewidth]{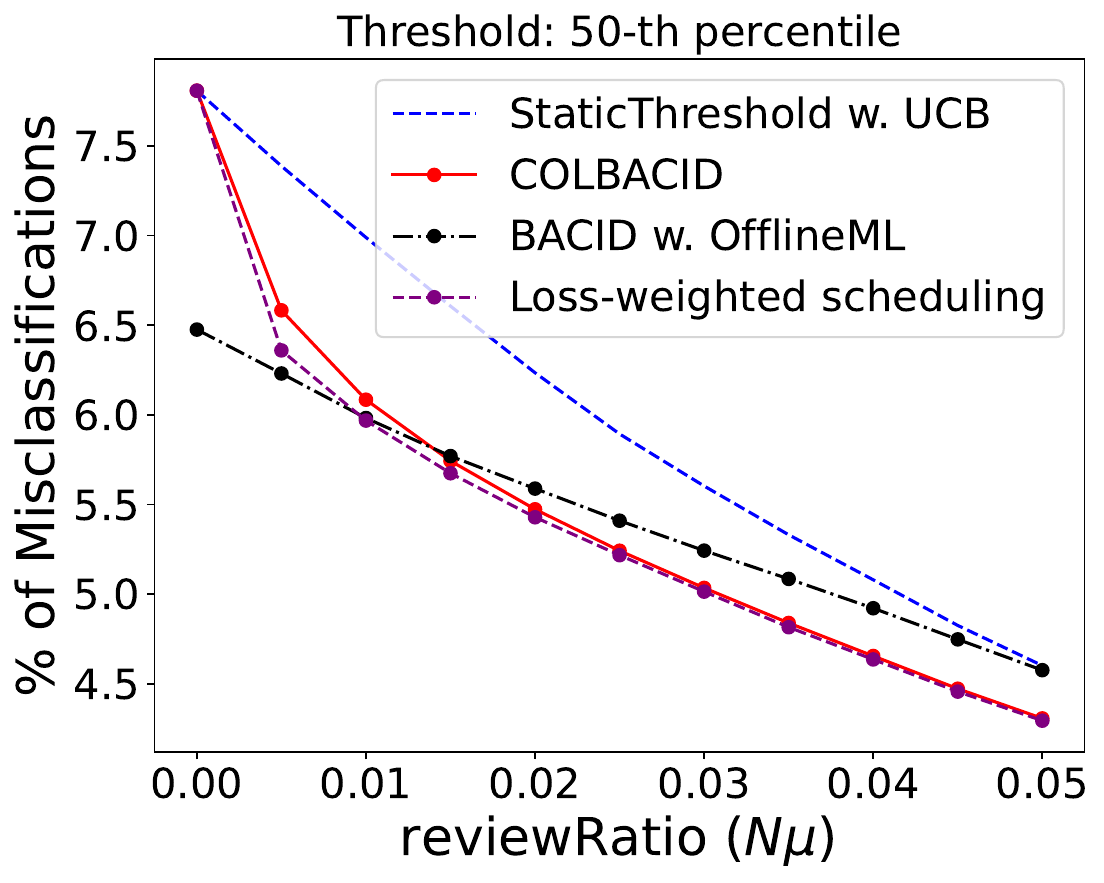}
  \end{subfigure}\hfill
  \begin{subfigure}[b]{0.24\textwidth}
    \includegraphics[width=\linewidth]{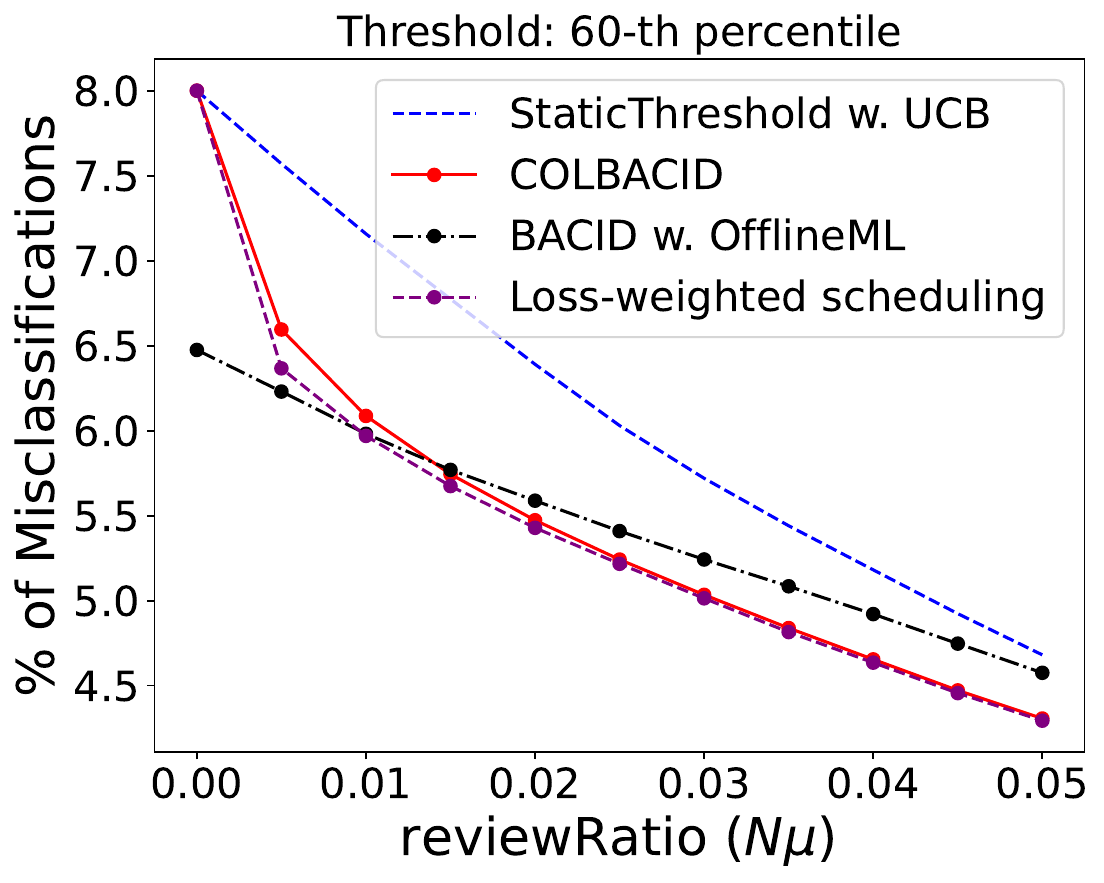}
  \end{subfigure}\hfill
  \begin{subfigure}[b]{0.24\textwidth}
    \includegraphics[width=\linewidth]{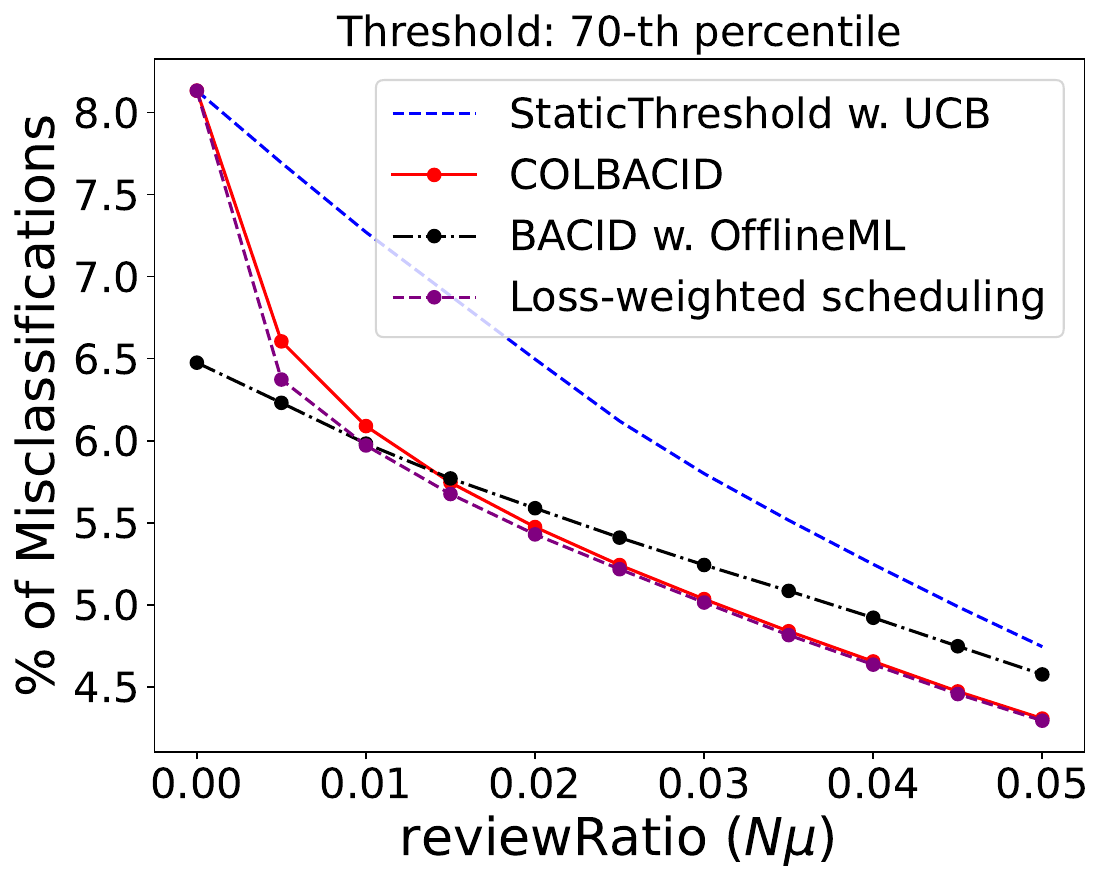}
  \end{subfigure}

  \caption{Algorithms' percentage of misclassifications for different auto-deletion threshold}
  \label{fig:four-pdfs}
\end{figure}
\section{Analytical Facts}
\subsection{Concentration Inequality}
We frequently use the following Hoeffding's Inequality \cite{boucheron2013concentration}.
\begin{fact}[Hoeffding's Inequality]\label{fact:hoeffding}
Given $N$ independent random variables $X_n$ taking values in $[0,1]$ almost surely. Let $X = \sum_{n=1}^N X_n$. Then for any $x > 0$,
    \[\Pr\{X-\expect{X} > x\} \leq e^{-2x^2/N},~\Pr\{X-\expect{X} < -x\} \leq e^{-2x^2/N}.\]
\end{fact}
We also use the Chernoff bound for sub-Gaussian random variables, which is given in \cite[Page 25]{boucheron2013concentration}.
\begin{fact}[Chernoff Bound]\label{fact:chernoff}
Given $n$ i.i.d. random variables $X_i$ that are sub-Gaussian with variance proxy $\sigma^2$ and letting $X = \sum_{i=1}^n X_i / n$, for any $x > 0$,
\[
\Pr\{X-\expect{X} > x\} \leq e^{-x^2/(2n\sigma^2)},~\Pr\{X-\expect{X} < -x\} \leq e^{-x^2 / (2n\sigma^2)}.
\]
\end{fact}
\subsection{Facts on Matrix Norm}
Our results rely on a self-normalized tail inequality derived in \cite{Abbasi-YadkoriPS11} which we restate here. Let us consider a unknown $\btheta^\star \in \mathbb{R}^d$, an arbitrary sequence $\{\bolds{X}_t\}_{t=1}^{\infty}$ with $\bolds{X}_t \in \mathbb{R}^d$, a real-valued sequence $\{\eta_t\}_{t=1}^{\infty}$, and a sequence $\{Z_t\}_{t=1}^{\infty}$ with $Z_t = \bolds{X}_t^{\trans}\btheta^\star + \eta_t$. We define the $\sigma-$algebra $\set{F}_t = \sigma(\bolds{X}_1,\ldots,\bolds{X}_{t+1},\eta_1,\ldots,\eta_t)$. In addition, recall the definition of $\hat{\bolds{\btheta}}_t$ as the solution to the ridge regression with regularizer $\xi$ and $\bar{V}_t$ in \eqref{eq:regression}. Then we have the following result implied by the second result in \cite[Theorem 2]{Abbasi-YadkoriPS11}.
\begin{fact}\label{fact:abbasi-bound}
Assume that $\|\bolds{\theta}^\star\|_2,~\|\bolds{X}_t\|_2 \leq U$ and that $\eta_t$ is conditionally sub-Gaussian with variance proxy $R^2$ such that $\forall u \in \mathbb{R}$, $\expect{e^{u\eta_t} \mid \set{F}_{t-1}} \leq \exp(u^2R^2/2)$. Then for any $\delta > 0$, with probability at least $1-\delta$, for all $t \geq 0$, $\btheta_\star$ lies in the set  
\[
\set{C}_t = \left\{\btheta \in \mathbb{R}^d \colon \|\hat{\btheta}_t - \btheta\|_{\bar{V}_t} \leq R\sqrt{d\ln\left(\frac{1+tU^2/\xi}{\delta}\right)} + \sqrt{\xi}U\right\}.
\]
\end{fact}
The following result is a restatement of a result in \cite[Lemma 11]{Abbasi-YadkoriPS11}.
\begin{fact}\label{fact:ellipsoid-bound}
Let $\{\bolds{X}_t\}$ be a sequence in $\mathbb{R}^d$, $\bV$ a $d\times d$ positive definite matrix and define $\bV_t = \bV + \sum_{\tau=1}^t \bolds{X}_t\bolds{X}_t^{\trans}$. If $\|\bolds{X}_t\|_2 \leq U$ for all $t$, and $\lambda_{\min}(\bV) \geq \max(1,U^2)$, then $\sum_{t=1}^n \|\bolds{X}_t\|_{\bV_{t-1}^{-1}}^2 \leq 2\ln \frac{\mathrm{det}(\bV_n)}{\mathrm{det}(\bV)}$ for any $n$.
\end{fact}
We also use a determinant-trace inequality from \cite[Lemma~10]{Abbasi-YadkoriPS11}.
\begin{fact}\label{fact:matrix-norm-bound}
Suppose $\bolds{X}_1,\ldots,\bolds{X}_t \in \mathbb{R}^d$ and for any $\tau \leq t$, $\|\bolds{X}_{\tau}\|_2 \leq U$. Let $\bV_t = \xi \bolds{I} + \sum_{\tau=1}^t \bolds{X}_{\tau}\bolds{X}_{\tau}^{\trans}$ for some $\xi > 0$. Then $\mathrm{det}(\bV_t) \leq (\xi+tU^2/d)^d$.
\end{fact}

\subsection{Additional Analytical Facts}
\begin{fact}\label{fact:negative-subgaussian}
If $X$ is a sub-Gaussian random variable with variance proxy $\sigma^2$, then $-X$ is sub-Gaussian with variance proxy $\sigma^2$. 
\end{fact}
\begin{proof}
For any $s \in \mathbb{R}$, $\expect{\exp(s(-X - \expect{-X}))} = \expect{\exp((-s)(X-\expect{X})} \leq \exp(s^2 \sigma^2 / 2)$. As a result, $-X$ is sub-Gaussian with variance proxy $\sigma^2$. 
\end{proof}
\begin{fact}\label{fact:addition-subGaussian}
Given two independent sub-Gaussian random variables $X,Y$ with variance proxy $\sigma_X^2$ and $\sigma_Y^2$ respectively, their sum $X+Y$ is sub-Gaussian with variance proxy $\sigma_X^2 + \sigma_Y^2.$
\end{fact}
\begin{proof}
For any $s \in \mathbb{R}$, $\expect{\exp(s(X+Y-\expect{X+Y}))} = \expect{\exp(s(X-\expect{X})}\expect{\exp(s(Y-\expect{Y})} \leq \exp(s^2 (\sigma_X^2+\sigma_Y^2) / 2)$, where the equality uses independence between $X,Y$.
\end{proof}

\begin{fact}\label{fact:lnt}
For $t \geq 100$, we have $t/4 \geq \sqrt{t\ln t}$.
\end{fact}
\begin{proof}
It suffices to show $t/16 \geq \ln t$ for $t \geq 100$. Let $f(t) = t/16 - \ln(t)$. Then $f'(t) = 1/16 - 1/t \geq 0$ for $t \geq 16$. Thus $f(t)$ increases for $t \geq 16$. We prove the desired result by noting $f(100) \geq 0$.
\end{proof}
\begin{fact}\label{fact:ln-property}
We have $x\ln(1+3/x) \geq 1$ for any $x \geq 1$.
\end{fact}
\begin{proof}
Leg $f(x) = x\ln(1+3/x)$. Then $f'(x) = \ln(1+3/x)-\frac{3}{x+3}$ and $f''(x) = \frac{-3}{x(x+3)} + \frac{3}{(x+3)^2)} \leq 0$. Therefore, $f'(x)$ is a decreasing function with $f'(+\infty) = 0$, and $f'(x) \geq 0$. This shows $f(x)$ is an increasing function. Since $f(1) = \ln(4) \geq 1$, we have $f(x) \geq 1$ for any $x \geq 1$.
\end{proof}

\end{document}